\documentclass[]{IEEEtran}%

\usepackage{hyperref}       
\usepackage{url}            
\usepackage{booktabs}       
\usepackage{amsfonts}       
\usepackage{nicefrac}       
\usepackage{microtype}      
\usepackage{float} 
\usepackage{lipsum}
\usepackage[font=small,skip=10pt]{caption}
\usepackage{algpseudocode}
\usepackage[ruled,vlined,linesnumbered]{algorithm2e}

\usepackage{stfloats}

\include{pythonlisting}
\usepackage{subcaption}
\usepackage[]{graphicx}
\usepackage{xcolor}
\usepackage{afterpage}

\usepackage{epsfig, graphicx, amsmath, amssymb,cite}
\usepackage{dsfont}
\usepackage{amsthm}
\usepackage{bm}
\usepackage{lipsum}
\usepackage{mwe}

\newtheorem{lemma}{Lemma}
\usepackage{flushend}
\newtheorem{theorem}{Theorem}
\newtheorem{corollary}{Corollary}
\DeclareMathOperator{\Tr}{Tr}
\newtheorem{definition}{\bf Definition}

\DeclareMathOperator*{\argmax}{\arg\!\max}
\DeclareMathOperator*{\argmin}{\arg\!\min}
\IEEEoverridecommandlockouts

\begin{document}
\title{\huge A hybrid model-based and learning-based approach for classification with limited number of data samples}
\author{\IEEEauthorblockN{Alireza Nooraiepour, Waheed U. Bajwa and Narayan B. Mandayam\\} 
\thanks{The authors are with WINLAB, Department of Electrical and Computer Engineering, Rutgers University, NJ, USA. Emails: {$\{$alinoora,narayan$\}$}@winlab.rutgers.edu, {waheed.bajwa}@rutgers.edu.

This work was supported by the National Science Foundation (NSF) under grants ECCS-2028823 and OAC-1940074, and in part by ACI-1541069.
}
    }
\date{}
\maketitle
\vspace{-3em}
\begin{abstract}
The fundamental task of classification is considered for physical systems with known parametric statistical models given a limited number of data samples. The proposed solution, \textsf{\textsc{HyPhyLearn}}, is a hybrid classification method, which exploits both the physics-based statistical models and the learning-based classifiers. Notably, the standalone learning-based and statistical classifiers face major challenges towards the fulfillment of the classification task. Specifically, the physics-based statistical models usually suffer from the inability to properly tuning the underlying unobservable parameters, which would lead to a mismatched representation of the system's behaviors. Learning-based classifiers, on the other hand, rely on a large number of data from the underlying physical process which might not be accessible in most practical scenarios. In this vein, we conjecture that \textsf{\textsc{HyPhyLearn}} would alleviate the difficulties associated with each approach by fusing their individual strength. The proposed hybrid approach estimates the the unobservable model parameters using the available (suboptimal) estimation procedures, and subsequently use the physics-based statistical models to generate \textit{synthetic data}. Then, the data samples are incorporated with the synthetic data in a learning-based classifier powered from adversarial training of the neural networks. In particular, to address the mismatch problem the classifier learns a mapping to a common feature space from the data samples and the synthetic data. Simultaneously, the classifier is trained to find discriminative features from this space in order to fulfill the classification task. Two case studies revolving around two concrete communication problems are presented in order to highlight the applicability of \textsf{\textsc{HyPhyLearn}}. Numerical results demonstrate that the proposed approach leads to major classification improvements in comparison to the existing standalone or hybrid classification methods.
\end{abstract}
\section{Introduction}
\label{section:Introduction}
We revisit the problem of classification with limited number of training data samples in this paper. The fundamental task of classification comes up in various fields and is traditionally tackled within two frameworks: $1)$ statistical setting, and $2)$ fully data-driven setting. In the first case, the main assumption is that data generation adheres to a known probabilistic model of the underlying physical process. Subsequently, the classification problem is usually dealt with within a hypothesis testing (HT) framework aimed at testing between two (or more) hypotheses. Here, optimality in both the Bayesian sense and the Neyman--Pearson sense relies on computation of the \textit{likelihood-ratio} terms, which requires clairvoyant knowledge of the probabilistic models under different hypotheses \cite{lehmann2005testing}. However, accurate modeling of the physical processes in increasingly complex engineered systems is either not tractable or it relies on a large number of unobservable parameters, estimation of which from limited number of data samples could be a major hurdle \cite{meroune,Fink}. As a result, a mismatch between the physics-based statistical models and the real physical processes is inevitable. This precludes exact computation of the {likelihood-ratio} values, which deteriorates the classification performance \cite{mismatch-sufficient}. The fully data-driven (i.e., learning based) setting, on the other hand, relies on a large number of data samples for finding an optimal mapping from the data samples to the corresponding labels. But availability of such data in many real-world problems, e.g., channel-based spoofing detection \cite{CFR1} and signal identification \cite{SignalIdentification}, is generally limited, which might lead to learning of a suboptimal map. Moreover, one should always expect mislabeled data in many applications, since the employed labeling procedures might not be error free. Consequently, classification performance of data-driven models can be seriously limited for many real-world applications.


The overarching objective of this paper is to develop an algorithmic framework for classification from limited number of training data samples in applications in which neither model-based nor learning-based approaches alone result in very good classification performance. To this end, note that learning-based approaches traditionally tend to disregard the physics-based models developed to describe the physical phenomena through tractable mathematical analysis. For instance, in the context of wireless communications, numerous theoretical models for channels and resource management have been developed over the years \cite{meroune,CFR1,CFR2}. Despite being approximations in many cases, these models provide important prior information about the corresponding physical systems that might be utilized to facilitate the subsequent classification tasks. At the same time, physics-based models consist of numerous unobservable parameters, the tuning of which is a major hurdle for complex systems \cite{Fink}. For example, physical channel models in the multi-input multi-output (MIMO) and 5G communications scenarios rely on a large number of multidimensional parameters that are defined over a mixed set of discrete and continuous spaces \cite{Rappaport-model,ImpactofIncomplete}. In such cases, the maximum likelihood estimation (MLE) of the parameters could incur a formidable computational cost \cite{ImpactofIncomplete,DMC-in11GHz,Buzzi0}. 
Our goal in this context is to develop a classification framework that can deal with these practical considerations through a hybrid approach that consolidates physics-based and fully data-driven classification approaches. The expectation is that the hybrid approach would fuse the strengths of the two approaches towards achieving an overall superior classification performance. 

Our proposed hybrid approach first employs the (necessarily) suboptimal parameter estimation methods to estimate the unobservable parameters. Then, it utilizes them in the physics-based models to generate \textit{synthetic data}, which enables us to leverage learning-based classification approaches. The mismatch between the physics-based models and the underlying physical process is addressed in a learning setting. Specifically, a neural network is trained to map the training and synthetic data to a common discriminative feature space, which is often referred to as domain-invariant space in the domain adaptation literature \cite{domain-invariant,DA-overview}. Meanwhile, a neural network-based classifier is trained on the mapped synthetic data to extract class-specific discriminative features from them. The resulting classifier in this way is expected to perform well on both synthetic and training data distributions.

\subsection{Relation to prior works}
In the realm of statistical model-based classifiers, the difficulties associated with estimating the parameters of the physics-based models are recognized in various works \cite{mismatch2001,mismatch-sufficient}. This is mainly attributed to the inherent difficulties associated with determining probability distributions from only a limited number of data samples. Along these lines, classification under the assumption of mismatched models is considered in several works \cite{mismacth-1980,mismatch2001,novelTight,mismatch-sufficient}. Specifically, \cite{mismatch2001,novelTight} derive bounds on the probability of classification error in the presence of mismatch via
the $f$-divergence between the true and mismatched distributions. In contrast to these bounds that are general in the sense that no assumption is made regarding the underlying distributions, \cite{mismatch-sufficient} considers data that are contained in a linear subspace. This enables the authors to derive an upper bound on the classification error of the mismatched model that predicts the presence/absence of an error floor. The analyses in these works, however, do not lead to a classification algorithm for the mismatched setting as they merely analyze the mismatch problem itself.


The mismatch problem for the learning-based classifiers corresponds to the cases where the distribution of the available training data is different from that of the test data. Such mismatches are primarily studied in the transfer learning (TL) and the data-shift literature \cite{DA-overview}. In particular, covariate shift \cite{covsh}, which is also studied under the name of transductive TL \cite{TLSurvey}, refers to the case where the underlying data distributions for the test and training data are different. Concept shift \cite{ConceptShift}, also known as inductive TL \cite{TLSurvey}, on the other hand, deals with situations in which the posterior distribution of the labels given the data is not the same for the training and the test data. A wide range of algorithms have been proposed in order to alleviate the performance loss due to such shifts. For example, importance-weighting technique \cite{importancewighting,KDE} is proposed for the covariate shift scenario to remove the bias from the training data. Furthermore, algorithms based on subspace mapping \cite{subspace} and learning domain-invariant representations \cite{domain-invariant} have also been proposed in the literature to address the mismatch problem. The authors in \cite{subspace} propose a transfer component analysis method aimed at finding a transformation under which the maximum mean discrepancy between the true and mismatched distributions is small. The work in \cite{domain-invariant} aims at finding a representation that is invariant for the training and test distributions in order to mitigate the effect of discrepancies in the subsequent learning tasks. For the specific task of classification, the authors in \cite{DANN} introduce the domain-adversarial neural network (DANN) framework, which extracts domain-invariant representations via (deep) neural networks that are discriminative for the training data in order to devise a classifier on the test data.

Deep transfer learning (DTL) is another prime subject related to our work that studies the transfer learning concept in the context of deep neural networks (DNNs). DTL considers a DNN that has been pre-trained on the training data as transferable knowledge useful for the test data. This knowledge can be transferred based on different strategies. The pre-trained DNNs can either be used directly for the test data, or serve as an intermediate feature extracting step {that} could facilitate the subsequent learning process for the test data. In another DTL strategy called \textit{fine-tuning}, the pre-trained DNN or, certain parts of it, is refined using the available test data to further improve the effectiveness of transfer knowledge. We refer the reader to \cite{TLSurvey-wirelesscomm,TL-Srurvey} for a survey on DTL methods.

Model-based deep learning is another related line of work that aims at designing systems whose operation combines physics-based models (domain knowledge) and data. To this end, two main strategies are typically exploited in such works, known as model-aided networks and DNN-aided inference. The former results in specialized DNN architectures by identifying structures in a model-based algorithm; e.g., an iterative structure for the case of deep unfolding \cite{hershey2014deep}. The latter primarily utilizes model-based methods for inference, but replaces explicit domain-specific computations with dedicated DNNs in order to facilitate operation in complex environments; e.g., using generative models for compressed sensing applications \cite{bora2017compressed}. We refer the readers to \cite{shlezinger2021modelbased} and references therein for the state-of-the-art strategies in model-based deep learning methods.

There also have been previous attempts to incorporate physics-inferred information in the fully data-driven setting. In the field of wireless communications, for instance, the authors in \cite{meroune} employ DTL to solve a specific resource management problem. Similarly, the task of signal classification is tackled via DTL under different practical assumptions, such as real propagation effects \cite{DTL-Signal-propagation-effect}, hardware impairments \cite{DTL-Signal-hardware} and weak received signal strength \cite{DTL-Signal-Weak-RSS}. These works utilize abundant data from an approximate model along with limited data from the real-world model in the DTL fine-tuning approach. More closely to the idea of physics-guided machine learning (ML), a recurrent neural network (RNN) is modified in \cite{RNN-physics} to incorporate information from the physics-based model as an internal state of the RNN. Furthermore, parameters of the physics-based models are combined with sensor
readings and used as input to a DNN to develop a hybrid prognostics model in \cite{Fink}. 

 We note that the aforementioned works in domain adaptation literature do not employ any available physics-based statistical models and, consequently, rely on large number of training data samples for dealing with the mismatch problem. In addition, model-based deep learning strategies might not be applicable to the statistical classification problem in general due to the lack of algorithmic structure such as an iterative structure. Equally importantly, DTL fine-tuning and physics-guided learning approaches do not consider the difficulties associated with estimating the physics-based parameters, which would indeed lead to inaccurate physics-based statistical models. The resulting discrepancy between the model and the underlying physical process necessitates a learning-based classifier that is capable of leveraging the data in a way to alleviate this mismatch problem.
\subsection{Our contributions}
The main contributions of this work are as follows.
\begin{itemize}
    \item We focus on the task of classification for a physical process assuming that a limited number of training data samples, with possibly mislabeled instances, is available. We consider the case where the physical process (or its approximation) can be described by physics-based parametric statistical models. As these models tend to be complex in general, estimation of the unknown model parameters using the maximum likelihood estimation (MLE) procedure could be computationally prohibitive.\footnote{As discussed later in Section \ref{section:Problem statement}, even using the MLE does not {always} provide any optimality {guarantees} in general for the classification problem in a HT setting \cite{GLRT-different-fields}.} We instead propose \textsf{\textsc{HyPhyLearn}}---a novel hybrid classification method---as a solution, which exploits both physics-based statistical models and learning-based classifiers. This approach {makes} use of (necessarily suboptimal) parameter estimation algorithms/heuristics to obtain (approximate) parameter estimates. Next, plugging in these estimates in the physics-based statistical models enables us to generate \textit{synthetic} data. \textsf{\textsc{HyPhyLearn}} then relies on neural networks (NNs), which are powerful tools for finding a discriminative feature space, towards {obtaining} a learning-based classifier. Specifically, the learning process involves training a NN to map the training and synthetic data to a common space under which they are not distinguishable. In the mean time, a learning-based classifier is trained on the synthetic data mapped to the new space to find discriminative class-level features. Indeed, learning the common feature space addresses the {distribution} mismatch problem between the {training} data samples and the generated synthetic data {due to the errors in parameter estimation}. It is then expected that the classifier trained on the mapped synthetic data will perform well on both data distributions. We repurpose theories from the domain adaptation literature based on learning invariant representations for our specific problem to justify the proposed hybrid approach. A schematic of \textsf{\textsc{HyPhyLearn}} for a binary classification example is illustrated in Fig. \ref{Fig:graphic}.
  \begin{figure*}[ht]
        \centering
         \includegraphics[width=14cm]{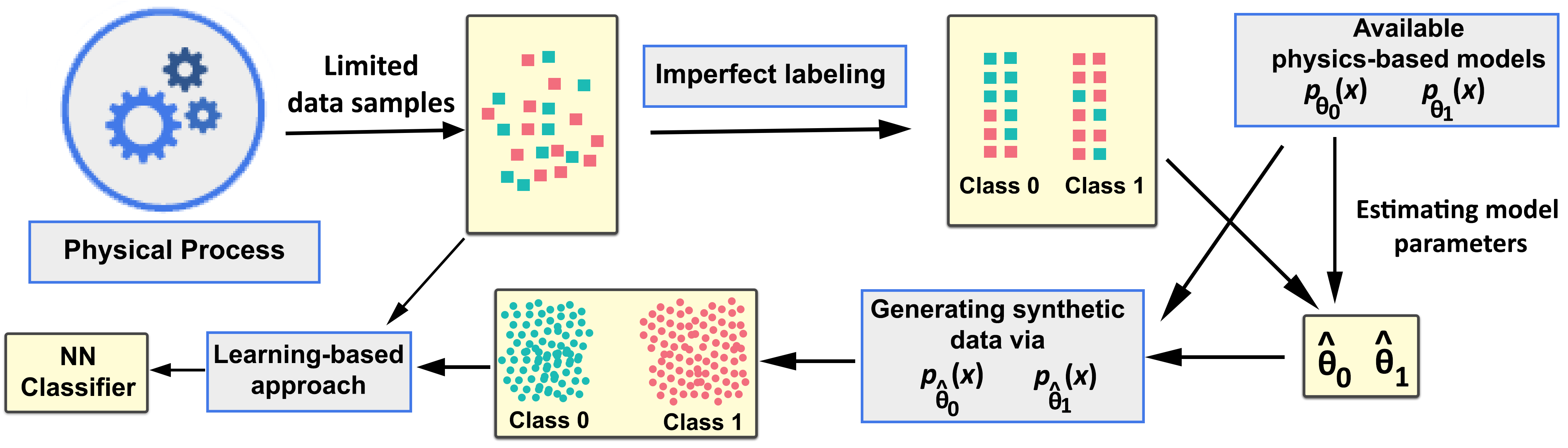} 
        \caption{A schematic of our proposed hybrid classification approach (\textsf{\textsc{HyPhyLearn}}) illustrated for a binary classification setting, which exploits both physics-based statistical models and learning-based classifiers.}
        \label{Fig:graphic}
 \end{figure*}
 
\item We also consider two prototypical problems from the wireless communications literature to investigate the performance of our proposed approach and show its superiority in comparison to the stand-alone statistical model-based classifiers as well as the fine-tuning approach as the best existing hybrid approach applicable to these problems. We first consider the problem of channel spoofing in the wireless communications setting, where an adversary (Eve) spoofs a legitimate transmitter (Alice) and sends a message to a legitimate receiver (Bob) \cite{CFR1,SurveyPLS,TCCN}. The {spoofing detection} at Bob involves making a decision on whether an incoming message corresponds to Alice or Eve. This can be cast as a binary classification problem at Bob. Second, we revisit the problem of multi-user detection (MUD) in the uplink of a cellular network, where different users are asynchronously sharing a channel with a base station \cite{Buzzi0}. For a $K$-user system, MUD is basically a $2^K$-ary classification problem {in which the} goal is to infer $K$ binary information bits from a given observation. By obtaining likelihood ratio test (LRT) for each problem, we show that statistical model-based classifiers rely heavily on the wireless channel parameters in the above problems. However, estimation performance of these parameters  suffers from both the paucity of training data and complexity of the physics-based statistical models. In fact, these models are complex {in the sense} that MLEs of the corresponding parameters {require} an exhaustive search over the space of the parameters, which is not feasible for many communication scenarios including MIMO transmissions in a 5G setting \cite{Rappaport-model}. 
For both problems, numerical results show that \textsf{\textsc{HyPhyLearn}} provides major improvements in terms of the classification accuracy in comparison to the best existing approaches.
\end{itemize}
\subsection{Notation and organization}
Throughout the paper, vectors are denoted with lowercase bold letters, while uppercase bold letters are reserved for matrices. Furthermore, equality by definition is expressed through the symbol $\overset{\bigtriangleup}{=}$. Non-bold letters are used to denote scalar values and calligraphic letters denote sets. Furthermore, the cardinality of a set $\mathcal{S}$ is denoted by $\lvert \mathcal{S}\rvert$. The spaces of real and complex vectors of length $d$ are denoted by $\mathbb{R}^d$ and $\mathbb{C}^d$, respectively. The $m$th element of a vector $\mathbf{u}$ and the trace of a matrix $\mathbf{U}$ {are} shown by $\mathbf{u}[m]$ and $\Tr(\mathbf{U})$, respectively. Also, real and imaginary parts of a complex number $a$ are denoted by $\Re\{a\}$ and $\Im\{a\}$, respectively. The probability density function and expectation of a random variable $w$ are denoted by $p(w)$ and $\mathbb{E}_p(w)$, respectively, while $\mathbb{P}[\cdot]$ is used to denote the probability of an event. The Gaussian and circularly-symmetric complex Gaussian distributions are denoted by $\mathcal{N}$ and $\mathcal{CN}$, respectively, while the uniform distribution {supported} between two real numbers $a$ and $b$ is denoted by $\text{unif}(a,b)$. We denote the $k$th standard basis vector of length $N$ in $\mathbb{R}^N$ by $\mathbf{e}_k$, and use $\mathbf{\|\mathbf{u}\|}$ to refer to the Euclidean norm of the vector $\mathbf{u}$. We refer to identity matrix of size $N$ and the indicator function by $\mathbf{I}_{N}$ and $\mathds{1}_{\mathcal{A}}(\mathbf{x})\overset{\bigtriangleup}{=}\begin{cases}
1, \mathbf{x}\in \mathcal{A}\\
0, \mathbf{x}\notin \mathcal{A}
\end{cases}$, respectively. Transpose and conjugate transpose of $\mathbf{u}$ are denoted by $\mathbf{u}^T$ and $\mathbf{u}^H$, respectively. Furthermore, $\mathbf{e}_n(y)$ refers to a \textit{one-hot} encoded version of a non-negative integer $y$, which equals to an all-zero vector of length $n$ except for the $y$th element which is set to $1$. Also, $\circ$ and
$\odot$ denote the Schur componentwise and the Khatri-Rao product, respectively, while $\otimes$ is reserved for the Kronecker product. Finally, given two vectors $\mathbf{a}$ and $\mathbf{b}$ of length $M$, Toeplitz matrix of size $M\times M$ is defined as
$\text{toep}(\mathbf{a},\mathbf{b})\overset{\bigtriangleup}{=}\begin{bmatrix}\mathbf{a}[1]&\mathbf{b}[2]&\dots &\mathbf{b}[M]\\\mathbf{a}[2]&\ddots&\ddots&\mathbf{b}[M-1]\\\vdots&\ddots&\ddots&\vdots\\\mathbf{a}[M]&\mathbf{a}[M-1]&\dots&\mathbf{a}[1]\end{bmatrix}$.\normalsize

The rest of the paper is organized as follows. The problem is formally posed in Section \ref{section:Problem statement}. Our proposed solution is described in {Section}~\ref{Section:proposed solution}{, which discusses} various pieces of \textsf{\textsc{HyPhyLearn}} approach. We introduce the first case study involving the spoofing detection problem in Section \ref{Section:Case study 1: Spoofing detection via channel frequency response}. The second case study, which concerns the multi-user detection problem, is presented in Section \ref{section:Case study 2: Multi-user detection}. We present numerical {results} concerning the application of our proposed approach in the above two case studies in Section \ref{section:Nuemrial Results}, and contrast it with the existing methods. Finally, the paper is concluded in Section \ref{section:Conclusions}.
\section{Problem Formulation}
\label{section:Problem statement}
Consider a physical process consisting of $C$ distinct behaviors where the physics-based parametric statistical {model} for the $i$th behavior is available in the form of a parametric probability density function (PDF) denoted by the conditional prior $p_i(\mathbf{x};{\bm{\theta}_i})$ {on observations $\mathbf{x}$ that belong to an observation space $\mathcal{X}$}. Assuming the true underlying parameter for the $i$th behavior is $\bm{\theta}^*_i$, the data for this behavior is generated by drawing independent and identically distributed (i.i.d.) samples from $p_i(\mathbf{x};{\bm{\theta}^*_i})$. {Assuming further} that the $i$th behavior is chosen with a prior probability $\pi_i$, our goal is to devise a decision rule to determine a given sample $\mathbf{x}=[x_1,\dots,x_n]^T$ is generated under which behavior. {Clearly}, this can be cast as a $C$-ary classification problem via
$
    H_i:\mathbf{x}\sim p_i(\mathbf{x};{\bm{\theta}^*_i}),\ i=0,\dots,C-1
$. We consider the case where this decision is made by a classifier {$h_{\bm\phi}(\cdot)$} parameterized by $\bm{\phi}\in\mathbb{R}^d$, $h_{\bm\phi}(\mathbf{x}):\mathcal{X}\rightarrow\{0,\dots,C-1\}$, which partitions $\mathcal{X}$ into $C$ disjoint sets, $\{\mathcal{X}_i\}$, and decides in favor of $H_i$ if $\mathbf{x}\in\mathcal{X}_i$. {Defining} ${\bm{\theta}^*} {\overset{\bigtriangleup}{=}} [{\bm{\theta}^*_0},\dots,{\bm{\theta}^*_{C-1}}]$, we denote {the} probability of error associated with $h_{\bm\phi}(\mathbf{x})$ by $ \mathbb{P}_{\bm{\theta}^*}[{e}_\mathbf{\phi}]$, which can be computed as
\begin{align}
\label{eq:err-probability}
    \mathbb{P}_{\bm{\theta}^*}[{e}_{\bm{\phi}}]=\sum_{i=0}^{C-1}\pi_i\int_{\mathcal{X}}p_i(\mathbf{x};{\bm{\theta}^*_i})\mathds{1}_{\{h_{\bm{\phi}}(\mathbf{x})\neq i\}}(\mathbf{x})d\mathbf x, 
\end{align}
where ${e}_{\bm{\phi}}$ indicates the event that $h_{\bm\phi}(\mathbf{x})$ makes an erroneous decision. The optimal classifier $h_{\bm\phi^*}(\mathbf{x})$ that minimizes the error probability is given by the Bayes decision rule, i.e., $
    h_{\phi^*}(\mathbf{x})=\argmax_{i=0,\dots,C-1}\ \pi_i p_i(\mathbf{x};{\bm{\theta}^*_i})$ \cite{lehmann2005testing}. For the specific case of $C=2$, this rule takes the famous form of the likelihood ratio test, $\frac{ p_1(\mathbf{x};{\bm{\theta}^*_1})}{ p_0(\mathbf{x};{\bm{\theta}^*_0})}\underset{y=0}{\overset{y=1}{\gtreqless}}\frac{\pi_0}{\pi_1}$, where $y=i$ implies making a decision in favor of the $i$th behavior. 

We focus {in this paper} on the case where although the parametric model $p_i(\mathbf{x};{\bm{\theta}_i})$ is known for the $i$th behavior, one does not have access to the corresponding underlying true parameter $\bm{\theta}^*_i$. Instead, only a {small} number of training data generated in an i.i.d.\ manner from $p_i(\mathbf{x};{\bm{\theta}^*_i}), \forall i,$ {are available}. Specifically, we {denote} the available dataset by {$\mathcal{D}_r=\{\mathbf{x}_{r,n}\}_{n=1}^{N_{r}}$}, where $N_{r}$ is the total number of data samples. Also, the corresponding ground-truth label for the {$n$th sample is denoted by $y_{r,n}$} which is only given for $N_{r,l}$ number of data samples where $N_{r,l}\leq N_{r}$. Furthermore, we consider the case where the model $p_i(\mathbf{x};{\bm{\theta}_i})$ under the $i$th behavior is a {non-trivial} function of the underlying parameter for which conventional estimation procedures {such as} maximum likelihood estimation (MLE) are either not available or {are} computationally prohibitive to implement. {The implication of this aspect of the problem formulation is that the} performance of {any} suboptimal parameter estimation {method is bound to be} limited. As a result, statistical model-based classifiers, which plug-in these estimates in $p_i(\mathbf{x};{\bm{\theta}_i})$, would have a deteriorated performance as well. 

Unlike these classifiers that rely heavily on the knowledge of {the} parametric statistical models and the estimated parameters, a purely data-driven approach can result in a classifier that disregards the available parametric models. However, as the data generation processes are governed {by non-trivial} models, a large number of data is needed in this case to extract related patterns from each behavior {that would lead to a highly} discriminative feature space. By noting that the performance of the fully data-driven and the statistical model-based classifiers is particularly curbed when they are used in {a} stand-alone fashion, we conjecture that fusing the {strengths} of the two can lead to a superior classification algorithm in our setting, as described in the next section. 

Before delving into the proposed solution {for the described problem setting}, we {discuss further} two existing approaches towards obtaining a statistical model-based classifier {for the benefit of the reader}. {Recall that within the framework of statistical model-based classification}, one would {first} estimate the {unknown model} parameters as $\widehat{\bm{\theta}}_i$'s, $i=1,\dots,C$, and plug them in the available models to obtain $p_i(\mathbf{x};{\widehat{\bm{\theta}}}_i)$. The resulting plug-in {models are} then used in practice in lieu of the {true models within the} optimal Bayes decision rule. The parameters{, $\bm{\phi}$,} of {the resulting \emph{plug-in}} classifier consist solely of the parameters of physics-based statistical models, i.e., $\bm{\phi}={\bm{\theta}}=[{\bm{\theta}}_0,\dots,{\bm{\theta}}_{C-1}]$.\footnote{For notational simplicity and without loss of generality, we have not included the priors as part of the unknown parameters in the current discussion.} Based on this fact, we denote the \emph{plug-in} classifier by $h_{\bm{\theta}}(\mathbf{x})$ in the remainder of this section. The unknown model parameters can be estimated {using numerous approaches}. In the following, we discuss two of the most {popular} ways to estimate {them as well as the shortcomings of these approaches that warrant a new approach to classification}.

    \textbf{Empirical error minimizer}: Given a set of training data with their corresponding labels, {$\{\mathbf{x}_{r,n},y_{r,n}\}_{n=1}^{N_r}$}, the most natural approach for parameter estimation {corresponds to the setting in which} the resulting plug-in classifier, {$h_{{\bm{\theta}}}(\mathbf{x})$}, minimizes the \textit{empirical} error probability defined by {$\widehat{\mathbb{P}}^{N_r}[{e}_{\bm{\theta}}]\overset{\bigtriangleup}{=}\frac{1}{N_r}\sum_{n=1}^{N_r}\mathds{1}_{\{h_{\bm{\theta}}(\mathbf{x}_{r,n})\neq y_{r,n}\}}$}. Specifically, for the case of $C=2$ consider the family of the classifiers
$
    {h_{{\bm{\theta}}}(\mathbf{x})}=\begin{cases}
    0, \quad   \pi p_{{\bm\theta}_0}(\mathbf{x})>(1-\pi) p_{{\bm\theta}_1}(\mathbf{x}),\\
    1, \quad \text{otherwise},
    \end{cases}
$ for which the parameter values ${{\bm\theta}_0}$ and ${{\bm\theta}_1}$ are chosen from a space $\bm{\Theta}$. The parameter estimates {that} minimize the empirical error are obtained as $\widehat{\bm{\theta}}=[\widehat{\bm{\theta}}_0,\widehat{\bm{\theta}}_1] \in \argmin_{{\bm{\theta}}} \widehat{\mathbb{P}}^{N_r}[{e}_{{\bm{\theta}}}]$. The following lemma, which is a direct result of Corollary $16.1$ in \cite{pbook}, presents an upper bound on the performance of the Bayes decision rule in terms of that of the plug-in classifier {that is} obtained {using} empirical error minimization.
 \begin{lemma}
 If $\bm{\theta}^*_0,\ \bm{\theta}^*_1\in \bm{\Theta}$, then the error probability of the Bayes decision rule, with the probability at least $1-\delta$, is bounded by
 \begin{align}
  {\mathbb{P}}_{{\bm{\theta}^*}}[{e}_{{\bm{\theta}^*}}]   \leq \widehat{\mathbb{P}}^{N_r}[{e}_{\widehat{\bm{\theta}}}]+8\sqrt{\frac{2}{N_r}\log\frac{8b}{\delta}},
 \end{align}
 where $b$ denotes the Vapnik–Chervonenkis (VC) dimension \cite{pbook} of the family of classifiers, $h_{{\bm{\theta}}}(\mathbf{x})$, defined above.
 \end{lemma}
 The above lemma guarantees a $\mathcal{O}(\sqrt{\log {N_r}/{N_r}})$ {rate of convergence} to the Bayes error for $h_{\hat{\bm{\theta}}}(\mathbf{x})$ when $\widehat{\bm{\theta}}$ is chosen to minimize the empirical error. However, obtaining such $\widehat{\bm{\theta}}$ is computationally expensive in general as the empirical error probability might be a non-trivial function of the parameters. 
 
 \textbf{Maximum likelihood estimator}: In practice, the unknown {model} parameters are commonly replaced with their corresponding MLEs under each beahvior; the resulting plug-in classifier {gives rise to the well-known} generalized likelihood ratio test (GLRT) for the binary case ($C=2$) \cite{lehmann2005testing}. Specifically, assuming the training data and their corresponding labels are available in the {form of $\{\mathbf{x}_{r,n},y_{r,n}\}_{n=1}^{N_i}$ for the $i$th hypothesis}, the MLE of $\bm\theta_i$ is obtained by $\widehat{\bm\theta}^{MLE}_i=\argmax_{\bm\theta_i}\mathcal{L}(\mathcal{D}_i|\bm\theta_i)$, where $\mathcal{L}$ denotes the likelihood function. For the binary case where $\frac{\pi p_1(\mathbf{x};{{\bm{\theta}}_1})}{(1-\pi) p_0(\mathbf{x};{{\bm{\theta}}_0})+\pi p_1(\mathbf{x};{{\bm{\theta}}_1})}$ is continuous in $({\bm{\theta}}_0, {\bm{\theta}}_1, \pi)$, as the parameters' estimates converge to the true values, the error of the plug-in classifier also {converges} to that of the Bayes decision rule. However, not only no optimality condition can be stated in general for the plug-in classifier relying on MLEs \cite{GLRT-different-fields}, obtaining such estimates might also be computationally prohibitive for system with complex likelihood functions.
\section{Proposed Solution: \textsf{\textsc{HyPhyLearn}}}
\label{Section:proposed solution}
The main deciding factor in superiority of a solution for the problem setup introduced in Section \ref{section:Problem statement} is the extent to which it exploits the available information, i.e., training data and the parametric statistical models. In particular, the plug-in classifiers tend not to exploit this information in the most optimal fashion as performance of the parameter estimation procedures can be curbed due to the complexity of the underlying models and lack of the corresponding ground-truth labels. We instead propose a novel hybrid classification method to make use of the available information in learning-based classifiers, which are powerful tools for finding discriminative feature spaces. Specifically, our proposed solution relies on the parametric models to generate synthetic data and incorporate them with the training data in a classifier {that makes use of} adversarial training between NNs. Next, we describe the {various steps of the proposed} solution {that is termed \textsf{\textsc{HyPhyLearn}}} in detail.

\textbf{Step $\mathbf{1}$---Imperfect labeling:} As the available data are not {assumed} completely labeled in our problem setup, the first step in our solution {deals} with assigning labels to the unlabeled data samples in $\mathcal{D}_r$. This involves {a} clustering step {that} partitions the dataset $\mathcal{D}_r$ into $C$ distinct groups. Then, the groups are labeled using the available $N_{r,l}$ labels. For example, a label can be assigned to a group based on the number of labeled training data it includes from each behavior; If the majority of such samples corresponds to the $i$th behavior, the group is labeled as $i$. {Subsequently, we refer to a group assigned with the label $i$ by $\mathcal{D}_{r,i}$ for $i=0,\dots,C-1$.} Denoting this {imperfect labeling process} by $g(\mathbf{x}):\mathcal{X}\rightarrow\{0,\dots,C-1\}$, a non-trivial labeling error over $\mathcal{D}_r$ is associated with $g(\mathbf{x})$ that can be computed via $e_r=\frac{1}{N_{r}}\sum_{n=1}^{N_r}\mathds{1}_{\{g(\mathbf{x}_{r,n})\neq y_{r,n}\}}$. In the {remainder of this paper}, we refer to the number of samples in the {cluster} labeled as $i$ by $N_{r,i}$. The function $g(\mathbf{x})$ may be obtained based on {any one of the} simple clustering {algorithms} from {the} ML literature, {such as the} Gaussian mixture model \cite{MLMurphy}, or {it may be a decision rule obtained based on the statistical analysis of the  parametric models}. For instance, for the problem of channel spoofing detection, a hypothesis test is proposed in \cite{CFR1} {that assigns} labels to {unlabeled samples} based on {their} similarity, measured in terms of {the} Euclidean {distance}, to a reference data sample.

\textbf{Step $\mathbf{2}$---Parameter estimation:} Based on the labels assigned in Step $1$ to the {unlabeled} data samples, we estimate the parameters of the physics-based statistical models under each behavior. {To this end}, we utilize $\mathcal{D}_{r,i}$ to estimate the parameter vector $\bm{\theta}^*_i$ corresponding to the $i$th behavior. Furthermore, the priors are estimated {as} $\widehat{\pi}_i=N_{r,i}/N_r$. {We note that the procedure for  estimating $\bm{\theta}^*_i$ depends on the available parametric models corresponding to the $i$th behavior, i.e., $p_i(\mathbf{x};\bm{\theta}_i)$. We recall from our problem setup that {the} MLE, which is usually utilized for parameter estimation purposes, might not be employed here due to the formidable complexity of optimizing $p_i(\mathbf{x};{\bm{\theta}_i})$ over ${\bm{\theta}_i}$. Instead, a (necessarily) suboptimal estimator, $T(\cdot)$, built upon either heuristics or optimization techniques like alternate maximization (see Sections \ref{Section:Spoofing-paramEst} and \ref{section:CDMA-parameterEstimation}) could be utilized to estimate the parameters as $\widehat{\bm{\theta}}_i=T(\mathcal{D}_{r,i})$ for all the behaviors}. The parameter estimation performance is therefore limited here due to both {the} suboptimality of $T(\cdot)$ and presence of the mislabeled samples in $\mathcal{D}_{r,i}, \forall i$.

\textbf{Step $\mathbf{3}$---Forming a synthetic dataset:} The paucity of available data in our problem formulation seems to preclude utilization of a learning-based classifier as part of the solution. However, we note that the available physics-based statistical models, in the form of parametric PDFs, enable us to generate synthetic data to augment the available data, and make it possible to exploit the discriminative power of learning-based classifiers. Having access to the estimated parameter $\widehat{\bm{\theta}}_i$ obtained in Step $2$, we plug it in the available physics-based statistical model to obtain a PDF $p_i(\mathbf{x};{\widehat{\bm{\theta}}}_i)$ for the $i$th behavior. {In order to generate} a synthetic {dataset}, we first sample {$w$ from a categorical distribution parameterized by $\widehat{\bm{\pi}}=[\widehat{\pi}_0,\dots,\widehat{\pi}_{C-1}]$} over the sample space of $\{0,\dots,C-1\}$. Then, we sample a data point $\mathbf{x}_{s,i}$ according to $\mathbf{x}_{s,i}\sim p_w(\mathbf{x};{\widehat{\bm{\theta}}_w})$ with {the} associated label $y_{s,i}=w$. Repeating this process $N_s$ number of times, we obtain a synthetic dataset $\mathcal{D}_s=\{\mathbf{x}_{s,i},y_{s,i}\}_{i=1}^{N_s}$ {in which} the data {samples} are generated in a statistically independent fashion.

\textbf{Step $\mathbf{4}$---Incorporating synthetic and training data in a learning-based classifier:} The synthetic data generated in Step $3$, besides retaining essential information about the underlying physics-based statistical models, enables us to utilize the discriminative power of learning-based classifiers. However, the errors introduced during the labeling and the parameter estimation steps that precede the synthetic data generation process incur a mismatch between the distributions corresponding to the training and synthetic datasets. This mismatch is bound to deteriorate the performance of a classifier trained on the synthetic data alone, when utilized in a real-world setting. Then the question is how a learning-based classifier can be trained to alleviate this problem. For example, {in the fine-tuning approach \cite{meroune}, a NN-based classifier will be trained on the synthetic data first, and then, training data are used to refine the weights of the corresponding NN}. However, we conjecture that such {learning strategies} that utilize the training and synthetic data in the separate stages of training are not the best solution {here;} {rather,} synthetic and training data should jointly be incorporated in a learning-based classifier. To this end, inspired by the works in the domain-adaptation literature and specifically feature space mapping \cite{domain-invariant}, we propose to map the synthetic and training data through a {data-driven} function $M_{\bm{\psi}}:\mathcal{X}\rightarrow\mathcal{Z}$, which is parameterized by a real vector $\bm{\psi}$, into a common feature space $\mathcal{Z}$. Consequently, a classifier $h_{\bm{\phi}_1}(\mathbf{z})$, parameterized by $\bm{\phi}_1$, which is trained on the synthetic data within the space $\mathcal{Z}$ is expected to perform well on both training and synthetic data. To this end, we choose $M_{\bm{\psi}}$ and $h_{\bm{\phi}_1}$ to be NNs, which are powerful tools for finding discriminative features from a given dataset. We discuss this step in detail in the following subsection. 

\textbf{\textsf{\textsc{HyPhyLearn}}:}
We now present our final solution as an algorithmic framework composed of the aforementioned four steps. In a nutshell, \textsf{\textsc{HyPhyLearn}} generates synthetic data based on the physics-based parametric statistical models and {utilizes} them along with {the} available data in a learning-based classifier powered from the adversarial training of the NNs (see the following subsection). {In order to train} the NNs based on their specific {loss functions}, described in the following subsection, we utilize {the} stochastic gradient descent method \cite{MLMurphy} along with mini-batches consisting of {random samples} from the training and synthetic datasets in an iterative manner. The details of the whole process is presented in Algorithm \ref{Alg:mainALG}.
\begin{algorithm}
\small
\SetAlgoLined
\textbf{Input:} Parametric models $p_i(\mathbf{x};{\bm{\theta}_i})$ ($i=0,\dots,C-1$); Training dataset $\mathcal{D}_r=\{\mathbf{x}_{r,n}\}_{n=1}^{N_r}$; learning rates $\mu_{r_1}$, $\mu_{r_2}$, $\mu_{r_3}$; Number of training steps $N_{tr}$; Mini-batch size $N_b<N_r$; Number of synthetic data samples $N_{s}$ to be generated \\
\textbf{Output:} The mapping $M_{\bm{\psi}}(\cdot)$ and the classifier $h_{\bm{\phi}_1}(\cdot)$, parameterized by the real vectors ${\bm{\psi}}$ and ${\bm{\phi}_1}$, respectively\\
\tcp{Step $1$ - Imperfect labeling}
$\{\mathcal{D}_{r,0},\dots,\mathcal{D}_{r,C-1}\}\gets$ Applying $g(\mathbf{x})$ to unlabeled samples\\
\tcp{Step $2$ - Parameter estimation}
$\widehat{\bm{\theta}}_i\gets T_i(\mathcal{D}_{r,i})$, $\widehat{\pi}_i\gets \frac{\lvert\mathcal{D}_{r,i}\rvert}{N_{r}}$ for $i=0,\dots,C-1$ \\
\tcp{Step $3$ - Forming a synthetic dataset}
$p_i(\mathbf{x};{\widehat{\bm{\theta}}_i})\gets$Plug $\widehat{\bm{\theta}}_i$ in $p_i(\mathbf{x};{{{\theta}}_i})$ for $i=0,\dots,C-1$\\
\For{$n=1$ \textbf{to} $N_{s}$}{
\tcp{Choosing a behavior}
$r\sim\text{unif}(0,1)$, 
$w=\argmin_k\sum_{i=0}^{k-1}\widehat{\pi}_i\geq r$\\ \tcp{Synthetic data generation}
$\mathbf{x}_{s,n}{\sim} p_w(\mathbf{x};{\widehat{\bm{\theta}}_w})$, $ y_{s,n}=w$\\
Add $\{\mathbf{x}_{s,n},y_{s,n}\}$ to $\mathcal{D}_s$
}
 \tcp{Step $4$ - Training the learning-based classifier}
\For{$n_{tr}=1$ \textbf{to} $N_{tr}$}{
$\mathcal{D}_{r,b}\gets$ $N_b$ random samples from $\mathcal{D}_r$, $\mathcal{D}_{s,b}\gets$ $N_b$ random samples from $\mathcal{D}_s$\\
\tcp{Forward propagation via (\ref{eq:L_s}),  (\ref{eq:L_c})}
$L_s\gets\mathcal{L}_s(\bm{\psi},\bm{\phi}_1|\mathcal{D}_{s,b})$ \\
$L_c\gets\mathcal{L}_c(\bm{\psi},\bm{\zeta}|\mathcal{D}_{r,b},\mathcal{D}_{s,b})$  \\
\tcp{Backward propagation}
Computing gradients: $\mathcal{G}_{s,\bm{\phi}_1}\gets \nabla_{\bm{\phi}_1}{L}_s$, $\mathcal{G}_{s,\bm{\psi}}\gets \nabla_{\bm{\psi}}{L}_s$\\
Computing gradients: $\mathcal{G}_{c,\bm{\zeta}}\gets \nabla_{\bm{\zeta}}{L}_c$, $\mathcal{G}_{c,\bm{\psi}}\gets \nabla_{\bm{\psi}}{L}_c$\\
\tcp{Update network parameters via (\ref{eq:saddlepoints})}
$\bm{\psi}\gets\bm{\psi}-\mu_{r_1}(\mathcal{G}_{s,\bm{\psi}}-\mathcal{G}_{c,\bm{\psi}})$, $\bm{\phi}_1\gets\bm{\phi}_1-\mu_{r_2}\mathcal{G}_{s,\bm{\phi}_1}$, $\bm{\zeta}\gets\bm{\zeta}-\mu_{r_3}\mathcal{G}_{c,\bm{\zeta}}$}
 \caption{\textsf{\textsc{HyPhyLearn}}}
 \label{Alg:mainALG}
\end{algorithm}
\subsection{Incorporating synthetic and training data in a learning-based classifier {for \textsf{\textsc{HyPhyLearn}}}}
\label{Section:Incorporating Synthetic and Real Data in a Data-driven Classification Model}
To elaborate further on Step $4$, we first denote {the} distributions corresponding to the real and synthetic data as $p_{{{\bm{\theta}^*}}}(\mathbf{x})=\sum_{i=0}^{C-1}\pi_ip_i(\mathbf{x};{{{\bm{\theta}}}^*_i})$ and $p_{\widehat{\bm{\theta}}}(\mathbf{x})=\sum_{i=0}^{C-1}\widehat{\pi}_ip_i(\mathbf{x};{{\widehat{\bm\theta}}_i})$, respectively. We refer to $p_{{{\bm{\theta}^*}}}(\mathbf{x})$ and $p_{\widehat{\bm{\theta}}}(\mathbf{x})$ as the true and estimated distributions, respectively. For each {distribution}, applying the mapping $M_{\bm{\psi}}(\cdot)$ {to the input space $\mathcal{X}$} would induce a distribution over the {feature} space $\mathcal{Z}$. Specifically, we denote the mapping of the true distribution $p_{{{\bm{\theta}^*}}}(\mathbf{x})$ to $\mathcal{Z}$ by $p_{\bm{\psi},{\bm{\theta}^*}}(\mathbf{z})$, where $\mathbf{z}=M_{\bm{\psi}}(\mathbf{x})$, $\mathbf{x}\sim p_{{{\bm{\theta}^*}}}(\mathbf{x})$. Assuming that $\mathcal{X}$ and $\mathcal{Z}$ are topological spaces, for any $\mathcal{A}\subset\mathcal{Z}$ the probability of $\mathcal{A}$ in space $\mathcal{Z}$ is 
\begin{align}
    &\mathbb{P}_{\mathbf{z}}[{\mathcal{A}}]\overset{\bigtriangleup}{=}\mathbb{P}_{\mathbf{x}}\big[M^{-1}_{\bm{\psi}}({\mathcal{A}})\big]=\sum_{i=0}^{C-1}\pi_i\int_{M^{-1}_{\bm{\psi}}({\mathcal{A}})}p_i(\mathbf{x};{\bm{\theta}^*_i})d\mathbf{x},
\end{align}
where the pre-image $M^{-1}_{\bm{\psi}}({\mathcal{A}})$ {belongs to the} Borel $\sigma$-algebra over {$\mathcal{X}$}. Subsequently, the probability of error corresponding to a classifier $h_{\bm{\phi}_1}(\mathbf{z})$, parameterized by a real vector $\bm{\phi}_1$, with respect to the mapping of the true distribution to the $\mathcal{Z}$ space is computed via
\begin{align}
\label{eq:prob-error-Z-training}
    \mathbb{P}_{\bm{\psi},\bm{\theta}^*}[{e}_{{\bm{\phi}}_1}]=\sum_{i=0}^{C-1}\pi_i\int_{\mathcal{Z}}p_{\bm{\psi},\bm{\theta}^*_i}(\mathbf{z})\mathds{1}_{\{h_{{\bm{\phi}}_1}(\mathbf{z})\neq i\}}(\mathbf{z})d\mathbf z,
\end{align}
where the {dependence of $\mathbb{P}$ on $\pi_i$'s} is suppressed for notational simplicity. Similarly, mapping of the estimated distribution to the space $\mathcal{Z}$ is characterized by a distribution denoted by $p_{\bm{\psi},{\widehat{\bm{\theta}}}}(\mathbf{z})$. Furthermore, the probability of error for a classifier $h_{\bm{\phi}_1}(\mathbf{z})$ {with respect to} $p_{\bm{\psi},{\widehat{\bm{\theta}}}}(\mathbf{z})$ can be computed {similar to} \eqref{eq:prob-error-Z-training}, which we refer to as $\mathbb{P}_{\bm{\psi},\widehat{\bm{\theta}}}[{e}_{\bm{\phi}}]$.

{Our} main goal is to {learn} {a map} $M_{\bm{\psi}}(\cdot)$ and a classifier $h_{\bm{\phi}_1}(\mathbf{z})$ in a way that {the} probability of error {of $h_{\bm{\phi}_1}(\mathbf{z})$ with respect to} the mapping of the true distribution to $\mathcal{Z}$, i.e., $\mathbb{P}_{{\bm{\psi},{\bm{\theta}^*}}}[{e}_{\bm{\phi}_1}]$, is small. To this end, we repurpose theories from the domain-adaptation literature in the following to obtain an upper bound on $\mathbb{P}_{{\bm{\psi},{\bm{\theta}^*}}}[{e}_{\bm{\phi}_1}]$, which leads to {explicit loss functions} for the joint learning of $M_{\bm{\psi}}$ and $h_{\bm{\phi}_1}(\mathbf{z})$ {using both} the training and synthetic datasets. {Specifically}, it is desired for the mapping $M_{\bm{\psi}}(\cdot)$ {from $\mathcal{X}$ to $\mathcal{Z}$} to {transform} the true and estimated distributions in a way that $p_{{\bm{\psi},{\bm{\theta}^*}}}(\mathbf{z})$ and $p_{\bm{\psi},{\widehat{\bm{\theta}}}}(\mathbf{z})$, which are defined in the feature space $\mathcal{Z}$, are similar.  Mathematically, this similarity should be measured in terms of a distance metric. However, as there are only a limited number of samples available from $p_{{\bm{\psi},{\bm{\theta}^*}}}(\mathbf{z})$, we need to be able to approximate this distance from a finite {number} of samples. We {expand further on this idea by} primarily {focusing} on binary classification in this section, although the results are extendable to the {classification task} in general. We begin with the following distance definitions.
\begin{definition}
\label{def:A_distance}
For a family of binary-valued functions $\mathcal{H}_{\bm{\Phi}}=\{h_{\bm{\phi}}:\mathcal{Z}\rightarrow\{0,1\}\}$, in which {every} member $h_{\bm{\phi}}\in \mathcal{H}_{\bm{\Phi}}$ is parameterized by a real vector $\bm{\phi}\in \bm{\Phi}$, and the set ${A}_\phi=\{\mathbf{z}|h_\phi(\mathbf{z})=1,\mathbf{z}\in\mathcal{Z}\}$, the $\mathcal{A}_\Phi$-distance between $p_{{\bm{\psi},{\bm{\theta}^*}}}(\mathbf{z})$ and $p_{\bm{\psi},{\widehat{\bm{\theta}}}}(\mathbf{z})$ is defined as
\small
\begin{equation}
\label{eq:A_phi-distance1}
    d_{\mathcal{A}_\Phi}\big(p_{{\bm{\psi},{\bm{\theta}^*}}}(\mathbf{z}), p_{\bm{\psi},{\widehat{\bm{\theta}}}}(\mathbf{z})\big)\overset{\bigtriangleup}{=}2\sup_{h_\phi\in\mathcal{H}_\Phi}\bigg|\int_{A_\phi}(p_{{\bm{\psi},{\bm{\theta}^*}}}(\mathbf{z})- p_{\bm{\psi},{\widehat{\bm{\theta}}}}(\mathbf{z}))d\mathbf{z}\bigg|.
\end{equation}
\normalsize
Similarly, for ${B}_{\phi_1,\phi_2}=\{\mathbf{z}|h_{\phi_1}(\mathbf{z})\neq h_{\phi_2}(\mathbf{z}),\mathbf{z}\in\mathcal{Z}\}$, the $\mathcal{B}_\Phi$-distance refers to \footnote{Similar to the total variation distance, it can be readily verified that $d_{\mathcal{A}_\Phi}$ and $d_{\mathcal{B}_\Phi}$ are also distance metrics.}

\begin{align}
        d_{\mathcal{B}_\Phi}\big(p_{{\bm{\psi},{\bm{\theta}^*}}}(\mathbf{z})&, p_{\bm{\psi},{\widehat{\bm{\theta}}}}(\mathbf{z})\big)\overset{\bigtriangleup}{=}\nonumber\\&2\sup_{h_{\phi_1},h_{\phi_2}\in\mathcal{H}_\Phi}\bigg|\int_{B_{\phi_1,\phi_2}}(p_{{\bm{\psi},{\bm{\theta}^*}}}(\mathbf{z})-p_{\bm{\psi},{\widehat{\bm{\theta}}}}(\mathbf{z}))d\mathbf{z}\bigg|.
\end{align}
\normalsize
\end{definition}
 The $\mathcal{A}_\Phi$-distance is also referred to via other names like $A$-distance and $\mathcal{H}$-distance in \cite{BenDavid,DANN}. By looking at the following {extreme choices} of $\mathcal{H}_{\bm{\Phi}}$, these distances are clearly a function of richness of the class $\mathcal{H}_\Phi$. For a very restrictive choice of only constant functions, i.e., $\mathcal{H}_\Phi=\{h_{\bm{\phi}}|h_{\bm{\phi}}(\mathbf{z})=0, \forall\mathbf{z}\}\bigcup\{h_{\bm{\phi}}|h_{\bm{\phi}}(\mathbf{z})=1, \forall\mathbf{z}\}$, $d_{\mathcal{A}_\Phi}$ is always zero as the only possible choice for $A_{\bm{\phi}}$ is either the empty set or $\mathcal{Z}$. On the other hand, {for $\mathcal{H}_\Phi=\{h_{\bm{\phi}}|h_{\bm{\phi}}(\mathbf{z})=0\ \text{or}\ h_{\bm{\phi}}(\mathbf{z})=1, \forall\mathbf{z}\}$,} which represents all the binary functions, $d_{\mathcal{A}_\Phi}$ is identical to definition of the total variation distance \cite{tv} as the $\sup$ in (\ref{eq:A_phi-distance1}) will effectively be over the {$\sigma$-algebra of subsets of the $\mathcal{Z}$ space}. {This} dependence of $d_{\mathcal{A}_\Phi}$ on the underlying family of functions makes it possible to obtain an expression for the ${\mathcal{A}_\Phi}$-distance based on the finite set of samples from each distribution. Specifically, consider two sets $\mathcal{Z}_{{\bm{\psi},{\bm{\theta}^*}}}=\{\mathbf{z}_{r,i}\}_{i=1}^{N_r}$ and $\mathcal{Z}_{\bm{\psi},{\widehat{\bm{\theta}}}}=\{\mathbf{z}_{s,i}\}_{i=1}^{N_s}$ sampled from the distributions $p_{{\bm{\psi},{\bm{\theta}^*}}}(\mathbf{z})$ and $p_{\bm{\psi},{\widehat{\bm{\theta}}}}(\mathbf{z})$ in an i.i.d. fashion, respectively. In this case, for a family $\mathcal{H}_{\bm{\Phi}}$ that {satisfies the condition that} if $h_\phi\in\mathcal{H}_\Phi$ then $1-h_\phi\in\mathcal{H}_\Phi$, the $\mathcal{A}_\Phi$-distance can be approximated from {$\mathcal{Z}_{{\bm{\psi},{\bm{\theta}^*}}}$ and $\mathcal{Z}_{\bm{\psi},{\widehat{\bm{\theta}}}}$} using \cite{BenDavid}
\small
\begin{align}
\label{eq:empiricalA_distance}
  &\widehat{d}_{\mathcal{A}_\Phi}(\mathcal{Z}_{{\bm{\psi},{\bm{\theta}^*}}},\mathcal{Z}_{\bm{\psi},{\widehat{\bm{\theta}}}})=\nonumber\\&2\bigg(1-\inf_{h_\phi\in\mathcal{H}_\Phi}\Big(\frac{1}{N_r}\sum_{i=1}^{N_r}\mathds{1}_{\{h_\phi(\mathbf{z}_{r,i})=0\}}+\frac{1}{N_s}\sum_{i=1}^{N_s}\mathds{1}_{\{h_\phi({\mathbf{z}}_{s,i})=1\}}\Big)\bigg).
\end{align}
\normalsize

As the bound on $\mathbb{P}_{{\bm{\psi},{\bm{\theta}^*}}}[{e}_{\bm{\phi}_1}]$ should be obtained based on a finite number of training and synthetic samples, {it is then of interest to see how far $\widehat{d}_{\mathcal{A}_\Phi}$ is from  $d_{\mathcal{A}_\Phi}$}. To answer this question, one needs to rely on a measure of complexity for a given class of functions {such as the} VC dimension \cite{pbook} and Rademacher complexity \cite{rademacher}. As we {have chosen} the mapping function $M_{\bm{\psi}}(\mathbf{x})$ and the classifier $h_{\bm{\phi}_1}(\mathbf{z})$ to be NNs, we present the results based on the Rademacher complexity defined as follows, which can be computed for certain classes of neural networks in a closed-form fashion \cite{rademacher}.
\begin{definition}
Let $\mathcal{Z}_1=\{\mathbf{z}_{i}\}_{i=1}^{N}$ be a set of i.i.d. samples {drawn} from {a} distribution $p(\mathbf{z})$ {that is supported on $\mathcal{Z}$}. For $\mathcal{H}_\Phi$, a family of real-valued functions over $\mathcal{Z}$, the \textit{empirical Rademacher complexity} of $\mathcal{H}_\Phi$, given a dataset $\mathcal{Z}_1$, is defined as
\begin{equation}
   R_{\mathcal{Z}_1}(\mathcal{H}_\Phi)\overset{\bigtriangleup}{=} \underset{\underset{i=1,\dots,N}{\sigma_i\sim\{-1,+1\}}}{\mathbb{E}}\Bigg[\sup_{h_\phi\in \mathcal{H}_\Phi}\Bigg(\frac{1}{N}\sum_{i=1}^{N}\sigma_i h_\phi(\mathbf{z}_i)\Bigg)\Bigg],
\end{equation}
where the expectation is over all the $\sigma_i$'s, each taking a binary value with equal probability.
\end{definition}
\begin{lemma}[\!\cite{rademacher}]
\label{lemma:technical1}
Consider a family of functions $\mathcal{H}_{\bm{\Phi}}=\{h_{\bm{\phi}}:\mathcal{Z}\rightarrow\{0,1\}\}$ and a distribution $p(\mathbf{z})$ over $\mathcal{Z}$. For a set $\mathcal{Z}_1=\{\mathbf{z}_{i}\}_{i=1}^{N}$ of $N$ i.i.d. samples from $p(\mathbf{z})$ and any $0 < \delta < 1$, the following holds $\forall h_{\bm{\phi}}\in\mathcal{H}_{\Phi}$ {with probability at least $1-\delta$:}
\begin{align}
    \mathbb{E}_{\mathbf{z}\sim p(\mathbf{z})}[h_{\bm{\phi}}(\mathbf{z})]\leq \frac{1}{N}\sum_{i=1}^N h_{\bm{\phi}}(\mathbf{z}_{i})+2R_{\mathcal{Z}_1}(\mathcal{H}_{\Phi})+3\sqrt{\frac{\log(2/\delta)}{2N}}.
\end{align}
\end{lemma}

Now, the difference between $d_{\mathcal{A}_\Phi}$ and $\widehat{d}_{\mathcal{A}_\Phi}$ can be bounded in terms of the complexity of the underlying family of functions and the number of available samples as stated in the following lemma.
\begin{lemma}
\label{lemma:Rademachar-A_distance}
Let $\mathcal{Z}_{{\bm{\psi},{\bm{\theta}^*}}}=\{\mathbf{z}_{r,i}\}_{i=1}^{N_r}$ and $\mathcal{Z}_{\bm{\psi},{\widehat{\bm{\theta}}}}=\{\mathbf{z}_{s,i}\}_{i=1}^{N_s}$ be sets of i.i.d. samples corresponding to the distributions $p_{{\bm{\psi},{\bm{\theta}^*}}}(\mathbf{z})$ and $p_{\bm{\psi},{\widehat{\bm{\theta}}}}(\mathbf{z})$ on the space $\mathcal{Z}$, respectively. Then, for any $0<\delta<1$ and a family of functions $\mathcal{H}_{\bm{\Phi}}=\{h_{\bm{\phi}}:\mathcal{Z}\rightarrow\{0,1\}\}$, we have
\small
\begin{align}
    d_{\mathcal{A}_\Phi}&\big(p_{{\bm{\psi},{\bm{\theta}^*}}}(\mathbf{z}),p_{\bm{\psi},{\widehat{\bm{\theta}}}}(\mathbf{z})\big)\leq  \widehat{d}_{\mathcal{A}_\Phi}(\mathcal{Z}_{{\bm{\psi},{\bm{\theta}^*}}},\mathcal{Z}_{\bm{\psi},{\widehat{\bm{\theta}}}})+2R_{\mathcal{Z}_{{\bm{\psi},{\bm{\theta}^*}}}}(\mathcal{H}_\Phi)\nonumber\\&+2R_{\mathcal{Z}_{\bm{\psi},{\widehat{\bm{\theta}}}}}(\mathcal{H}_\Phi)+3\sqrt{(\log{2/\delta})/2N_r}+3\sqrt{(\log{2/\delta})/2N_s}
\end{align}
\normalsize
with probability at least $1-\delta$.
\end{lemma}
\begin{proof}
See Appendix \ref{Appen:lemma:Rademachar-A_distance}.
\end{proof}

The above lemma enables us to bound the $\mathcal{A}_{\bm{\Phi}}$ distance between two distributions in terms of the collected samples from each. Equipped with this result, we are able to bound the probability of error $\mathbb{P}_{{\bm{\psi},{\bm{\theta}^*}}}[{e}_{\bm{\phi}_1}]$ via the following theorem.
\begin{theorem}
\label{theorem:mainbound}
Assume that the training and synthetic datasets are mapped {into the feature space $\mathcal{Z}$} through the mapping function $M_{\bm{\psi}}(\mathbf{x})$, {with} the resulting samples denoted by $\mathcal{Z}_{{\bm{\psi},{\bm{\theta}^*}}}=\{\mathbf{z}_{r,i}\}_{i=1}^{N_r}$ and $\mathcal{Z}_{\bm{\psi},{\widehat{\bm{\theta}}}}=\{\mathbf{z}_{s,i}\}_{i=1}^{N_s}$, respectively. Then, for any $0<\delta<1$ and a family of functions $\mathcal{H}_{\bm{\Phi}}:\mathcal{Z}\rightarrow\{0,1\}$, {$\mathbb{P}_{{\bm{\psi},{\bm{\theta}^*}}}[{e}_{\bm{\phi}_1}], \forall h_{\bm{\phi}_1}\in\mathcal{H}_{\bm{\Phi}}$} is bounded by
\small
\begin{align}
    \mathbb{P}&_{{\bm{\psi},{\bm{\theta}^*}}}[{e}_{\bm{\phi}_1}]\leq
    \mathbb{P}_{\bm{\psi},{\widehat{\bm{\theta}}}}[{e}_{\bm{\phi}_1}]+\frac{1}{2}\widehat{d}_{\mathcal{A}_\Phi}(\mathcal{Z}_{{\bm{\psi},{\bm{\theta}^*}}},\mathcal{Z}_{\bm{\psi},{\widehat{\bm{\theta}}}})+R_{\mathcal{Z}_{{\bm{\psi},{\bm{\theta}^*}}}}(\mathcal{H}_{\bm{\Phi}})+\nonumber\\&R_{\mathcal{Z}_{\bm{\psi},{\widehat{\bm{\theta}}}}}(\mathcal{H}_{\bm{\Phi}})+\frac{3}{2}\sqrt{(\log{2/\delta})/2N_r}+\frac{3}{2}\sqrt{(\log{2/\delta})/2N_s)}.
\end{align}
\normalsize
\end{theorem}
\begin{proof}
See Appendix \ref{Appen:theorem:mainbound}.
\end{proof}

The above theorem bounds the probability of error {with respect to} $p_{\bm{\psi},\bm{\theta}^*}(\mathbf{z})$ associated with a classifier $h_{\bm{\phi}_1}(\mathbf{z})$ in terms of the quantities that {do} not depend on the the unknown true parameters $\bm{\theta}^*$. As our primary goal is to make $ \mathbb{P}_{{\bm{\psi},{\bm{\theta}^*}}}[{e}_{\bm{\phi}_1}]$ as small as possible, the mapping function $M_{\bm{\psi}}(\mathbf{x})$ and the classifier $h_{\bm{\phi}_1}(\mathbf{z})$ should be chosen in a way to minimize the above upper bound. We note that the complexity related terms {in the above bound} are fixed for a chosen family of the functions and the bound is primarily controlled by the first two terms. In other words, $M_{\bm{\psi}}(\mathbf{x})$ and $h_{\bm{\phi}_1}(\mathbf{z})$ should be chosen such that the probability of classification error {with respect to} the mapping of the estimated distribution in the $\mathcal{Z}$ space, i.e., $\mathbb{P}_{\bm{\psi},{\widehat{\bm{\theta}}}}[{e}_{\bm{\phi}_1}]$, and {the} approximated $\mathcal{A}_{\bm{\Phi}}$-distance between the synthetic and training datasets are minimized simultaneously. To achieve this goal, we {restrict ourselves to} $M_{\bm{\psi}}(\mathbf{x})$ and $h_{\bm{\phi}_1}(\mathbf{z})$ {that correspond to} NNs {that} are trained to minimize a loss function in accordance with the first two terms of the above bound. One can efficiently solve {the resulting optimization} problem via {the} stochastic gradient descent method as described in the following.

{\textbf{Joint learning of the feature map and the classifier:}} In terms of specifics, we assume $M_{\bm{\psi}}(\mathbf{x})$ and $h_{\bm{\phi}_1}(\mathbf{z})$ belong {to} the class of feed-forward (deep) NNs whose parameters, i.e., $\bm{\psi}$ and $\bm{\phi}_1$, correspond to the weights and biases of each network. The input and output layers of the NNs corresponding to $M_{\bm{\psi}}$ have $n_x$ and $n_z$ number of neurons, respectively, which denote the dimensions of the spaces $\mathcal{X}$ and $\mathcal{Z}$, respectively. We note that $n_x$ is chosen according to the length of the observation vector as part of the problem formulation, while $n_z$ can be picked as a hyper-parameter to facilitate the training process. Subsequently, the input layer of $h_{\bm{\phi}_1}$ has $n_z$ neurons while its output layer contains $C$ neurons whose activation function is chosen to be the softmax function $\bm{\sigma}(\mathbf{z})$ for which the $i$th element is given by $\frac{e^{\mathbf{z}[i]}}{\sum_{i=1}^{n_z} e^{\mathbf{z}[i]}}$. In this way, the $i$th component of the vector
${\mathbf{y}}_{\bm{\psi},\bm{\phi}_1,\mathbf{x}}\overset{\bigtriangleup}{=}h_{\bm{\phi}_1}\big(M_{\bm{\psi}}(\mathbf{x})\big)$ denotes the probability that the classifier assigns to the input $\mathbf{x}$ {that it belongs} to the $i$th class for $i=0,\dots,C-1$. Consequently, the averaged cross-entropy loss, minimizing of which leads to minimizing the classification error associated with $h_{\bm{\phi}_1}$, over the synthetic dataset $\mathcal{D}_s$ equals
\begin{align}
\label{eq:L_s}
    \mathcal{L}_s({\bm{\psi},\bm{\phi}_1}|\mathcal{D}_s)=\frac{1}{n_s}\sum_{n=1}^{n_s}\sum_{i=1}^{C} \mathbf{l}_{s,n}[n]\log{{\mathbf{y}}_{\bm{\psi},\bm{\phi}_1,\mathbf{x}_{s,n}}[n]},
\end{align}
 where {$\mathbf{l}_{s,n}=\mathbf{e}_{C}({y}_{s,n})$ denotes the one-hot encoded version of the label ${y}_{s,n}$ corresponding to the $n$th sample}. Regrading the computation of $\widehat{d}_{\mathcal{A}_\Phi}$ between the two sets $\mathcal{Z}_{{\bm{\psi},{\bm{\theta}^*}}}$ and $\mathcal{Z}_{\bm{\psi},{\widehat{\bm{\theta}}}}$, it is suggested by the authors in \cite{BenDavid,DANN} that the classification accuracy corresponding to a classifier trained to distinguish between the samples from the two sets can be used as a surrogate for the {$\inf$ part} in (\ref{eq:empiricalA_distance}) that can be readily computed during the learning process. To train such classifier, we consider a NN $d_{\bm{\zeta}}$ with $n_z$ input neurons and $2$ output neurons with softmax activation function, which is trained to distinguish between $\mathcal{Z}_{{\bm{\psi},{\bm{\theta}^*}}}$ and $\mathcal{Z}_{\bm{\psi},{\widehat{\bm{\theta}}}}$ labeled as $0$ and $1$, respectively. Consequently, by defining a two-dimensional vector $d_{\bm{\psi},{\bm{\zeta}},\mathbf{x}}\overset{\bigtriangleup}{=}d_{\bm{\zeta}}\big(M_{\bm{\psi}}(\mathbf{x})\big)$, the $\widehat{d}_{\mathcal{A}_\Phi}$ term can be {approximated} by the cross-entropy loss associated with $d_{\bm{\zeta}}$ as follows:
  \begin{align}
    &\mathcal{L}_d({\bm{\psi},\bm{\zeta}}|\mathcal{D}_s,\mathcal{D}_r)=2\big(1-2\mathcal{L}_c({\bm{\psi},\bm{\zeta}}|\mathcal{D}_s,\mathcal{D}_r)\big),\label{eq:L_d}\\
    &\mathcal{L}_c({\bm{\psi},\bm{\zeta}}|\mathcal{D}_s,\mathcal{D}_r)=\frac{1}{n_r}\sum_{i=1}^{n_r}\log{d_{\bm{\psi},{\bm{\zeta}},\mathbf{x}_{r,n}}[1]}+\nonumber\\&\ \ \ \ \ \ \ \ \ \ \  \ \ \ \ \ \ \ \ \ \ \ \ \ \ \ \ \ \ \ \ \ \ \ \ \ \ \frac{1}{n_s}\sum_{n=1}^{n_s}\log{d_{\bm{\psi},{\bm{\zeta}},\mathbf{x}_{s,n}}[2]} \label{eq:L_c}.
\end{align}

Now, using Theorem \ref{theorem:mainbound} the training goal for the constituent NNs is set to simultaneously minimize the classification error corresponding to the synthetic data and the distance between the real and synthetic data, both measured in the mapped space $\mathcal{Z}$. Specifically, the NNs $M_{\bm{\psi}}$ and $h_{\bm{\phi}_1}$ should be trained to minimize the sum of the losses in (\ref{eq:L_s}) and (\ref{eq:L_d}), while the classifier $d_{\bm{\zeta}}$ is trained to minimize (\ref{eq:L_c}). As $M_{\bm{\psi}}$ is trained to maximize $\mathcal{L}_c({\bm{\psi},\bm{\zeta}})$ despite $d_{\bm{\zeta}}$'s goal to minimize $\mathcal{L}_c({\bm{\psi},\bm{\zeta}})$, the learning process involves adversarial training between these two NNs. Based on the approach taken in \cite{DANN} for adversarial training in the context of domain adaptation, we train the above three NNs for finding the saddle points $\widehat{\bm{\psi}}$, $\widehat{\bm{\phi}}_1$ and $\widehat{\bm{\zeta}}$, such that
\begin{align}
\small
\label{eq:saddlepoints}
&\widehat{\bm{\psi}}, \widehat{\bm{\phi}}_1=\argmin_{{\bm{\psi}}, {\bm{\phi}}_1} \mathcal{L}_t({\bm{\psi},\bm{\phi}_1,\widehat{\bm{\zeta}}}|\mathcal{D}_s,\mathcal{D}_r),\\& \widehat{\bm{\zeta}}=\argmin_{{\bm{\zeta}}}- \mathcal{L}_t({\widehat{\bm{\psi}},\widehat{\bm{\phi}}_1,\bm{\zeta}}|\mathcal{D}_s,\mathcal{D}_r),\\
&\mathcal{L}_t({\bm{\psi},\bm{\phi}_1,\bm{\zeta}}|\mathcal{D}_s,\mathcal{D}_r)=\mathcal{L}_s({\bm{\psi},\bm{\phi}_1}|\mathcal{D}_s)+\mathcal{L}_d({\bm{\psi},\bm{\zeta}}|\mathcal{D}_s,\mathcal{D}_r),
\end{align}
\normalsize
 which {can be} achieved {by} utilizing the stochastic gradient descent algorithm for each minimization task. To this end, the minimization is performed over the NN's parameters, $\bm{\psi}$, $\bm{\phi}_1$ and $\bm{\zeta}$, that are real vectors whose dimensions are determined by the architecture of each network.

\subsection{{An illustrative example: The case of two-dimensional Gaussian data}}
\label{section:Two-dimensional Gaussian}
Next, we show how the learning-based classifier in Section~\ref{Section:Incorporating Synthetic and Real Data in a Data-driven Classification Model} performs on simple training and synthetic datasets in an illustrative manner. To this end, we consider a toy example where the true and estimated distributions are a mixture of two bivariate Gaussian distributions with full-rank covariance matrix each. In particular, we focus on the problem of binary classification where the distribution for the $i$th class is denoted by $p_i(\mathbf{x};{{{\bm{\theta}}}^*_i})=\mathcal{N}(\bm{\mu}_i,\bm{\Sigma})$ for $i=0, 1$, $\bm{\mu}_i\in\mathbb{R}^{2\times1}$, $\bm{\Sigma}\in\mathbb{R}^{2\times2}$, and equal priors. In order to investigate the effect of mismatch between only mean parameters, the corresponding estimated distributions are assumed to have the same {covariance} but different means, i.e., $p_i(\mathbf{x};{\widehat{\bm{\theta}}}_i)=\mathcal{N}(\widehat{\bm{\mu}}_i,\bm{\Sigma})$ for $i=0, 1$ and equal priors. For two multivariate Gaussian distributions, the authors in \cite{tv} have proposed a bound for the corresponding total variation as part of the following theorem.
\begin{figure*}[h]
  \begin{subfigure}{0.34\textwidth}
      \captionsetup{width=0.95\textwidth}
    \includegraphics[width=\linewidth]{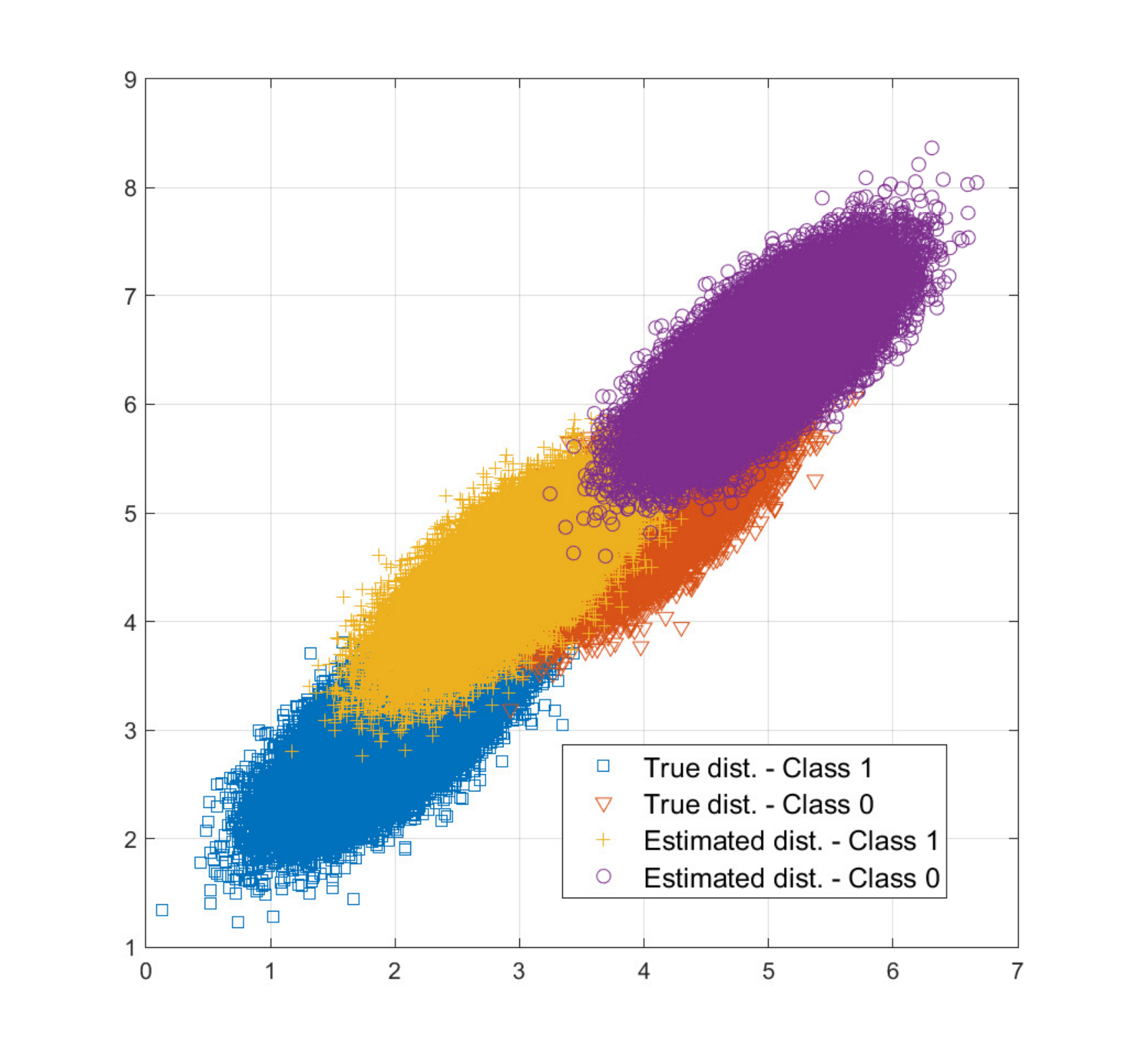}
    \caption{Samples corresponding to the true and estimated distributions in space $\mathcal{X}$.} \label{fig:Gaussian2D-samples-a}
  \end{subfigure}%
  \hspace*{\fill}   
  \begin{subfigure}{0.34\textwidth}
    \captionsetup{width=0.95\textwidth}
    \includegraphics[width=\linewidth]{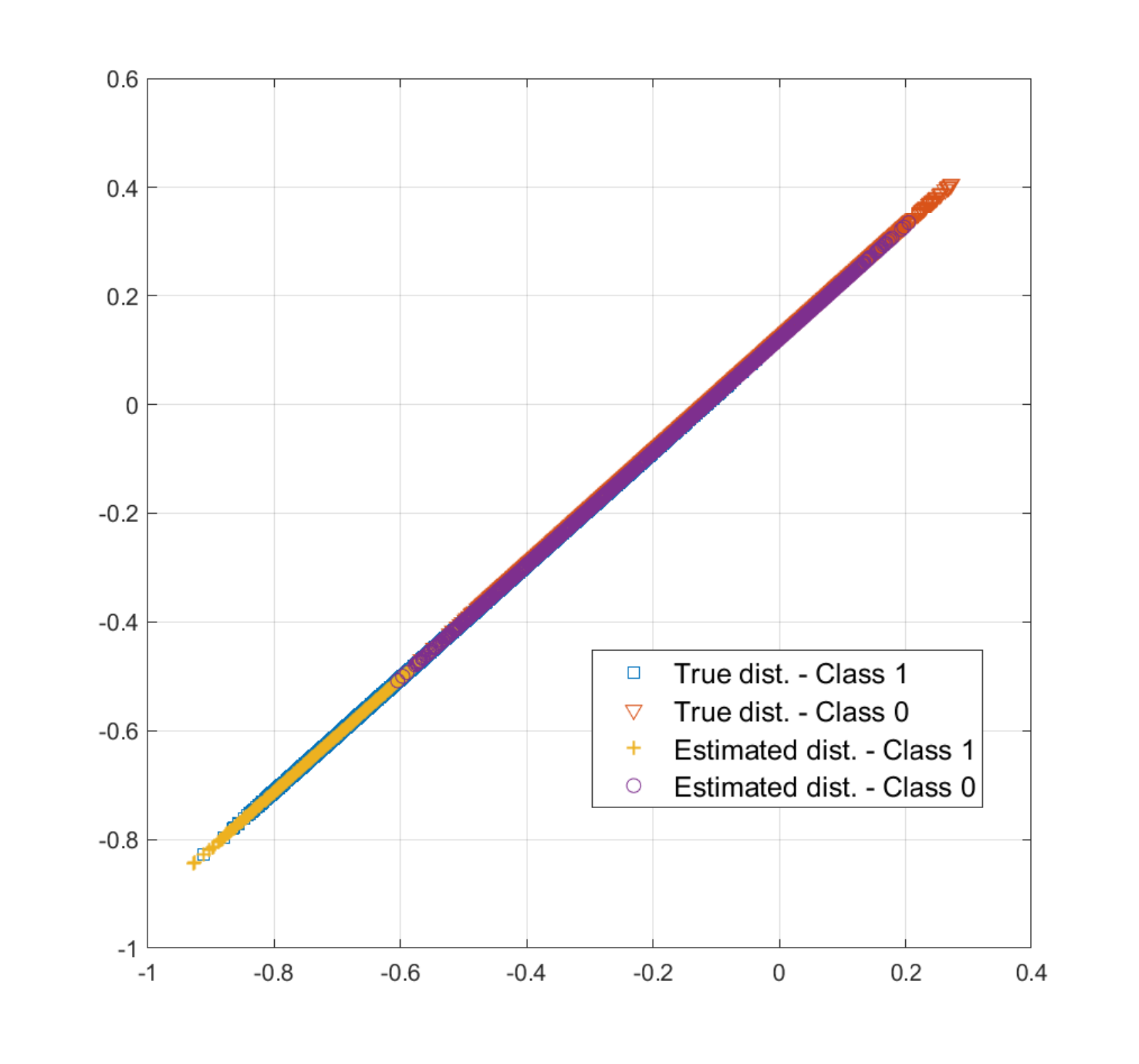}
    \caption{Mapping of the samples via the function $M_{\bm{\psi}}$ to the space $\mathcal{Z}.$} \label{fig:Gaussian2D-samples-b}
  \end{subfigure}%
  \hspace*{\fill}   
  \begin{subfigure}{0.34\textwidth}
      \captionsetup{width=0.95\textwidth}
    \includegraphics[width=\linewidth]{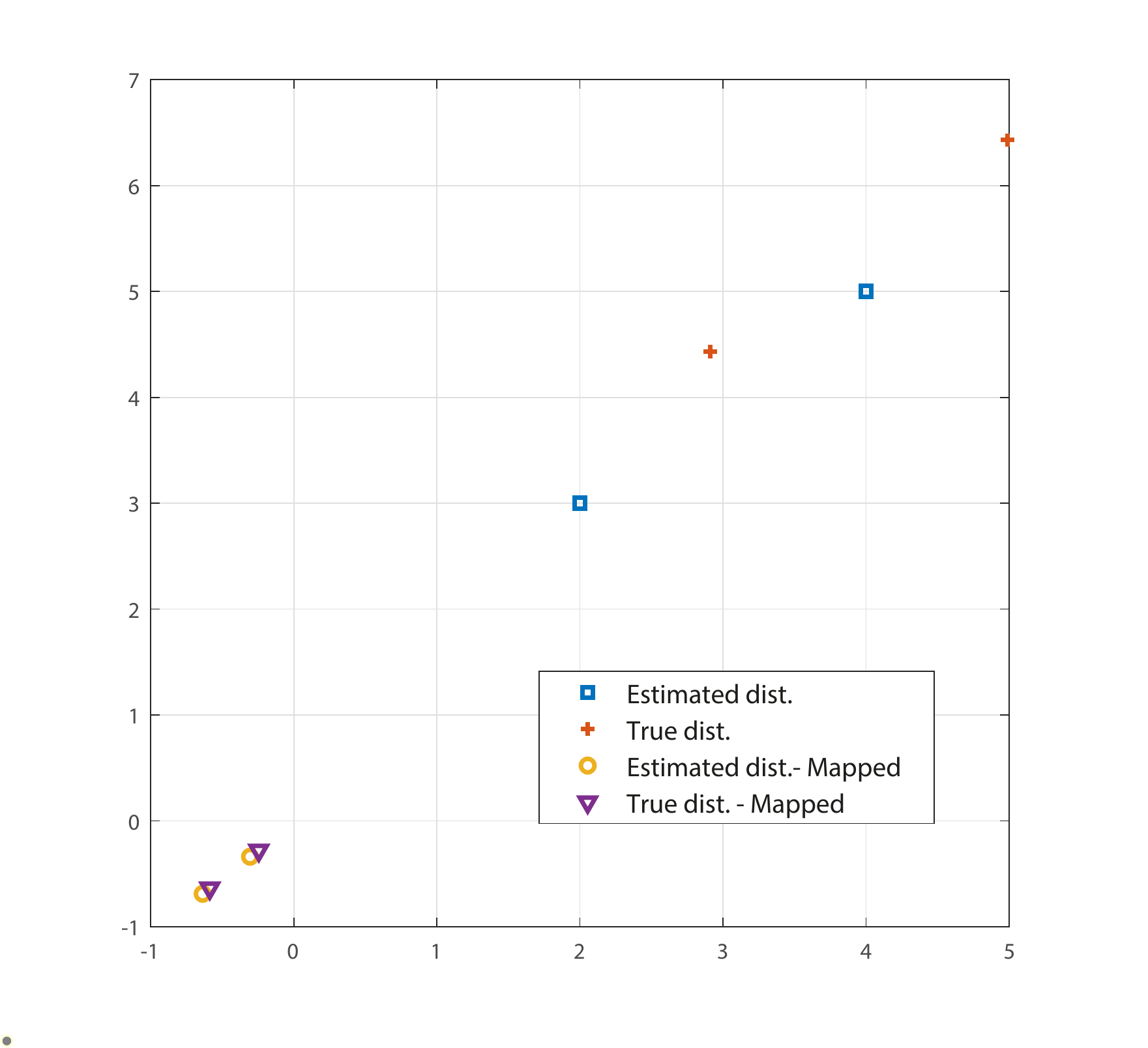}
    \caption{Position of the means in the original space $\mathcal{X}$ and the space $\mathcal{Z}$.} \label{fig:Gaussian2D-samples-c}
  \end{subfigure}
  \captionsetup{width=0.94\textwidth}
\caption{Visualization of the true and estimated distributions and their mappings to the space $\mathcal{Z}$ for the case of $2$D Gaussian {datasets}.} \label{fig:Gaussian2D-samples}
\end{figure*}
\begin{theorem}[\!\cite{tv}] 
Consider two $d$-dimensional Gaussian distributions $\mathcal{N}({{\bm{\mu}}}_1,{\bm{\Sigma}}_1)$ and $\mathcal{N}({{\bm{\mu}}}_2,{\bm{\Sigma}}_2)$ where ${{\bm{\mu}}}_1\neq {{\bm{\mu}}}_2$ and $\bm{\Sigma}_1$ and $\bm{\Sigma}_2$ are positive definite. Let $\mathbf{v}={{\bm{\mu}}}_1-{{\bm{\mu}}}_2$ and $\mathbf{\Pi}$ be a $d\times (d-1)$ matrix whose columns form a basis for the subspace orthogonal to $\mathbf{v}$. Denote the eigenvalues of $(\mathbf{\Pi}^T{\bm{\Sigma}}_1\mathbf{\Pi})^{-1}\mathbf{\Pi}^T{\bm{\Sigma}}_2\mathbf{\Pi}-\mathbf{I}_{d-1}$ by $\rho_1,\dots,\rho_{d-1}$. Then, {the} total variation between the two distribution can be bounded as
\begin{align}
    \frac{1}{200}\leq\frac{TV(\mathcal{N}({{\bm{\mu}}}_1,{\bm{\Sigma}}_1),\mathcal{N}({{\bm{\mu}}}_2,{\bm{\Sigma}}_2))}{\min(1,V)}\leq\frac{9}{2}
\end{align}
where 
$
    V\stackrel{\text{def}}{=}\max\Bigg\{\frac{|\mathbf{v}^T({\bm{\Sigma}}_1-{\bm{\Sigma}}_2)\mathbf{v}|}{\mathbf{v}^T{\bm{\Sigma}}_1\mathbf{v}},\frac{\mathbf{v}^T\mathbf{v}}{\sqrt{\mathbf{v}^T{\bm{\Sigma}}_1\mathbf{v}}},\sqrt{\sum_{i=1}^{d-1}\rho_i^2}\Bigg\}.
$
\end{theorem}

We note that a bound on total variation would also bound the $\mathcal{A}_\Phi$ distance following the discussion after Definition \ref{def:A_distance}. Using the above result, we can bound the total variation distance between $p_i(\mathbf{x};{{{\bm{\theta}}}^*_i})$ and $p_i(\mathbf{x};{\widehat{\bm{\theta}}}_i)$ as follows, which will provide useful insights in the {remainder} of this section about the learning process described in Section \ref{Section:Incorporating Synthetic and Real Data in a Data-driven Classification Model}.
\begin{corollary}
\label{corollary:TVtwoGaussian_equal_variance}
For two Gaussian distributions $\mathcal{N}({{\bm{\mu}}}_0,{\bm{\Sigma}})$ and $\mathcal{N}(\widehat{{{\bm{\mu}}}}_0,{\bm{\Sigma}})$ with the same positive definite covariance matrix ${\bm{\Sigma}}$, the corresponding total variation is bounded from the above by
$
   \frac{9}{2}\min\Big(1,\frac{({{\bm{\mu}}}_0-\widehat{{{\bm{\mu}}}}_0)^T({{\bm{\mu}}}_0-\widehat{{{\bm{\mu}}}}_0)}{\sqrt{({{\bm{\mu}}}_i-\widehat{{{\bm{\mu}}}}_0)^T{\bm{\Sigma}}({{\bm{\mu}}}_0-\widehat{{{\bm{\mu}}}}_0)}}\Big).
$
\end{corollary}

Regarding the specific architecture for the NNs utilized in Section \ref{Section:Incorporating Synthetic and Real Data in a Data-driven Classification Model}, {let us now} choose the mapping function $M_{\bm{\psi}}$ to be $M_{\bm{\psi}}(\mathbf{x})=\mathbf{W}_{\bm{\psi},2}\mathbf{W}_{\bm{\psi},1}\mathbf{x}$ parameterized by $\bm{\psi}=\{\mathbf{W}_{\bm{\psi},1}\in\mathbb{R}^{20\times 2},\mathbf{W}_{\bm{\psi},2}\in\mathbb{R}^{2\times 20}\}$. In particular, we have set the dimension of the space $\mathcal{Z}$ to $n_z=2$ in order to be able to readily visualize it within the $2$D coordinate system. For each $h_{\bm{\phi}_1}$ and $d_{\bm{\zeta}}$, we choose a two-layer NN with softmax activation function. Specifically, for $h_{\bm{\phi}_1}$ we have $h_{\bm{\phi}_1}\big(M_{\bm{\psi}}(\mathbf{x})\big)=\text{softmax}\big(\mathbf{V}_{\bm{\phi}_1,2}(\mathbf{V}_{\bm{\phi}_1,1}M_{\bm{\psi}}(\mathbf{x})+\mathbf{b}_{\bm{\phi}_1,1})+\mathbf{b}_{\bm{\phi}_1,2}\big)$ where $\bm{\phi}_1=\{\mathbf{V}_{\bm{\phi}_1,1}\in\mathbb{R}^{20\times {2}},\mathbf{b}_{\bm{\phi}_1,1}\in\mathbb{R}^{20},\mathbf{V}_{\bm{\phi}_1,2}\in\mathbb{R}^{2\times 20},\mathbf{b}_{\bm{\phi}_1,2}\in\mathbb{R}^{2}\}$. Similarly, $d_{\bm{\zeta}}$ is chosen to be $d_{\bm{\zeta}}\big(M_{\bm{\psi}}(\mathbf{x})\big)=\text{softmax}\big(\mathbf{U}_{\bm{\phi}_1,2}(\mathbf{U}_{\bm{\phi}_1,1}M_{\bm{\psi}}(\mathbf{x})+\mathbf{b}_{\bm{\phi}_1,1})+\mathbf{b}_{\bm{\phi}_1,2}\big)$ for $\bm{\zeta}=\{\mathbf{U}_{\bm{\phi}_1,1}\in\mathbb{R}^{20\times {2}},\mathbf{b}_{\bm{\phi}_1,1}\in\mathbb{R}^{20},\mathbf{U}_{\bm{\phi}_1,2}\in\mathbb{R}^{2\times 20},\mathbf{b}_{\bm{\phi}_1,2}\in\mathbb{R}^{2}\}$. Training of these NNs involves finding the saddle points of (\ref{eq:saddlepoints}) based on the available training and synthetic datasets which would lead to the learning-based classifier $h_{\bm{\phi}_1}$. We note that the above simple choice of the mapping function maps $p_i(\mathbf{x};\mathbf{\theta}_i^*),\ i=0,\ 1$ to Gaussian distributions in the $\mathcal{Z}$ space which allows us to utilize Corollary \ref{corollary:TVtwoGaussian_equal_variance} for analyzing the total variation distance between these mappings in the following.

{We now resort to numerical results for further illustration of this example.} {To this end}, we set ${{\bm{\mu}}}_0=[2.9,4.4]$, ${{\bm{\mu}}}_1=[5,6.4]$, $\widehat{{{\bm{\mu}}}}_0=[2,3]$, $\widehat{{{\bm{\mu}}}}_1=[4,5]$ and ${\bm{\Sigma}}=\begin{bmatrix}0.15&0.11\\0.11&0.15\end{bmatrix}$. Also, we {generate} $n_r=40$ samples from the true distribution, while $n_s=2000$ samples are generated from the estimated distribution. The Figs. \ref{fig:Gaussian2D-samples-a} and \ref{fig:Gaussian2D-samples-b} depict the samples from the true and estimated distributions and their mapping through the function $M_{\bm{\psi}}$ into the $\mathcal{Z}$ space, respectively. Furthermore, the positions of the means corresponding to the samples from the real and estimated distributions in both space $\mathcal{X}$ and $\mathcal{Z}$ are illustrated in Fig.~\ref{fig:Gaussian2D-samples-c}. An important observation in relation to the Corollary \ref{corollary:TVtwoGaussian_equal_variance} can be made by noting that the total variation between $\mathcal{N}({{\bm{\mu}}}_i,{\bm{\Sigma}})$ and $\mathcal{N}(\widehat{{{\bm{\mu}}}}_i,{\bm{\Sigma}})$ is bounded by the term $||\mathbf{v}||\frac{\mathbf{e}_v^T\mathbf{e}_v}{\sqrt{\mathbf{e}_v^T{\bm{\Sigma}}\mathbf{e}_v}}$ where $\mathbf{v}=\bm{\mu}_i-\widehat{\bm{\mu}}_i$ and $\mathbf{e}_{v}=\mathbf{v}/||\mathbf{v}||$. Assuming $\lambda_1$ and $\lambda_2$ are eigenvalues of ${\bm{\Sigma}}$ with corresponding eigenvectors $\mathbf{u}_1$ and $\mathbf{u}_2$ such that $\lambda_1>\lambda_2$, it is straightforward to show that the maximal value of $\mathbf{e}_v^T{\bm{\Sigma}}\mathbf{e}_v=\lambda_1(\mathbf{u}_1^T\mathbf{e}_v)+\lambda_2(\mathbf{u}_2^T\mathbf{e}_v)$ is achieved when $\mathbf{e}_{v}\perp\mathbf{u}_2$. Therefore, for $||\mathbf{v}||\frac{\mathbf{e}_v^T\mathbf{e}_v}{\sqrt{\mathbf{e}_v^T{\bm{\Sigma}}\mathbf{e}_v}}$ to be minimized $\mathbf{e}_v$ ought to be in the same direction of $\mathbf{u}_1$ while $||\mathbf{v}||$ become minimum. Notably, Figs. \ref{fig:Gaussian2D-samples-b} and \ref{fig:Gaussian2D-samples-c} highlight the fact that finding the saddle points in (\ref{eq:saddlepoints}) in part {corresponds} to mapping the datasets to a feature space $\mathcal{Z}$ {that} satisfy both these two criteria. 

\section{Case study 1: Channel-based Spoofing Detection for Physical Layer Security}
\label{Section:Case study 1: Spoofing detection via channel frequency response}
We now present the first case study concerning channel spoofing detection, which can be posed as a binary hypothesis testing problem. We obtain the likelihood ratio test based on the time-variant channel models. Then, we discuss parameter estimation procedures for the likelihood function corresponding to channel frequency responses (CFRs), and show how our proposed solution is applicable for achieving an enhanced spoofing detection system. 
\subsection{System Model}
The problem of channel spoofing detection arises in a wireless communication environment where a legitimate transmitter (Alice) is transmitting signals to a legitimate receiver (Bob) in the presence of an adversary (Eve). Eve aims at spoofing the Alice-Bob's channel by using the Alice's MAC address \cite{CFR1,CFR2}. Bob's goal, in this setting, is to distinguish between the signals coming from Alice and Eve based on the corresponding channel frequency responses (CFRs).  

We envision the communication parties in a $5$G propagation setting relying on MIMO-OFDM wideband communications where the number of antennas are set to $N_{Tx}$ and $N_{Rx}$ at the transmitter (Tx) and the receiver (Rx), respectively. We assume bob measures and stores channel frequency response samples at $M$ tones, across an overall system bandwidth of $W$. We consider a generalized time-variant channel model, where each measured frequency response sample is made up of three components: $1)$ specular paths ($\overline{\mathbf{h}}$), $2)$ time-varying part ${\mathbf{d}}_u$, and $3)$ noise ${\mathbf{n}}$, which are complex vectors of size $M\times1$. The specular paths model the dominant portion of the channel which remains unchanged within a coherence time. The time-varying part models the dense multipath components which accounts for the diffuse scattering between two transceivers. Finally, the noise part models the measurement noise. The measured CFR at Bob at time $t=uT$, $u\in \mathbb{N}$, is denoted by $\mathbf{h}_k$ which is a $M\times 1$ vector such that
\begin{align}
\label{eq:full-channel}
    \mathbf{h}_u=\overline{\mathbf{h}}+{\mathbf{d}}_u+{\mathbf{n}}.
\end{align}

We first introduce the dominant paths model suitable for MIMO-OFDM communication, under a frequency-dependent array response \cite{arrayresponse}. For this scenario, the $N_{Rx}\times N_{Tx}$ channel matrix associated with the $n$th subcarrier ($n=1,\dots,N_f$) is expressed as
\begin{align}
    \mathbf{H}[n]=\mathbf{A}_{R}[n]\mathbf{\Gamma}[n]\mathbf{A}^H_{T}[n].
\end{align}
The antenna steering response vectors are defined as
\begin{align}
    &\mathbf{A}_{T}[n]=[\mathbf{a}_{T,n}(\psi_{T,0}),\dots,\mathbf{a}_{T,n}(\psi_{T,K-1})],\\
    &    \mathbf{A}_{R}[n]=[\mathbf{a}_{Rx,n}(\psi_{R,0}),\dots,\mathbf{a}_{R,n}(\psi_{Rx,K-1})],
\end{align}
where $K$ is the total number of dominant paths, $\psi_{T,k}=\frac{2\pi}{\lambda_n}d\sin(\theta_{Tx,k})$, $\lambda_n=c(NT_s+f_c)/n$ denotes the signal bandwidth at the $n$th subcarrier, and $d$ refers to the distance between two antenna elements. The structure of the frequency-dependent antenna steering and response vectors $\mathbf{a}_{T,n}(\psi_{T,K-1})$
and $\mathbf{a}_{R,n}(\psi_{R,K-1})$ depends on the specific array structure. For the case of a uniform linear array
(ULA) which we consider in this work, we have
\begin{align}
    \mathbf{a}_{T,n}(\psi_{T,k})=\frac{1}{N_{Tx}}[e^{-j\frac{N_{Tx}-1}{2}\psi_{T,k}},\dots,e^{j\frac{N_{Tx}-1}{2}\psi_{T,k}}],
\end{align}
Similarly, $\mathbf{a}_{Rx,n}(\psi_{Rx,k})$ can be defined for the receiver's antennas. The path gain matrix is obtained by
\begin{align}
    \mathbf{\Gamma}[n]&=\sqrt{N_{Rx}N_{Tx}}\\&\text{diag}\bigg\{\rho_0e^{-j2\pi n\tau_0/(NT_s)},\dots,\rho_{K-1}e^{-j2\pi n\tau_{K-1}/(NT_s)}\bigg\},
\end{align}
where $h_k$ and $\tau_k$ denotes the complex channel gain and delays of the $k$th path while $T_s$ is the sampling interval. Then, $\bar{\mathbf{h}}$ is defined as concatenation of the vectorized version of $\mathbf{H}[n]$ for all the subcarriers $n=1,\dots,N$, i.e.,
\begin{align}
\label{eq:fixedpart-channel}
    \Bar{\mathbf{h}}=\big[\text{vec}\{\mathbf{H}[1]\}^T,\dots,\text{vec}\{\mathbf{H}[N]\}^T\big]^T,
\end{align}
where $\text{vec}\{\cdot\}$ denotes the column vector operator. We denote the associated parameters with the specular paths contribution, $\overline{\mathbf{h}},$ which remains constant during a coherence time $T_c$, via a $4K\times1$ vector $\bm{\theta}_{sp}$ defined as
\begin{align}
    \bm{\theta}_{sp}=[\bm{\psi}_{T},\bm{\psi}_{R},\bm{\tau},\bm{\rho}]^T,
\end{align}
where $\bm{\psi}_{T}=[ \psi_{T,0},\dots,\psi_{T,K-1}]$, $\bm{\psi}_{R}=[ \psi_{R,0},\dots,\psi_{R,K-1} ]$, $\bm{\tau}=[ \tau_0,\dots,\tau_{K-1}]$ and $\bm{\rho}=[ \rho_0,\dots,\rho_{K-1}]$.

For modeling the  variable part of the channel we first assume that the wide-sense stationary uncorrelated scattering (WSSUS) assumption holds, and use a multipath tapped delay line, $h(t,\tau)=\sum_{l=0}^{L-1} A_l(t)\delta(\tau-l\Delta\tau),$ to model the impulse response at time $t$ between two antennas. Here, $A_l(t)$ and $\Delta\tau=1/W$ denote the (complex) amplitude of the $l$th path and the delay between two consecutive paths, respectively. Sampling the impulse response at time $t=uT$, followed by taking the Fourier transform w.r.t. $\tau$ would result in a vector $\mathbf{q}_{{u}}$ whose $n$th element is denoted by
\begin{align}
\label{eq:variablepart-channel}
&\mathbf{q}_{u}[n]=\mathcal{F}\{h(kT,\tau)\}|_{f=f_0-W/2+n\Delta f}=\nonumber\\&\sum_{l=0}^{L-1}A_{u,l}e^{-j2\pi(f_0-W/2+n\Delta f)l/W}, n=1,\dots,N_f,
\end{align}
where $\Delta f$ and $A_{u,l}$ denotes the subcarrier width and the $l$th channel gain at time $u$, respectively. Following the exponential decay model which holds for the power delay profile of $\mathbf{q}_{u}$ based on various experimental observations \cite{CFR1}, $A_{u,l}$ is modeled with zero-mean Gaussian distribution with variance $\text{Var}(A_{u,l})=\alpha^2(1-e^{-2\pi\beta})e^{-2\pi\beta}$. Here, $\alpha^2$ and $\beta$ denotes the average power and the normalized coherence bandwidth, respectively. The distribution of $\mathbf{q}_u$ is given in the following lemma.
\begin{lemma}
\label{lemma:Dist-of-CFR}
The vector $\mathbf{q}_u$ has a multivariate Gaussian distribution $\mathcal{CN}(\mathbf{0},\mathbf{R}_{\mathbf{q}})$ with a Toeplitz covarinace matrix $\mathbf{R}_{\mathbf{q}}=\text{toep}(\bm{\nu}_{\mathbf{q}},\bm{\nu}_{\mathbf{q}}^H)$ assuming
\begin{align}
    \bm{\nu}_{\mathbf{q}}\overset{\bigtriangleup}{=}\big[\kappa(\bm{\theta}_{\mathbf{q}},0), \kappa(\bm{\theta}_{\mathbf{q}},\frac{1}{M}),\dots,  \kappa(\bm{\theta}_{\mathbf{q}},1-\frac{1}{M})\big],
\end{align}
where $\kappa(\bm{\theta}_{\mathbf{q}},m)\overset{\bigtriangleup}{=}\frac{\alpha^2(1-e^{-2\pi\beta})(1-e^{-2\pi L(\beta-mj)})}{(1-e^{-2\pi (\beta-mj)})}$ and $\bm{\theta}_{\mathbf{q}}=[\alpha^2,\beta,L]$.
\end{lemma}
\begin{proof}
See Appendix \ref{Appen:lemma:Dist-of-CFR}.
\end{proof}
Next, the contribution of measurement noise $\mathbf{n}$ is modelled with a zero-mean complex multivariate Gaussian random variable as $\mathbf{n}\sim\mathcal{CN}(\mathbf{0},\sigma^2\mathbf{I})$ where $\sigma^2$ denotes the variance of the noise. This can be incorporated in the above lemma by defining $\bm{\nu}_{\mathbf{q},\mathbf{n}}\overset{\bigtriangleup}{=}\bm{\nu}_{\mathbf{q}}+[\sigma^2,0,\dots,0]$ and $\mathbf{R}_{\mathbf{q},\mathbf{n}}=toep(\bm{\nu}_{\mathbf{q},\mathbf{n}},\bm{\nu}_{\mathbf{q},\mathbf{n}}^H)$. We then follow the Kronecker model to obtain the covariance matrix of the CFR, which holds when the diffuse spectrum contribution in the angular domains is independent from that in the frequency domain \cite{ImpactofIncomplete,Rimax}. Under the Kronecker model, the covariance matrix of the CFR can be decomposed as $\mathbf{R}=\mathbf{I}_{N_{Rx}}\otimes\mathbf{I}_{N_{Tx}}\otimes\mathbf{R}_{\mathbf{q},\mathbf{n}}$. Therefore, the distribution of the CFR in (\ref{eq:full-channel}) within the above model can be given as $\mathbf{h}_u\sim \mathcal{CN}(\overline{\mathbf{h}},\mathbf{R})$. We denote the parameters associated with the covariance matrix by $\bm{\theta}_{vn}=[\alpha,\beta,L,\sigma]$, which relates to the variable part of the CFR and noise. As mentioned earlier, the mean $\overline{\mathbf{h}}$ solely depends on the specular paths parameters $\bm{\theta}_{sp}$.
\subsection{Channel spoofing detection problem}
\label{section:LRT-Spoofing}
Channel-based spoofing detection is generally studied \cite{CFR1,CFR2} in the ``snapshot” scenario where Bob receives a new message claiming to be sent
by Alice, and he has to check whether the claim is true. To this end, we assume that Bob is able to measure and store a noisy version of the CFR corresponding to a transmitting terminal.
Based on the CFRs associated with the incoming messages, the goal in this scenario is to determine whether a message at time $t=(u+1)T$ belongs to Alice or Eve given a reference message from Alice\footnote{In the remaining of this section, we use $A$ or $E$ in the superscript of vector or scalar to indicate that it belongs to Alice or Eve, respectively.}, $\mathbf{h}^A_k$, at time $t=uT$. In this setup, we use the terms message and CFR interchangeably. One can pose the spoofing detection as a binary classification problem for which two hypotheses can be made:
\begin{align}
\label{eq:Spoofin-hypothesisTesting}
    &\mathcal{H}_0: \mathbf{h}_{u+1}=\mathbf{h}^A_{u+1},\\
    &\mathcal{H}_1: \mathbf{h}_{u+1}=\mathbf{h}^E_{u+1}.
\end{align}
 Under the null hypothesis, $\mathcal{H}_0$, the message at time $t=(u+1)T$ belongs to Alice, while under the
alternative hypothesis $\mathcal{H}_1$ a spoofing attack has occurred, i.e., the
message belongs to Eve. 

For the data acquisition phase, we consider a setting where Bob spots a finite number of snapshots from a coherence time and stores the observed CFRs. Furthermore, in order to label the incoming CFRs, we use a heuristic method given by
\begin{align}
\label{eq:imperfectlabeling-spoofing}
    \lVert \mathbf{h}_{u+1}-\mathbf{h}^A_{u}\lVert^2\underset{\mathcal{H}_0}{\overset{\mathcal{H}_1}{\gtreqless}} \eta,
\end{align}
in the lieu of the channel parameters $\bm{\theta}_{vn}$ and $\bm{\theta}_{sp}$. This method can be viewed as an imperfect labeling mechanism which decides in favor of $\mathcal{H}_0$ if the Euclidean distance between an incoming CFR and the reference CFR is smaller than a predefined threshold $\eta$.

From a statistical perspective, likelihood ratio test is the main approach for deciding between the two hypothesis, which relies on the knowledge of unknown channel parameters as obtained in the following. The likelihood ratio test for the snapshot scenario at time $t=(u+1)T$ is defined as
\begin{align}
\label{eq:LRT_Diff}
    \mathds{L}\big(\mathbf{h}_{u+1}|\mathbf{h}^A_u\big)\overset{\bigtriangleup}{=}\frac{p(\mathbf{h}_{u+1}-\mathbf{h}^A_u|\mathcal{H}_0)}{p(\mathbf{h}_{u+1}-\mathbf{h}^A_u|\mathcal{H}_1)}\underset{\mathcal{H}_0}{\overset{\mathcal{H}_1}{\gtreqless}}\zeta,
\end{align}
for a predefined threshold $\zeta$, where the conditional probability distribution of $\mathbf{h}_{k+1}-\mathbf{h}^A_k$ serves as the likelihood function under each behavior. In the following, we obtain closed-form expressions for these likelihood functions assuming the statistical dependence on the reference CFR for each hypothesis is specified via the conditional distributions $p(\mathbf{q}_{u+1}^A|\mathbf{q}_{u}^A)$ and $p(\mathbf{q}_{u+1}^E|\mathbf{q}_{u}^A)$ under $\mathcal{H}_0$ and $\mathcal{H}_1$, respectively. Specifically, the dependence of $\mathbf{q}_{u+1}^A$ on $\mathbf{q}_{u}^A$ is characterized through the corresponding channel gains in terms of an order-$1$ auto-regressive model (AR-$1$) \cite{CFR1}, i.e.,
\begin{align}
\label{eq:AR1}
    A^A_{u+1,l}=a^AA^A_{u,l}+\sqrt{(1-({a^A})^2)\text{Var}(A^A_{u+1,l})}w_{u+1,l}
\end{align}
where $a^A$ denotes the \textit{similarity parameter}, and $w_{u+1,l}\sim \mathcal{CN}(0,1)$ is independent of $A_{u,l}$. Similarly, the $l$th path gain corresponding to $\mathbf{q}_{u+1}^E$ and $\mathbf{q}_{u}^A$ are related with and AR-$1$ model with similarity parameter $a^E$.

\begin{lemma}
\label{lemma:Dist_Null_Hypothesis}
Under the null hypothesis, $p(\mathbf{q}_{u+1}-\mathbf{q}^A_u|\mathcal{H}_0)=\mathcal{CN}(\mathbf{0},\mathbf{R}_{q,\mathcal{H}_0})$ for
\begin{flalign}
    &\mathbf{R}_{q,\mathcal{H}_0}=\text{toep}(\bm{\nu}_{\mathcal{H}_0},\bm{\nu}_{\mathcal{H}_0}^H),\\  &\bm{\nu}_{\mathcal{H}_0}\overset{\bigtriangleup}{=}\Big[2(1-a^A)\kappa(\bm{\theta}_{\mathbf{q}}^A,0),2(1-a^A)\kappa\big(\bm{\theta}_{\mathbf{q}}^A,\frac{1}{M}\big),\dots\nonumber\\&\ \ \ \ \ \ \ \  \ \ \ \ \ \ \  \ \ \ \ \  \ \ \ \ \ \ \ \ \ \  ,2(1-a^a)\kappa\big(\bm{\theta}_{\mathbf{q}}^A,\frac{M-1}{M}\big)\Big],
\end{flalign}
where $\bm{\theta}_{\mathbf{q}}^A\overset{\bigtriangleup}{=}[\alpha^A,\beta^A,L^A]$ and the $\kappa$ function defined in Lemma \ref{lemma:Dist-of-CFR}.
\end{lemma}
\begin{proof}
See Appendix \ref{Appen:lemma:Dist_Null_Hypothesis}.
\end{proof}
\begin{lemma}
\label{lemma:Dist-Alternate-Hypothesis}
Under the alternative hypothesis, $p(\mathbf{q}_{u+1}-\mathbf{q}^A_u|\mathcal{H}_1)=\mathcal{CN}(\mathbf{0},\mathbf{R}_{\mathbf{q},\mathcal{H}_1})$ for
\begin{align}
    &\mathbf{R}_{\mathbf{q},\mathcal{H}_1}=\text{toep}(\bm{\nu}_{\mathcal{H}_1},\bm{\nu}_{\mathcal{H}_1}^H),\\  &\bm{\nu}_{\mathcal{H}_1}\overset{\bigtriangleup}{=}[\kappa^{'}(a^E,\theta_{\mathbf{q}}^A,\theta_{\mathbf{q}}^E,0),\kappa^{'}(a^E,\theta_{\mathbf{q}}^A,\theta_{\mathbf{q}}^E,\frac{1}{M}),\dots,\nonumber\\&\ \ \ \ \ \ \ \ \ \ \ \ \ \ \ \ \ \ \ \ \ \ \ \  \ \ \ \ \ \ \ \ \ \ \ \kappa^{'}(a^E,\theta_{\mathbf{q}}^A,\theta_{\mathbf{q}}^E,\frac{M-1}{M})],\\
    &\kappa^{'}(a^E,\theta_{\mathbf{q}}^A,\theta_{\mathbf{q}}^E,m)\overset{\bigtriangleup}{=}\kappa(\theta_{\mathbf{q}}^E,m)-2a^E\kappa(\theta_{\mathbf{q}}^A,m)+\kappa(\theta_{\mathbf{q}}^A,m),
\end{align}
 where $\theta_{\mathbf{q}}^A\overset{\bigtriangleup}{=}[\alpha^A,\beta^A,L^A]$, $\theta_{\mathbf{q}}^E\overset{\bigtriangleup}{=}[\alpha^E,\beta^E,L^E]$, and the $\kappa$ function is defined in Lemma \ref{lemma:Dist-of-CFR}.
\end{lemma}
\begin{proof}
See Appendix \ref{Appen:lemma:Dist-Alternate-Hypothesis}.
\end{proof}
The above two lemmas enables us to obtain the likelihood functions for (\ref{eq:LRT_Diff}). Regarding the null hypothesis $\mathcal{H}_0$, using the Kronecker model for the covariance matrix \cite{ImpactofIncomplete} and the fact that the measurement noise is independent from the other CFR's components, we obtain the covariance matrix of $\mathbf{h}_{u+1}-\mathbf{h}^A_{u}$ as $\mathbf{R}_{\mathcal{H}_0}=\mathbf{I}_{N_{Rx}}\otimes\mathbf{I}_{N_{Tx}}\otimes\mathbf{R}_{\mathbf{q},\mathcal{H}_0}+2(\sigma^A)^2\mathbf{I}_M$ where $\mathbf{R}_{\mathbf{q},\mathcal{H}_0}$ is given in Lemma \ref{lemma:Dist_Null_Hypothesis}. Furthermore, as a snapshot is captured in one coherence time, the specular paths contribution of the CFRs remains the same ($\overline{\mathbf{h}}^A$) between two consecutive times in this case. Therefore, under $\mathcal{H}_0$ the likelihood function is $\mathcal{CN}(\mathbf{0},\mathbf{R}_{\mathcal{H}_0})$. Similarly, for the alternate hypothesis $\mathcal{H}_1$, the likelihood function can be obtained as $\mathcal{CN}(\overline{\mathbf{h}}^E-\overline{\mathbf{h}}^A,\mathbf{R}_{\mathcal{H}_1})$ where $\mathbf{R}_{\mathcal{H}_1}=\mathbf{I}_{N_{Rx}}\otimes\mathbf{I}_{N_{Tx}}\otimes\mathbf{R}_{\mathbf{q},\mathcal{H}_1}+(\sigma^A)^2\mathbf{I}_M+(\sigma^E)^2\mathbf{I}_M$ and $\mathbf{R}_{\mathbf{q},\mathcal{H}_1}$ is given in Lemma \ref{lemma:Dist-Alternate-Hypothesis}. 

\subsection{Parameter estimation}
\label{Section:Spoofing-paramEst}
In order to utilize the likelihood ratio test in (\ref{eq:LRT_Diff}), Bob requires the knowledge of the channel parameters $\bm{\theta}_{sp}$, $\bm{\theta}_{vn}$ corresponding to Alice-Bob and Eve-Bob channels along with the similarity parameters. In practice, these parameters should be estimated based on the data collected from the observed snapshots. Recall from Lemma \ref{lemma:Dist-of-CFR} that a CFR associated with a terminal has a Gaussian distribution of the from $\mathcal{CN}\big(\overline{\mathbf{h}}_{\bm{\theta}_{sp}},\mathbf{R}_{\bm{\theta}_{vn}}\big)$. The MLE estimates of the parameters are obtained via
\small
\begin{subequations}
\begin{align}
    &\hat{\bm{\theta}}_{sp}, \hat{\bm{\theta}}_{vn}=\argmax_{\bm{\theta}_{sp},\bm{\theta}_{vn}} \mathcal{L}\big(\mathbf{h}|\bm{\theta}_{sp},\mathbf{R}_{\bm{\theta}_{vn}}\big),\label{eq:Loglikelihood-CFR0}\\
    &\mathcal{L}(\mathbf{h}|\bm{\theta}_{sp},\mathbf{R}_{\bm{\theta}_{vn}}\big)=\\&-M\ln\pi-\ln{\det{\mathbf{R}_{\bm{\theta}_{vn}}}}-\big(\mathbf{h}-\overline{\mathbf{h}}_{\bm{\theta}_{sp}}\big)^H\mathbf{R}^{-1}_{\bm{\theta}_{vn}}\big(\mathbf{h}-\overline{\mathbf{h}}_{\bm{\theta}_{sp}}\big),\label{eq:Loglikelihood-CFR}
\end{align}
\end{subequations}
\normalsize
which amounts to jointly maximizing the arguments of some nonlinear objective function. Besides, it can be proved that (\ref{eq:Loglikelihood-CFR}) is not a convex function of $\bm{\theta}_{sp}$, and as a result there is no unique solution set for the optimization problem in (\ref{eq:Loglikelihood-CFR0}). In practice, solving such problem is far from trivial, especially since the number of nonlinear parameters ($\bm{\theta}_{sp}$) is large and a multidimensional exhaustive search is not feasible. As a workaround, the authors in \cite{ImpactofIncomplete,Rimax} propose a suboptimal procedure to break the problem into two sub-problems and estimate $\bm{\theta}_{sp}$ and $\bm{\theta}_{vn}$ in a separate manner. In our problem, there is also similarity parameters $a^A$ and $a^E$ which similar to $\bm{\theta}_{vn}$ appear in the covariance matrix of a Gaussian distribution, as obtained in Lemmas \ref{lemma:Dist_Null_Hypothesis} and \ref{lemma:Dist-Alternate-Hypothesis}. Therefore, based on the approach taken in \cite{ImpactofIncomplete,Rimax}, we break the parameter estimation problem into three sub-problems. Each sub-problem involves numerically maximizing the objective function of the form (\ref{eq:Loglikelihood-CFR}) w.r.t. $\bm{\theta}_{sp}$ or $\bm{\theta}_{vn}$ or similarity parameters via an iterative local optimization technique, such as Gauss-Newton algorithm. In particular, the maximization processes are done sequentially and in an alternating manner between the three sets of parameters towards convergence. In the following, we elaborate on each sub-problem for the specific channel model we described earlier.
\subsubsection{Estimating the specular path parameters \texorpdfstring{$\bm{\theta}_{sp}$}{}}
\label{Section:Estimating the specular path parameters}
The main goal here is to obtain an estimate of $\bm{\theta}_{sp}$ which maximizes (\ref{eq:Loglikelihood-CFR}) for a given estimate of $\bm{\theta}_{vn}$. In the following, we use the $N$-exponential basis function defined as
\begin{align}
    \mathbf{U}_N^\mathbf{v}=\begin{bmatrix}e^{-j\big(-\frac{N-1}{2}\big)\mathbf{v}[1]}&\dots&e^{-j\big(-\frac{N-1}{2}\big)\mathbf{v}[n]}\\
    \vdots&\ddots&\vdots\\
    e^{-j\big(\frac{N-1}{2}\big)\mathbf{v}[1]}&\dots&e^{-j\big(\frac{N-1}{2}\big)\mathbf{v}[n]}
    \end{bmatrix},
\end{align}
for a vector $\mathbf{v}$ of length $N$. Partial derivative of $ \mathbf{U}_N^\mathbf{v}$ w.r.t $\mathbf{v}$ is readily computed as
$
     \mathbf{D}_N^\mathbf{v}=\frac{\partial  \mathbf{U}_N^\mathbf{v}}{\partial \mathbf{v}}=-j\Xi_N  \mathbf{U}_N^\mathbf{v}
$
where $\Xi_N=\text{diag}([-(N-1)/2,\dots,(N-1)/2])$. Furthermore, we recall that for arbitrary matrices $
\mathbf{A}\in\mathbb{C}^{N\times P},\ \mathbf{B}\in\mathbb{C}^{M\times P},\  \mathbf{Q}^{P\times P}=\text{diag}(\mathbf{q})$ and a vector $\mathbf{q}\in \mathbb{C}^{P\times 1}$, one can write
$
    \text{vec}\{\mathbf{B}\mathbf{Q}\mathbf{A}^T\}=(\mathbf{A}\odot\mathbf{B})\mathbf{q}
$. Utilizing this result along with the exponential basis function we can rewrite the specular path contribution introduced in (\ref{eq:fixedpart-channel}) for the CFR model as
\begin{align}
\overline{\mathbf{h}}=\big(\mathbf{U}_{N_{Rx}}^{\bm{\psi}_{T}}\odot\mathbf{U}_{N_{Tx}}^{\bm{\psi}_{R}}\odot\mathbf{U}_{N_f}^{\bm{\tau}}\big)\bm{\rho},
\end{align}
which greatly simplifies the calculation of the first and second derivatives of $\overline{\mathbf{h}}$ w.r.t $\bm{\theta}_{sp}$. Specifically, the Jacobian matrix for the above model is obtained via
   $ \mathbf{J}(\bm{\theta}_{sp})=\mathbf{J}_{{\bm{\psi}_T}}\odot\mathbf{J}_{\bm{\psi}_R}\odot\mathbf{J}_\mathbf{\rho}\odot\mathbf{J}_{\bm{\tau}}$
where the Jacobian matrix's components are given by
\begin{subequations}
\begin{align}
    \mathbf{J}_{{\bm{\psi}_T}}&=\begin{bmatrix}\mathbf{D}_{N_{Tx}}^{\bm{\psi}_T}&\mathbf{U}_{N_{Tx}}^{\bm{\psi}_T}&\mathbf{U}_{N_{Tx}}^{\bm{\psi}_T}&\mathbf{U}_{N_{Tx}}^{\bm{\psi}_T}&\mathbf{U}_{N_{Tx}}^{\bm{\psi}_T}\end{bmatrix},\\
    \mathbf{J}_{\bm{\psi}_R}&=\begin{bmatrix}\mathbf{U}_{N_{Rx}}^{\bm{\psi}_R}&\mathbf{D}_{N_{Rx}}^{\bm{\psi}_R}&\mathbf{U}_{N_{Rx}}^{\bm{\psi}_R}&\mathbf{U}_{N_{Rx}}^{\bm{\psi}_R}&\mathbf{U}_{N_{Rx}}^{\bm{\psi}_R}\end{bmatrix},\\
    \mathbf{J}_{\bm{\tau}}&=\begin{bmatrix}\mathbf{U}_{N_f}^{\bm{\tau}}&\mathbf{U}_{N_f}^{\bm{\tau}}&\mathbf{D}_{N_f}^{\bm{\tau}}&\mathbf{U}_{N_f}^{\bm{\tau}}&\mathbf{U}_{N_f}^{\bm{\tau}}\end{bmatrix},\\
    \mathbf{J}_{\bm{\rho}}&=\begin{bmatrix}{\bm{\rho}}^T&{\bm{\rho}}^T&{\bm{\rho}}^T&\mathbf{1}^T&\mathbf{1}^Tj\end{bmatrix}.
\end{align}
\end{subequations}
\normalsize
The first order partial derivative of the log likelihood function (\ref{eq:Loglikelihood-CFR}) given an observation $\mathbf{h}$ with respect to the parameters $\bm{\theta}_{sp}$ is denoted by $\mathbf{q}_{\bm{\theta}_{sp}}\big(\mathbf{h}|\mathbf{R}_{\bm{\theta}_{vn}}\big)$. For a given Jacobian matrix $\mathbf{J}(\bm{\theta}_sp)$, one can compute \cite{Rimax}

\begin{align}
\label{eq:first-order-specular-Jacobian}
    \mathbf{q}_{\bm{\theta}_{sp}}\big(\mathbf{h}|\mathbf{R}_{\bm{\theta}_{vn}}\big)=2\Re{\big\{\mathbf{J}^H(\bm{\theta}_{sp})\mathbf{R}^{-1}_{\bm{\theta}_{vn}}(\mathbf{h}-\overline{\mathbf{h}}_{\bm{\theta}_{sp}}\big\}}.
\end{align}
Furthermore, the negative covariance matrix of the above first order derivative, i.e., 
\begin{align}
   -\mathbb{E}\bigg[\frac{\partial \mathcal{L}(\mathbf{h}|\bm{\theta}_{sp},\mathbf{R}_{\bm{\theta}_{vn}}\big)}{\partial \bm{\theta}_{sp}}\Big(\frac{\partial \mathcal{L}(\mathbf{h}|\bm{\theta}_{sp},\mathbf{R}_{\bm{\theta}_{vn}}\big)}{\partial \bm{\theta}_{sp}}\Big)^T\bigg],
\end{align}
is called the Fisher information matrix (FIM), $\mathbf{F}\big({\bm{\theta}_{sp}}|\mathbf{R}_{\bm{\theta}_{vn}}\big)$, which can be expressed in terms of the Jacobian matrix $\mathbf{J}(\bm{\theta}_{sp})$ as \cite{Rimax}
\begin{align}
\label{eq:FIM-Jacobian}
     \mathbf{F}\big({\bm{\theta}_{sp}}|\mathbf{R}_{\bm{\theta}_{vn}}\big)=2\Re\{{\mathbf{J}^H(\bm{\theta}_{sp})\mathbf{R}^{-1}_{\bm{\theta}_{vn}}\mathbf{J}(\bm{\theta}_{sp})}\}.
\end{align}

Having obtained (\ref{eq:first-order-specular-Jacobian}) and (\ref{eq:FIM-Jacobian}), a local optimization technique can be utilized to obtain an iterative rule for estimation of $\bm{\theta}_{sp}$. To this end, we employ the Gauss-Newton algorithm as \begin{align}
\label{eq:GaussNewton-fixedparams}
    \hat{\bm{\theta}}^{i+1}_{sp}=\hat{\bm{\theta}}^{i}_{sp}+\zeta\ \mathbf{F}^{-1}\big(\hat{\bm{\theta}}^{i}_{sp}|\mathbf{R}_{\bm{\theta}_{vn}}\big)\mathbf{q}_{\hat{\bm{\theta}}^{i}_{sp}}\big(\mathbf{h}|\mathbf{R}_{\bm{\theta}_{vn}}\big)
\end{align} 
for a step length $\zeta$ which should be chosen such that $\mathcal{L}(\mathbf{h}|\bm{\theta}^{i+1}_{sp},\mathbf{R}_{\bm{\theta}_{vn}}\big)>\mathcal{L}(\mathbf{h}|\bm{\theta}^{i}_{sp},\mathbf{R}_{\bm{\theta}_{vn}}\big)$.

\subsubsection{Estimating the variable part and noise parameters \texorpdfstring{$\bm{\theta}_{vn}$}{}}
\label{section:Estimating the variable part and noise parameters}For a given estimate of $\bm{\theta}_{sp}$, the goal here is to estimate ${\bm{\theta}_{vn}}$. To this end, we assume $N$ number of CFRs denoted with $\{\mathbf{h}_i\}_{i=1}^N$ are available. For each CFR, the approach presented in \ref{Section:Estimating the specular path parameters} is utilized to estimate the corresponding specular paths parameters as $\{\hat{\bm{\theta}}_{sp,i}\}_{i=1}^N$. In order to remove the contribution of the specular paths from the CFRs, we from an $M\times N$ matrix $\mathbf{H}$ whose $i$th column amounts to $\mathbf{h}_i-\overline{\mathbf{h}}(\overline{\bm{\theta}}_{sp})$ where $\overline{\bm{\theta}}_{sp}$ denotes the average value of $\hat{\bm{\theta}}_{sp,i}$'s. In the following, we propose an estimation procedure for $\bm{\theta}_{vn}$ using $\mathbf{H}$. We first note that all the parameters in $\bm{\theta}_{vn}$ are continuous except for the number of diffuse spectrum paths $L$ which is an integer value. As a result, the objective function is not continuous in $L$ and the gradient of (\ref{eq:Loglikelihood-CFR}) does not exist w.r.t. $L$. In this vein, we take a sub-optimal approach towards estimating $L$ by separating it from the rest of the parameters in $\bm{\theta}_{vn}$. The authors in \cite{CFR-separation} have proposed an eigenvalue ratio method for estimating the number of harmonics present in the signals from $N$ available observations. For our specific case, we denote the MLE of the covariance of $\mathbf{H}$ by $\mathbf{C}_{\mathbf{H}}$, and the corresponding eigenvalues by $\mathbf{e}_i,\ i=1,\dots,M$. The propose heuristic approach is to choose $\hat{L}$ such that
$
    \frac{\sum_{i=1}^{\hat{L}}\mathbf{e}_i}{\sum_{i=1}^{M}\mathbf{e}_i}\geq \eta,
$
for a predefined value of$\eta$. 
   \begin{figure*}[b]
   \small
    \begin{subequations}
   \begin{align}
   & \frac{\partial \bm{\nu}_{\mathbf{q},\mathbf{n}}}{\partial\sigma^2}=\begin{bmatrix}1,0,\dots,0\end{bmatrix},\label{eq:longeq1}\\
   & \frac{\partial \bm{\nu}_{\mathbf{q},\mathbf{n}}}{\partial\alpha^2}=\begin{bmatrix}1-e^{-2\pi L\beta},\frac{(1-e^{-2\pi\beta})\big(1-f^L(\frac{1}{M})\big)}{1-f(\frac{1}{M})},\dots,\frac{(1-e^{-2\pi\beta})\big(1-f^L(1-\frac{1}{M})\big)}{1-f(1-\frac{1}{M})}\end{bmatrix},\\
   &\frac{\partial \bm{\nu}_{\mathbf{q},\mathbf{n}}}{\partial\beta}=\Big[\begin{matrix}2\pi\alpha^2 Le^{-2\pi\beta L},\frac{2\pi e^{-2\pi\beta}\big(f^L(\frac{1}{M})-1\big)}{f(\frac{1}{M})-1}+\frac{2L\pi f^L(\frac{1}{M})(e^{-2\pi\beta}-1)}{f(\frac{1}{M})-1}-\frac{2\pi f(\frac{1}{M})\big(f^L(\frac{1}{M})-1\big)(e^{-2\pi \beta}-1)}{\big(f(\frac{1}{M})-1\big)^2},\end{matrix}\nonumber\\
   &\ \ \ \ \ \ \ \ \ \ \ \ \ \ \ \begin{matrix}\dots,\frac{2\pi e^{-2\pi\beta}\big(f^L(1-\frac{1}{M})-1\big)}{f(1-\frac{1}{M})-1}+\frac{2L\pi f^L(1-\frac{1}{M})(e^{-2\pi\beta}-1)}{f(1-\frac{1}{M})-1}-\frac{2\pi f(1-\frac{1}{M})\big(f^L(1-\frac{1}{M})-1\big)(e^{-2\pi \beta}-1)}{\big(f(1-\frac{1}{M})-1\big)^2}\end{matrix}\Big].\label{eq:longeq3}
\end{align}
 \end{subequations}
\end{figure*}

We plug the estimated value of $L$ in the parameter vector to obtain $\bm{\theta}_{vn}=[\sigma,\alpha,\beta,\hat{L}]$. Then, the log-likelihood function for the zero-mean CFRs can be written as
\small
\begin{align}
    \mathcal{L}(\mathbf{H}|{\bm{\theta}_{vn}})=-MN\ln\pi-N\ln{\det{\mathbf{R}_{\bm{\theta}_{vn}}}}-\Tr\big(\mathbf{H}^H\mathbf{R}^{-1}_{\bm{\theta}_{vn}}\mathbf{H}\big).\label{eq:Loglikelihood-CFR2}
\end{align}
\normalsize
 The first-order gradient of $ \mathcal{L}(\mathbf{H}|{\bm{\theta}_{vn}})$ w.r.t. to each parameter can be computed as \cite{Rimax}
\small
\begin{align}
\label{eq:First-Order-Gradiernt-variablepart}
     \frac{\partial  \mathcal{L}(\mathbf{H}|{\bm{\theta}_{vn}})}{\partial{\bm{\theta}_{vn}}[i]}=N\Tr\Big(\mathbf{R}^{-1}_{\bm{\theta}_{vn}}\frac{\partial \mathbf{R}_{\bm{\theta}_{vn}}}{\partial{\bm{\theta}_{vn}}[i]}\mathbf{R}^{-1}_{\bm{\theta}_{vn}}(\widehat{\mathbf{R}}-\mathbf{R}_{\bm{\theta}_{vn}})\Big).
\end{align}
\normalsize
for $i=1,2,3$. Subsequently, the $(i,j)$th element of the FIM corresponding to $\mathcal{L}(\mathbf{H}|{\bm{\theta}_{vn}})$ equals \cite{Rimax}
\begin{align}
\label{eq:FIM-variablepart}
     -\mathbb{E}\Big[\frac{\partial^2  \mathcal{L}(\mathbf{H}|{\bm{\theta}_{vn}})}{\partial{\bm{\theta}_{vn}}[i]\partial{\bm{\theta}_{vn}}(j)}\Big]=N\Tr\Big(\mathbf{R}^{-1}_{\bm{\theta}_{vn}}\frac{\partial \mathbf{R}_{\bm{\theta}_{vn}}}{\partial{\bm{\theta}_{vn}}[i]}\mathbf{R}^{-1}_{\bm{\theta}_{vn}}\frac{\partial \mathbf{R}_{\bm{\theta}_{vn}}}{\partial{\bm{\theta}_{vn}}(j)}\Big).
\end{align}
To obtain explicit expressions for (\ref{eq:First-Order-Gradiernt-variablepart}) and (\ref{eq:FIM-variablepart}), one needs to compute partial derivatives terms, i.e., $\frac{\partial \mathbf{R}_{\bm{\theta}_{vn}}}{\partial{\bm{\theta}_{vn}}[i]}$. Considering the Toeplitz structure of the covariance model described in Lemma \ref{lemma:Dist-of-CFR}, we can write
\begin{align}
   & \frac{\partial \mathbf{R}_{\mathbf{q},\mathbf{n}}(\bm{\theta}_{vn})}{\partial{\bm{\theta}_{vn}}[i]}=toep\Big(\frac{\partial \bm{\nu}_{\mathbf{q},\mathbf{n}}}{\partial{\bm{\theta}_{vn}}[i]},\frac{\partial \bm{\nu}^H_{\mathbf{q},\mathbf{n}}}{\partial{\bm{\theta}_{vn}}[i]}\Big),\\
   &\frac{\partial \mathbf{R}_{\bm{\theta}_{vn}}}{\partial{\bm{\theta}_{vn}}[i]}=\mathbf{I}_{N_{Rx}}\otimes\mathbf{I}_{N_{Tx}}\otimes\frac{\partial \mathbf{R}_{\mathbf{q},\mathbf{n}}(\bm{\theta}_{vn})}{\partial{\bm{\theta}_{vn}}[i]},
   \end{align}
where the partial derivative for each parameter is obtained in (\ref{eq:longeq1})-(\ref{eq:longeq3}) for $f(m)=e^{-2\pi(\beta-jm)}$. Plugging this in (\ref{eq:First-Order-Gradiernt-variablepart}) and (\ref{eq:FIM-variablepart}) leads to computation of first-order gradient and the FIM of the likelihood function. Then, an iterative approach like the Gauss-Newton algorithm  can be employed for estimating ${\bm{\theta}_{vn}}$ in a similar fashion to the case of $\bm{\theta}_{sp}$ in (\ref{eq:GaussNewton-fixedparams}).

\subsubsection{Estimating the similarity parameters $a^A$ and $a^E$} In order to estimate the parameters $a^A$ and $a^E$, one can use the likelihood functions obtained in Lemmas \ref{lemma:Dist_Null_Hypothesis} and \ref{lemma:Dist-Alternate-Hypothesis}, respectively, while using the estimates obtained in Sections \ref{Section:Estimating the specular path parameters} and \ref{section:Estimating the variable part and noise parameters} for the values of all the other parameters. Specifically, as these parameters appear in the covariance matrix of a Gaussian distribution, a similar estimation procedure to that of $\bm{\theta}_{sp}$ can be employed here as well. In fact, the expressions for the first-order gradient and the FIM of the likelihood function in this case is the same as those in (\ref{eq:First-Order-Gradiernt-variablepart}) and (\ref{eq:FIM-variablepart}), respectively, except for the fact that there is only one parameter to estimate in this case. For example, for estimation of $a^A$, by considering the Toeplitz structure of the covariance model described in Lemma \ref{lemma:Dist_Null_Hypothesis}, we can write
\begin{align}
   & \frac{\partial \mathbf{R}_{\mathbf{q},\mathcal{H}_0}(a^A)}{\partial a^A}=toep\Big(\frac{\partial \bm{\nu}_{\mathcal{H}_0}}{\partial a^A},\frac{\partial \bm{\nu}^H_{\mathcal{H}_0}}{\partial a^A}\Big),
   \\&\frac{\partial \mathbf{R}_{\mathcal{H}_0}(a^A)}{\partial a^A}=\mathbf{I}_{N_{Rx}}\otimes\mathbf{I}_{N_{Tx}}\otimes\frac{\partial \mathbf{R}_{\mathbf{q},\mathcal{H}_0}(a^A)}{\partial a^A},
   \end{align}
   where the partial derivative w.r.t. $a^A$ can be obtained as
   \begin{align}
          \frac{\partial \bm{\nu}_{\mathcal{H}_0}}{\partial a^A}=-2\Big[& \frac{(\alpha^A)^2(1-e^{-2\pi\beta^A})(1-f^{L^A}(0))}{1-f(0)},\\\nonumber& \frac{(\alpha^A)^2(1-e^{-2\pi\beta^A})(1-f^{L^A}(\frac{1}{M}))}{1-f(\frac{1}{M})},\dots,\\\nonumber& \frac{(\alpha^A)^2(1-e^{-2\pi\beta^A})(1-f^{L^A}(1-\frac{1}{M}))}{1-f(1-\frac{1}{M})}\Big],
   \end{align}
   \normalsize
for $f(m)=e^{-2\pi(\beta-jm)}$. Subsequently, using the first-order gradient and the FIM of the likelihood function, we utilize the Gauss-Newton algorithm to estimate $a^A$. Similar approach can be taken for estimating $a^E$ using the covariance model described in Lemma \ref{lemma:Dist-Alternate-Hypothesis} which we omit here for brevity.

\subsection{\textsf{\textsc{HyPhyLearn}} for channel spoofing detection}
\label{section:HyPhyLearnforspoofing}
We propose to utilize \textsf{\textsc{HyPhyLearn}} described in Algorithm \ref{Alg:mainALG} for solving the spoofing detection problem which can be seen as a binary instance of the problem formulation described in Section \ref{section:Problem statement} with two behaviors, as described in (\ref{eq:LRT_Diff}). Besides, statistical parametric models are available for each behavior, the high complexity of which makes one to resort to suboptimal parameter estimation procedure. As mentioned in Section \ref{section:LRT-Spoofing} the data corresponding to the Alice and Eve are collected in the snapshot setting, and subsequently (imperfectly) labeled according to (\ref{eq:imperfectlabeling-spoofing}). Then, using these collected CFRs, the underlying parameters of each likelihood function in (\ref{eq:LRT_Diff}) is estimated using the sub-optimal parameter estimation procedure described in Section \ref{Section:Spoofing-paramEst}. Next, the estimated parameters are plugged in the available parametric models $\mathcal{CN}(\mathbf{0},\mathbf{R}_{\mathcal{H}_0})$ and $\mathcal{CN}(\overline{\mathbf{h}}^E-\overline{\mathbf{h}}^A,\mathbf{R}_{\mathcal{H}_1})$, which subsequently are used to generate synthetic CFRs. Finally, the collected and synthetic CFRs are incorporated in the Step $4$ of Algorithm \ref{Alg:mainALG} to train the learning-based classifier which can be utilized as a spoofing detector.

\section{Case study 2: Multi-user detection}
\label{section:Case study 2: Multi-user detection}
An important problem in multipoint-to-point digital communication
networks (e.g., radio networks, local-area networks, and uplink satellite channels)
is the optimum centralized demodulation of the information sent simultaneously
by several users through a Gaussian multiple-access channel. Even though the
users may not employ a protocol to coordinate their transmission epochs, effective
sharing of the channel is possible because each user modulates a different
signature signal waveform that is known by the intended receiver (Code
Division Multiple Access (CDMA)). In this section, we consider the uplink of a cellular communication system where $K$ users are asynchronously sharing a channel to communicate with a base station (BS). The problem of multi-user detection in this setting amounts to inferring the information associated with each user from this multiple access channel.
\subsection{Multi-user detection problem}
\label{section: Multi-user detection problem}
Consider the uplink of an asynchronous direct-sequence (DS) CDMA system shared by $K$ users,
employing long spreading codes, bandlimited chip pulses
and operating over a frequency-selective fading channel. The
baseband equivalent of the received signal may be written as
\begin{align}
    r(t)=\sum_{p=0}^{P-1}\sum_{k=0}^{K-1}A_kb_k(p)s'_{k,p}(t-\tau_k-pT_b)^*c_k(t)+w(t),
\end{align}
where $P$ is the number of transmitted packets and $s'_{k,p}(t)$ denotes the $k$th user signature waveform. Furthermore, $T_b$ is the bit-interval duration, $A_k$ and $\tau_k$ denote the respective
complex amplitude and timing offset of $k$th user, and $b_k(p)$
is the $k$th user’s information bit in the $p$th
signaling interval, whereas $w(t)$ is the complex envelope of the additive noise term, which is
assumed to be a zero-mean, wide-sense stationary complex white Gaussian process. Moreover, $c_k(t)$ is
the impulse response modeling the channel effects between
the BS and the $k$th user. Note that the channel
impulse responses $c_k(t)$ are assumed to be
time-invariant over each transmitted frame \cite{Buzzi} under the assumption that the channel coherence
time exceeds the packet duration $BT_b$. Regarding the $k$th user signature waveform, we have
\begin{align}
\label{eq:singature}
    s'_{k,p}(t-\tau_k-pT_b)=\sum_{n=0}^{N-1}\beta_{k,p}^{(n)}h_{SRRC}(t-nT_c),
\end{align}
where $\{\beta^{(n)}_{k,p}\}_{n=0}^{N-1}$ is the pseudo-noise (PN) code employed by
user $k$ for spreading its data bit on the $p$th symbol interval, $N$ is
the processing gain, and $T_c=T_b/N$ is the chip interval. Furthermore, $h_{SRRC}(t)$ denotes the square root
raised-cosine waveform as the bandlimited chip pulse which following \cite{Buzzi} is time-limited to $[0, 4Tc]$.

In the BS, chip-matched filtering and chip-rate sampling is done in order to convert the received signal to discrete time domain. To this end, $r(t)$ is convolved with chip-matched filter $h_{SRRC}(4T_c-t)$ followed by a sampler at a rate $2/T_c$ (Nyquist rate). The convolution operation results in
\begin{align}
    y(t)&=r(t)*h_{SRRC}(4T_c-t)\nonumber\\&=\sum_{p=0}^{P-1}\sum_{k=0}^{K-1}b_k(p)h_{k,p}(t-pT_b,\tau_k)+n(t),
\end{align}
where $h_{k,p}(t,\tau_k)=A_ks_{k,p}(t-\tau_k)^*c_k(t)$ is called the effective signature waveform for $s_{k,p}(t)=\sum_{n=0}^{N-1}\beta_{k,p}^{(n)}h_{RC}(t-nT_c)$, and $h_{RC}(t)$ represents a raised cosine chip waveform time-limited
to $[0, 8Tc)$. As $h_{k,p}(t-pT_b,\tau_k)$ has a
time domain support of $[pT_b, (p+2)T_b+7T_c]$, during the $p$th symbol interval
$\mathcal{I}_p=[pT_b, (p + 1)T_b]$, the contribution from
at most three bits for each user, i.e., the $p$th, the $p -1$th
and the $p-2$th ones, is observed assuming that
$\tau_k+T_m<T_b$ where $T_m$ stands for the maximum
delay spread. Therefore, sampling
the waveform $y(t)$ at rate $M/T_c$, the MN-dimensional vector
$y(p)$ collecting the data samples of the interval $\mathcal{I}_p$ can be
expressed as
\begin{align}
    &\mathbf{y}(p)=\sum_{k=0}^{K-1}[b_k(p-2)\mathbf{h}_{k,p-2}(p)+b_k(p-1)\mathbf{h}_{k,p-1}(p)\nonumber\\&\ \ \ \ \ \ \ \ \ \ \ \ \ \ \ \ \ \ \ \ \ \ \ \ \ \ \ \ \ \ \ \ \ \ \ \ +b_k(p)\mathbf{h}_{k,p}(p)]+\mathbf{n}(p)
\end{align}
where $\mathbf{h}_{k,p-i}(p)$ and $\mathbf{n}(p)$ comprise the $MN$ samples
of $h_{k,p-i}(t-(p-i)T_b,\tau_k)$, $i\in\{0,1,2\}$ and $n(t)$, respectively, during $\mathcal{I}_p$. We set $M=2$ in the following discussion. A compact representation of $\mathbf{y}(p)$ can be obtained by relying on the notion of effective chip pulse defined as $g_k(t,\tau_k)=A_kh_{RC}(t-\tau_k)^*c_k(t)$ which is supported on the interval $[0, T_b +8T_c]$. Noting that $h_{k,p}(t,\tau_k)=\sum_{i=0}^{N-1}\beta_{k,p}^{n}g_k(t-nT_c,\tau_k)$, and defining $\mathbf{g}_k\in\mathbb{C}^{MN+8M-1\times1}$ as
\begin{align}
    &\mathbf{g}_k=\big[g_k(T_c/M,\tau_k),g_k(2T_c/M,\tau_k),\dots,\nonumber\\&\ \ \ \ \ \ \ \ \ \ \ \ \ \ \ \ \ \ \ \ \ \ \ \ \ \ g_k(T_b+(8M-1)T_c/M,\tau_k)\big]^T,
\end{align}
one can write $\mathbf{h}_{k,p-i}(p)=\mathbf{C}_{k,p-i}(p)\mathbf{g}_k$, where $\mathbf{C}_{k,p-i}(p)$ is a $MN \times(MN + 8M -1)$ dimensional matrix as a function of
$\beta_{k,p}^n$, and obtained in details in ($9$)-($11$) of \cite{Buzzi}. Then, we have
\begin{align}
\label{eq:DS-CDMA-final-representation}
    \mathbf{y}(p)=\sum_{k=0}^{K-1}\mathbf{A}_k(p)\mathbf{g}_k+\mathbf{n}(p)=\mathbf{A}(p)\mathbf{g}+\mathbf{n}(p),
\end{align}
for $\mathbf{A}_k(p)=b_k(p-2)\mathbf{C}_{k,p-2}(p)+b_k(p-1)\mathbf{C}_{k,p-1}(p)+b_k(p)\mathbf{C}_{k,p}(p)$, $\mathbf{A}(p)=[\mathbf{A}_0(p),\dots,\mathbf{A}_{K-1}(p)]$, and $\mathbf{g}=[\mathbf{g}_0^T,\dots,\mathbf{g}_{K-1}^T]^T$. 

The multiuser detection problem can be cast as $2^K$-ary classification problem where the goal is to find the vector of information bits $\mathbf{b}=[b_0(p),\dots,b_{K-1}(p)]$ given a observation vector $\mathbf{y}'(p)$. Assuming all the vectors $\mathbf{b}\in\{0,1\}^K$ are a priori equiprobable the minimum distance rule gives the maximum a posteriori decision \cite{poorCDMA}. Mathematically, the multiuser detection is equivalent to solving the minimization problem $\argmin_{\mathbf{b}\in\{0,1\}^K}\ \mathbf{y}'(p)-\sum_{k=0}^{K-1}\mathbf{A}_k(p)\mathbf{g}_k$. However, the complexity of such detector is exponential in the number of users \cite{poorCDMA} and in practice sub-optimal methods like minimum mean square error (MMSE) detector \cite{poorCDMA} is utilized for performance evaluation. We note that the multiuser detection methods relies on the channel parameters and the spreading codes corresponding to each user, and we assume true knowledge of both are not available at the BS. Specifically, we consider a case where a mismatch exists between the true spreading codes \cite{CDMAmismatch}, utilized by the users, and the corresponding ones available at the BS. Besides, we assume that BS has access to $N_T$ number of training data from the $k$th user. The channel parameters are then estimated based on the available training data using the procedure described in the following section.

\subsection{Parameter estimation}
\label{section:CDMA-parameterEstimation}
The performance of multi-user detection relies heavily on the estimation of the channel parameters. We assume the channel impulse response (CIR), $c_k(t)$, takes the form of a time-invariant
multipath channel with $L$ paths, i.e., $c_k(t)=\sum_{l=0}^{L-1}\alpha_{k,l}\delta(t-\tau'_{k,l})$, which is parameterized by the complex path gains $\alpha_{k,l}$ and the corresponding path delays $\tau'_{k,l}$. The joint ML estimate of these parameters requires an exhaustive search over the continuous $K$-dimensional space $[0,T_b)^K$ which is computationally
prohibitive. It is shown in \cite{Buzzi0} that even using the conventional grid search based
scheme to find a near-ML solution NP-hard. As a workaround, alternative sub-optimal estimation methods of low-complexity are proposed
for practical settings. Notably, the authors in \cite{Buzzi} propose a two-step approach which first
estimates the the channel impulse response (CIR) using
the Least Squares (LS) criterion, and then extracts the underlying channel parameters. In particular, given the knowledge of the spreading codes and information bits for all the users in the training dataset, the overall CIR $\mathbf{g}$ may be directly estimated by
invoking the LS estimation procedure
\begin{align}
    \hat{\mathbf{g}}&=\argmin_{\mathbf{x}}\sum_{i=1}^{N_T-1}||\mathbf{y}(p)-\mathbf{A}(p)\mathbf{x}||^2\\&=\bigg[\sum_{p=0}^{N_T-1}\mathbf{A}^H(p)\mathbf{A}(p)\bigg]^{-1}\bigg[\sum_{p=0}^{N_T-1}\mathbf{A}^H(p)\mathbf{y}(p)\bigg].
\end{align}
Relying on $\hat{\mathbf{g}}$ the authors in \cite{Buzzi} propose an ad-hoc algorithm to estimate the channel parameters. Specifically, the explicit parameters to be estimated include delays $\tau_{k,l}=\tau'_{k,l}+\tau_k$, amplitudes $a_{k,l}=A_k|\alpha_{k,l}|$ and the phases $\phi_{k,l}=\arg(a_{k,l})$ for $k=0,\dots,K-1$ and $l=0,\dots,L-1$. We provide an overview of the above ad-hoc parameter estimation procedure in Appendix \ref{Appen:A heuristic approach for channel parameter estimation for CDMA system} for completeness.
\subsection{\textsf{\textsc{HyPhyLearn}} for multiuser detection}
\label{section:CDMA-proposedSolution}
We utilize \textsf{\textsc{HyPhyLearn}} to solve the problem of multiuser detection described in Section \ref{section: Multi-user detection problem} as a $2^K$-ary classification problem. In particular, due to the available statistical parametric models for each class on one hand, and lack of an estimation procedure for the underlying channel parameters which is optimal in some sense on the other hand, the multiuser detection can be framed within the problem formulation setting described in Section \ref{section:Problem statement}. We use the available data corresponding to the users in the suboptimal estimation method described in Section \ref{section:CDMA-parameterEstimation} to obtain the estimates of the channel parameters for $K$ users $\hat{{\bm{\tau}}}=[\hat\tau_{0,0},\dots,\hat\tau_{0,L-1},\dots,\hat\tau_{K-1,L-1}]$, $\hat{\mathbf{a}}=[\hat a_{0,0},\dots,\hat a_{0,L-1},\dots,\hat a_{K-1,K-1}]$ and $\hat{\mathbf{\Phi}}=[\hat\phi_{0,0},\dots,\hat\phi_{0,L-1},\dots,\hat\phi_{K-1,L-1}]$. Using these estimates along with the imperfect knowledge of spreading codes for the training data, one can utilize the parametric model (\ref{eq:DS-CDMA-final-representation}) to generate a synthetic data example corresponding to sample information bits $\mathbf{b}$. This synthetic data sample is subsequently added to the synthetic dataset along with its corresponding label $\mathbf{b}$. Then, the synthetic dataset is incorporated with the available training dataset according to the Step $4$ of the Algorithm \ref{Alg:mainALG} to find the learning-based classifier. Specifically, this classifier has $2^K$ output neurons, each corresponding to a specific information bits vector, which enables it to to serve as a data detection method for the $K$-user system.

\section{Numerical Results}
\label{section:Nuemrial Results}
In this section, we numerically evaluate the performance of our proposed solution, \textsf{\textsc{HyPhyLearn}}, described in Algorithm~\ref{Alg:mainALG} for the two case studies described in Sections \ref{Section:Case study 1: Spoofing detection via channel frequency response} and \ref{section:Case study 2: Multi-user detection}. {This involves comparing} the resulting performance against that of the existing statistical classifiers and other hybrid classification methods, and {highlighting} the superiority of our proposed solution for the problems under study.

\subsection{Spoofing detection problem}
In the Alice--Eve--Bob setting, we begin with a scenario where the coherence time of the Alice--Bob and {the} Eve--Bob channel are very large, and therefore the corresponding channel parameters are fixed between the training {and testing} stages. As mentioned in Section \ref{section:LRT-Spoofing}, the {training} data {in this problem} are collected by observing finite number of {snapshots by Bob.} The training CFRs from each snapshot are subsequently labeled using the heuristic test (\ref{eq:imperfectlabeling-spoofing}). The number of received antennas and transmit antennas at Alice and Bob is set to $2$. Also, following the discussion in \cite{CFR2} we assume Eve also uses the same number of antennas to impersonate Alice. The number of subcarriers is set to {$N_f=20$}, which makes the total number of samples associated with each CFR equal $M=80$. We assume the Alice--Bob parameters are $\sigma^2_A=20$, $\alpha^2_A=200$, $\beta_A=0.02$ and {$a^A=0.85$}, while $\sigma^2_E=26$, $\alpha^2_E=250$, $\beta_E=0.08$ and {$a^E=0.65$} are used for the Eve--Bob channel. Furthermore, {we set} $L_A=20$ and $L_E=16$ as the {number of} diffuse spectrum virtual paths, while the number of specular paths are set to $4$ for both channels in accordance with the experimental measurements reported in \cite{Rappaport-model}. 

Fig. \ref{Fig:Spoofing1} {illustrates} the spoofing detection performance {of different methods} for the above scenario averaged over $10^5$ CFRs from each Alice-Bob and Eve-Bob channel at the test stage, where the x-axis denotes {the number of snapshots} observed during the training stage. In particular, we have evaluated the performance of \textsf{\textsc{HyPhyLearn}} for this problem, as described in Section \ref{section:HyPhyLearnforspoofing}, and compared it with other classifiers designed based on {the} likelihood {ratio test} with plug-in estimates or existing ML algorithms. {By} looking at the resulting spoofing detection accuracy, {it can be seen} that the performance of the ML algorithms based on support vector machine (SVM) and {Gaussian mixture model (GMM)} is limited in this case due to limited {(and mislabeled)} training data. We note that the GMM is used as a classifier here by assigning labels to the clusters using the available labels corresponding to the reference CFRs. Specifically, we have used {the} radial basis function kernel \cite{MLMurphy} for the SVM and {two components for the GMM} for these simulations. Furthermore, one can see that {the} LRT method obtained in Section \ref{section:LRT-Spoofing} can improve upon the performance of these ML algorithms {by plugging the estimated parameters, as in Section \ref{Section:Spoofing-paramEst}, in the statistical parametric models}. {In these experiments, we also} use the {shrinkage method} \cite{shrink1} which improves the covariance matrix estimation for each likelihood function. {For this method,} a performance gain {can be observed for} this approach in comparison to the no shrinkage case, assuming the shrinkage parameter $\alpha$ is {clarivoyantly chosen} to maximize the spoofing detection accuracy over the test dataset. This method is labeled as `LRT ({best shrinkage})' in Fig. \ref{Fig:Spoofing1}. However, in practice the parameter $\alpha$ has to be estimated from the training data, which---as shown in the figure {with label `LRT (shrinkage)'}---could deteriorate the LRT performance as the available data includes mislabeled samples.

Furthermore, we evaluate the performance of an existing hybrid classification approach known as fine tuning \cite{meroune,TLSurvey-wirelesscomm} in DTL literature for this problem. In this method, we first generate $5\times10^{5}$ synthetic data {samples} using the {available likelihood parametric functions with plugged-in estimates}. Then, a neural network with $3$ hidden layers of $400$ neurons each is trained to classify the synthetic data for this example. The training data are used afterwards to refine the weights of this neural network. Notably, \textsf{\textsc{HyPhyLearn}} is shown to outperform {the aforementioned} existing classification methods by relying on both available and synthetic data and jointly using them in a learning-based classifier.

{For the sake of comparison, we have also considered a variation of \textsf{\textsc{HyPhyLearn}} that relies on a generative adversarial network (GAN) for generating synthetic data, i.e., it disregards the available physic-based models. We have observed that the performance of this approach is impacted in the limited data regime as GANs rely merely on the available training data for generating further synthetic data of similar distribution. In fact, for this example, we have verified that \textsf{\textsc{HyPhyLearn}} based on GAN needs to be trained on $20000$ data samples in order to achieve the same level of spoofing detection accuracy as \textsf{\textsc{HyPhyLearn}} based on physics-based models with $4000$ samples. Regrading the specifics of GAN, we have used a DNN of two hidden layers with $200$ neurons each as the generator, and a DNN with three hidden layers with $300$ neurons each as the discriminator}. {In our implementation of \textsf{\textsc{HyPhyLearn}}}, the number of generated synthetic data {samples} is set to  $4\times10^{5}$. {We have also used NNs with $3$ hidden layers of $400$ neurons each for $\mathbf{M}_{\bm{\psi}}$ and $h_{\bm{\phi}_1}$, while a NN with one hidden layer of $40$ neurons each is used for ${d}_{\mathbf{\zeta}}$. For all hidden layers, {the} ReLU activation function is used}. Furthermore, Adam optimizer \cite{MLMurphy} with a learning rate of $0.0001$ is used for training in this example. We also note that the optimal Bayes decision rule, which relies on the knowledge of the true parameters, results in the spoofing detection accuracy of $0.996$. 
   \begin{figure}
        \centering
         \includegraphics[width=8.2cm]{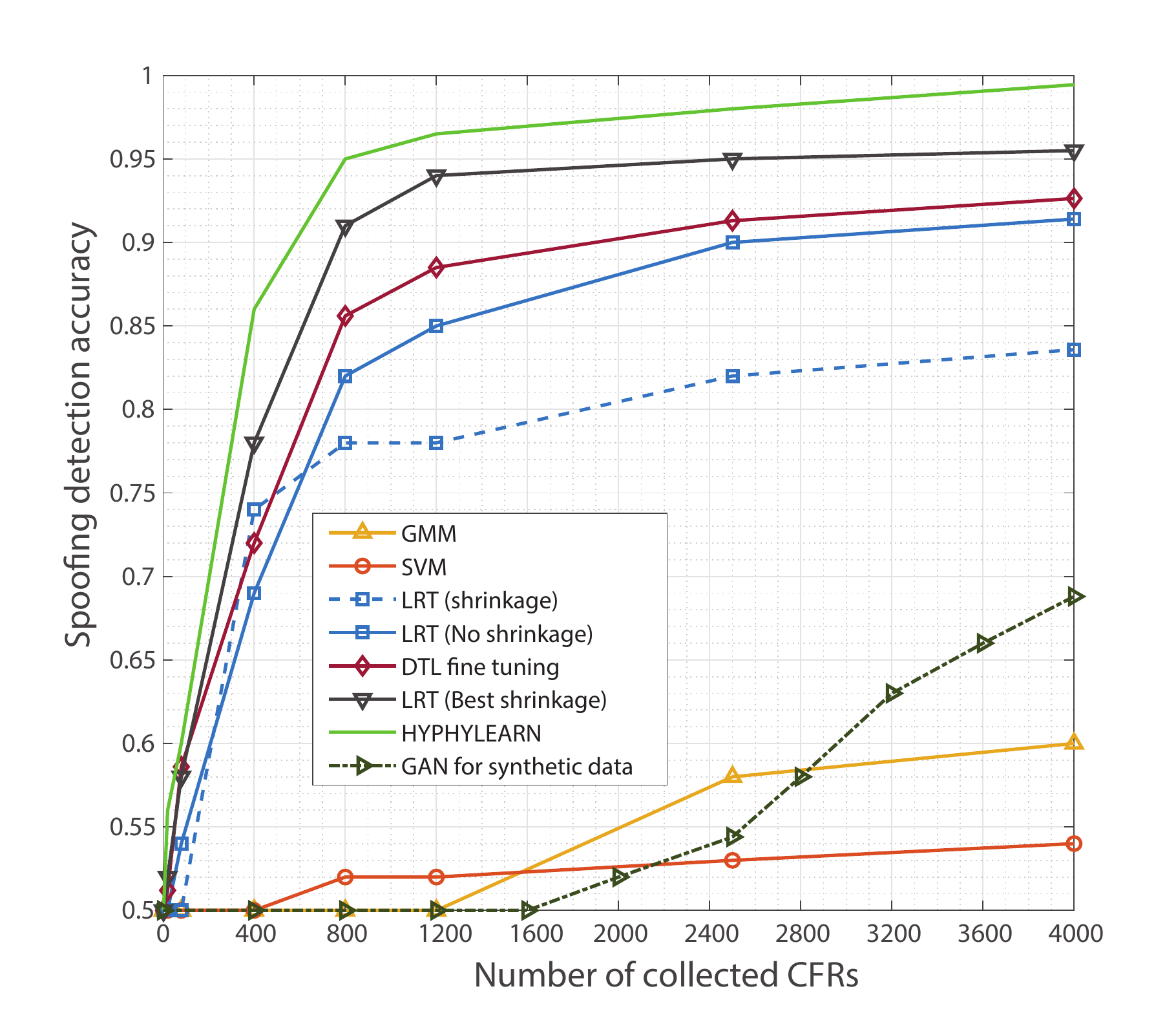} 
        \caption{Spoofing detection accuracy for different classification algorithms as a number of available training data for the case {when} training and test stage belong to the same coherence time.}
        \label{Fig:Spoofing1}
 \end{figure}
 
 Next, we consider a more realistic scenario where the channels' variations cause the training and test stage to not fall in the same coherence time. In this case, Bob uses the heuristic test (\ref{eq:imperfectlabeling-spoofing}) for some time as it does not have access to the channel parameters in this period. Afterwards, it uses the data collected in the previous coherence times to estimate the channel parameters for the current one. Fig. \ref{Fig:coherence} depicts this setting where the training stage consists of $n_c$ coherence times corresponding to {the Alice--Bob channel}. Furthermore, in contrast to Alice, Eve's transmissions are assumed to be intermittent due to the uncertainty associated with Eve's behaviour. During {each coherence time corresponding to the Alice-Bob channel}, it is assumed that Bob collects $100$ training data. Then, the estimation technique described in Section \ref{Section:Spoofing-paramEst} is utilized to estimate the channel parameters under each coherence time. Fig. \ref{Fig:Spoofing2} demonstrates the system performance as a function of number of coherence times in the training stage. Regarding the physical {setup}, we have used the same system parameters as those in Fig. \ref{Fig:Spoofing1}, and assumed that the coherence time of the Alice--Bob channel is {$4$ times that of the Eve--Bob channel} for illustrative purpose. For DTL fine-tuning approach and \textsf{\textsc{HyPhyLearn}}, the number of synthetic data generated for each behavior in a coherence time is {set to $20000$}. For these two learning-based approaches, the training specifications for are chosen to be the same as the ones used in Fig. \ref{Fig:Spoofing1}. The performance comparison again highlights the superiority of \textsf{\textsc{HyPhyLearn}} in comparison to the existing statistical and data-driven methods. 
    \begin{figure}
        \centering
         \includegraphics[width=7.6cm]{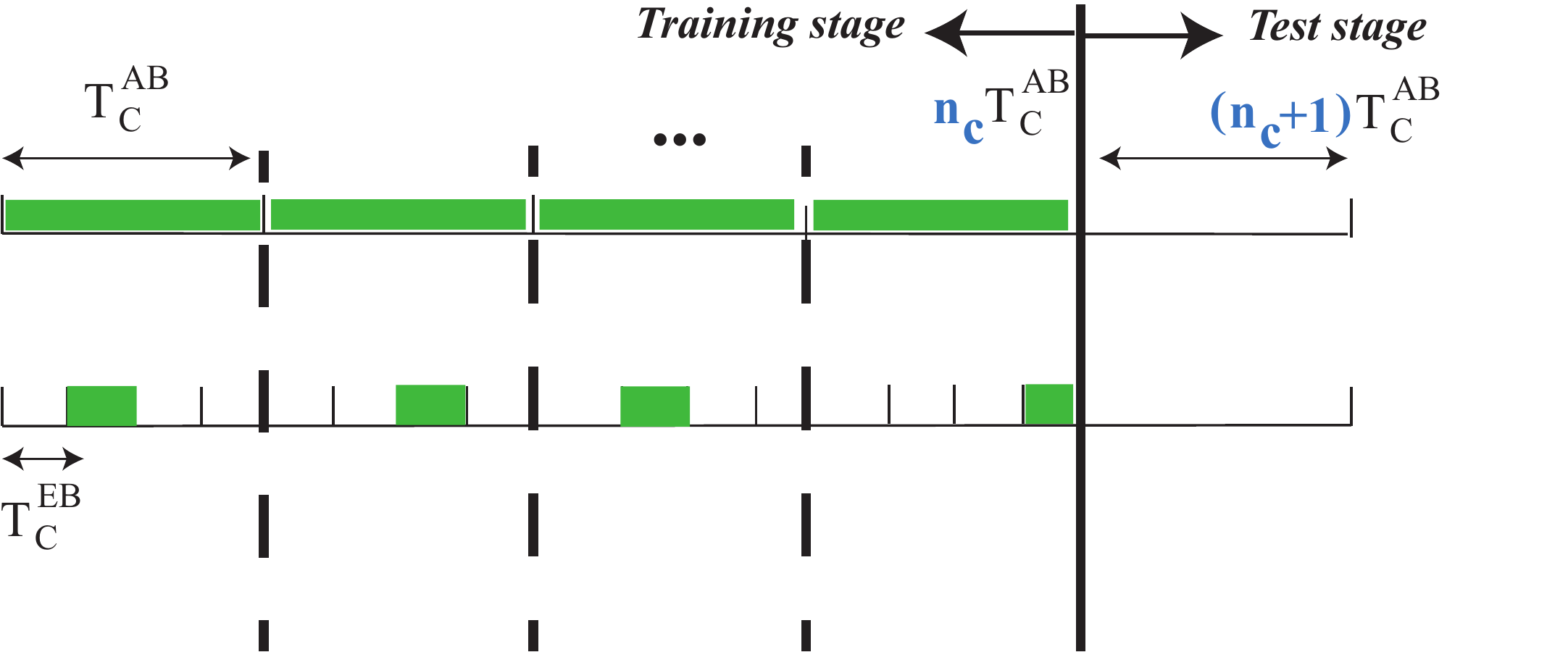} 
        \caption{Training and testing stages for the spoofing detection problem. $T_C^{AB}$ and $T_C^{EB}$ denote the coherence time corresponding to the Alice-Bob and Eve-Bob channels, respectively. The green bar indicates the time interval within which a snapshot is observed by Bob.}
        \label{Fig:coherence}
 \end{figure}
     \begin{figure}
        \centering
         \includegraphics[width=8cm]{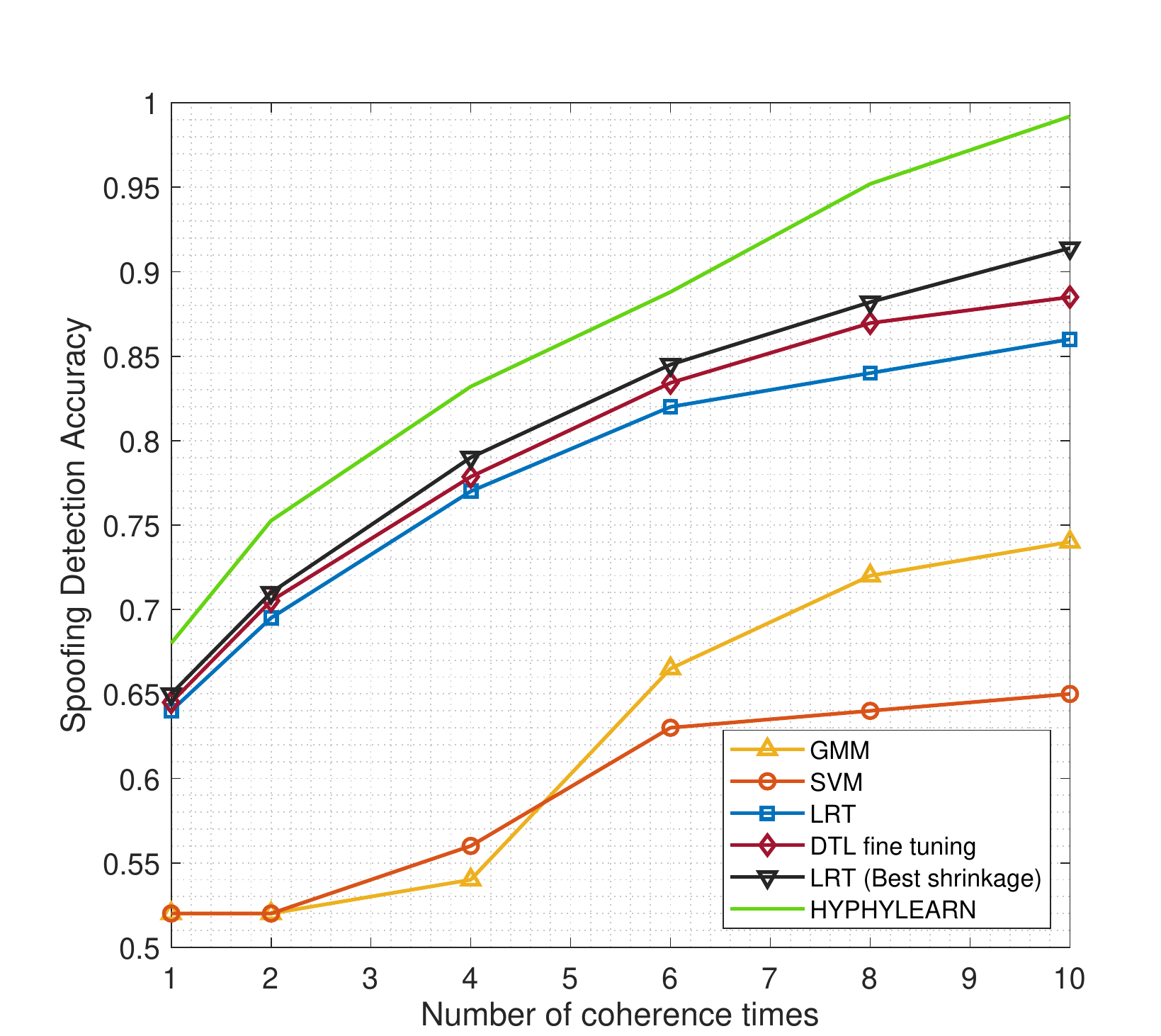} 
        \caption{Spoofing detection accuracy for the case where Bob collects training data during certain number of coherence times before employing a classification algorithm.}
        \label{Fig:Spoofing2}
 \end{figure}
\subsection{Multi-user detection problem}
In this section we present {results of} numerical simulations
to investigate the effectiveness of \textsf{\textsc{HyPhyLearn}} described in Section \ref{section:CDMA-proposedSolution} for the MUD problem. We choose the simulation parameters based on the setting described in \cite{Buzzi} and consider a system with processing gain of $N=32$ where the number of users is either $K=3$ or $K=5$. Golden codes of length $32$ are used by the BS as the pseudo-noise code in (\ref{eq:singature}) and the users' amplitudes ($A_k$'s) are set to $2$. {In addition}, a chip interval of length $T_c=0.001$ and a sampling rate of $2/T_c$ is employed. A near-far
ratio (NFR) of $10$ dB is assumed, which means the users' amplitude are randomly
unbalanced around $2$ with a variance of
$\pm5$ dB. The fading channel between the users and the BS consists of $3$ paths, which makes the total number of unknown parameters in Section \ref{section:CDMA-parameterEstimation} to be $9K$. We further consider a setting where {the BS might not have access to the perfect knowledge of the pseudo-noise sequences for all the users} at the time of detection, which would lead to a mismatched situation. To account for this phenomenon, we introduce a parameter $\rho$ {that} in order to quantify the averaged error in the {pseudo-noise sequences at the BS} while decoding.

As the performance metric, we consider {the} bit error rate (BER) at {the} BS while decoding the users' information bits, which is of major interest in digital communication systems.  As the MUD
algorithm we employ the minimum mean square error (MMSE) decoder introduced in \cite{Madhow}, which is shown to outperform other existing detection methods including matched filter receiver and box-constrained maximum likelihood detector \cite{Buzzi}. As mentioned in Section \ref{section:Case study 2: Multi-user detection}, MUD can be also solved by a classifier aiming at distinguishing between $2^K$ different classes each representing a unique decoded sequence of information bits. In this case, BER is directly related to the classification accuracy of the trained classifier. For the asynchronous system discussed in Section \ref{section:Case study 2: Multi-user detection}, 
the interval $\mathcal{I}_{2p} = [pT_b, (p+2)T_b]$ contains most of the energy content of the information
symbol $b_{k}(p)$. Therefore, it is sufficient for the MUD detector to process the data in the interval $\mathcal{I}_{2p}$ in order to
obtain estimates of the symbols $b_k(p)$, $\forall k = 0,\dots,K-1$.

 \begin{figure*}[b]
   \small
 \begin{subequations}
\begin{align}
2R_{\mathcal{Z}_{{\bm{\psi},{\bm{\theta}^*}}}}(\mathcal{H}_\Phi)&+2R_{\mathcal{Z}_{\bm{\psi},{\hat{\bm{\theta}}}}}(\mathcal{H}_\Phi)+3\sqrt{(\log{2\delta})/2N_r}+3\sqrt{(\log{2\delta})/2N_s}\geq\label{eq:lemma_A_empirical_line0}\\&
    \sup_{h_\phi\in\mathcal{H}_\Phi}\bigg|\int_{A_\phi}p_{{\bm{\psi},{\bm{\theta}^*}}}(\mathbf{z})d\mathbf{z}-\sum_{i=1}^{N_r}\mathds{1}_{\{h_\phi(\mathbf{z}_{r,i})=1\}}\bigg|+\sup_{h_\phi\in\mathcal{H}_\Phi}\bigg|\int_{A_\phi}p_{\bm{\psi},{\hat{\bm{\theta}}}}(\mathbf{z})d\mathbf{z}-\sum_{i=1}^{N_s}\mathds{1}_{\{h_\phi(\mathbf{z}_{s,i})=1\}}\bigg|\geq\\&
      \sup_{h_\phi\in\mathcal{H}_\Phi}\bigg|\int_{A_\phi}p_{{\bm{\psi},{\bm{\theta}^*}}}(\mathbf{z})d\mathbf{z}-\sum_{i=1}^{N_r}\mathds{1}_{\{h_\phi(\mathbf{z}_{r,i})=1\}}-\bigg(\int_{A_\phi}p_{\bm{\psi},{\hat{\bm{\theta}}}}(\mathbf{z})d\mathbf{z}-\sum_{i=1}^{N_s}\mathds{1}_{\{h_\phi(\mathbf{z}_{s,i})=1\}}\bigg)\bigg|\geq\label{eq:lemma_A_empirical_line2}\\&
      \sup_{h_\phi\in\mathcal{H}_\Phi}\bigg|\int_{A_\phi}p_{{\bm{\psi},{\bm{\theta}^*}}}(\mathbf{z})d\mathbf{z}-\int_{A_\phi}p_{\bm{\psi},{\hat{\bm{\theta}}}}(\mathbf{z})d\mathbf{z}\bigg|-\sup_{h_\phi\in\mathcal{H}_\Phi}\bigg|\sum_{i=1}^{N_r}\mathds{1}_{\{h_\phi(\mathbf{z}_{r,i})=1\}}-\sum_{i=1}^{N_s}\mathds{1}_{\{h_\phi(\mathbf{z}_{s,i})=1\}}\bigg|\label{eq:lemma_A_empirical_line3}=\\&
      d_{\mathcal{A}_\Phi}\big(p_{{\bm{\psi},{\bm{\theta}^*}}}(\mathbf{z}),p_{\bm{\psi},{\hat{\bm{\theta}}}}(\mathbf{z})\big)-\hat{d}_{\mathcal{A}_\Phi}(\mathcal{Z}_{{\bm{\psi},{\bm{\theta}^*}}},\mathcal{Z}_{\bm{\psi},{\hat{\bm{\theta}}}}),\label{eq:lemma_A_empirical_line4}
\end{align}
 \end{subequations}
 \end{figure*}
 
We present simulation results {for} the performance of the MMSE detector in the above setting in Fig. \ref{Fig:CDMA1}, and compare it with our proposed approach in Section \ref{section:CDMA-proposedSolution}. Specifically, {the} parameter estimation procedure {for \textsf{\textsc{HyPhyLearn}}} is done under two different levels of model mismatch, i.e., {$\rho=0.2$} and $\rho=0.25$. Furthermore, the number of {training} data available from each user $N_T$ is set to $40$. As a general observation, Fig. \ref{Fig:CDMA1} demonstrates that the performance of {all the detectors} is deteriorated as the number of users and the value of $\rho$ is increased. The perfect MMSE is referred to the case where the true pseudo-noise sequences are assumed to be known as part of the implementation of the decoder. In particular, huge performance gap between the perfect MMSE and the MMSE {decoder} indicates the high sensitivity of {the MMSE} detector to the mismatch. On the other hand, it is also highlighted that our proposed approach can achieve a substantial gain over a wide range of {SNRs} by dealing with the mismatch problem. For \textsf{\textsc{HyPhyLearn}}, the number of generated synthetic data is set to $10^{6}$ for this example. We have also used NNs with $4$ hidden layers of $300$ neurons each for $\mathbf{M}_{\bm{\psi}}$ and $h_{\bm{\phi}_1}$ here. Also, a shallow NN with one hidden layer of $40$ neurons is used for ${d}_{\mathbf{\zeta}}$, while ReLU activation function is used for all the hidden layers. During training, Adam optimizer with a learning rate of $0.0001$ is utilized as the stochastic gradient descent algorithm.
\begin{figure} 
        \centering
         \includegraphics[width=8.5cm]{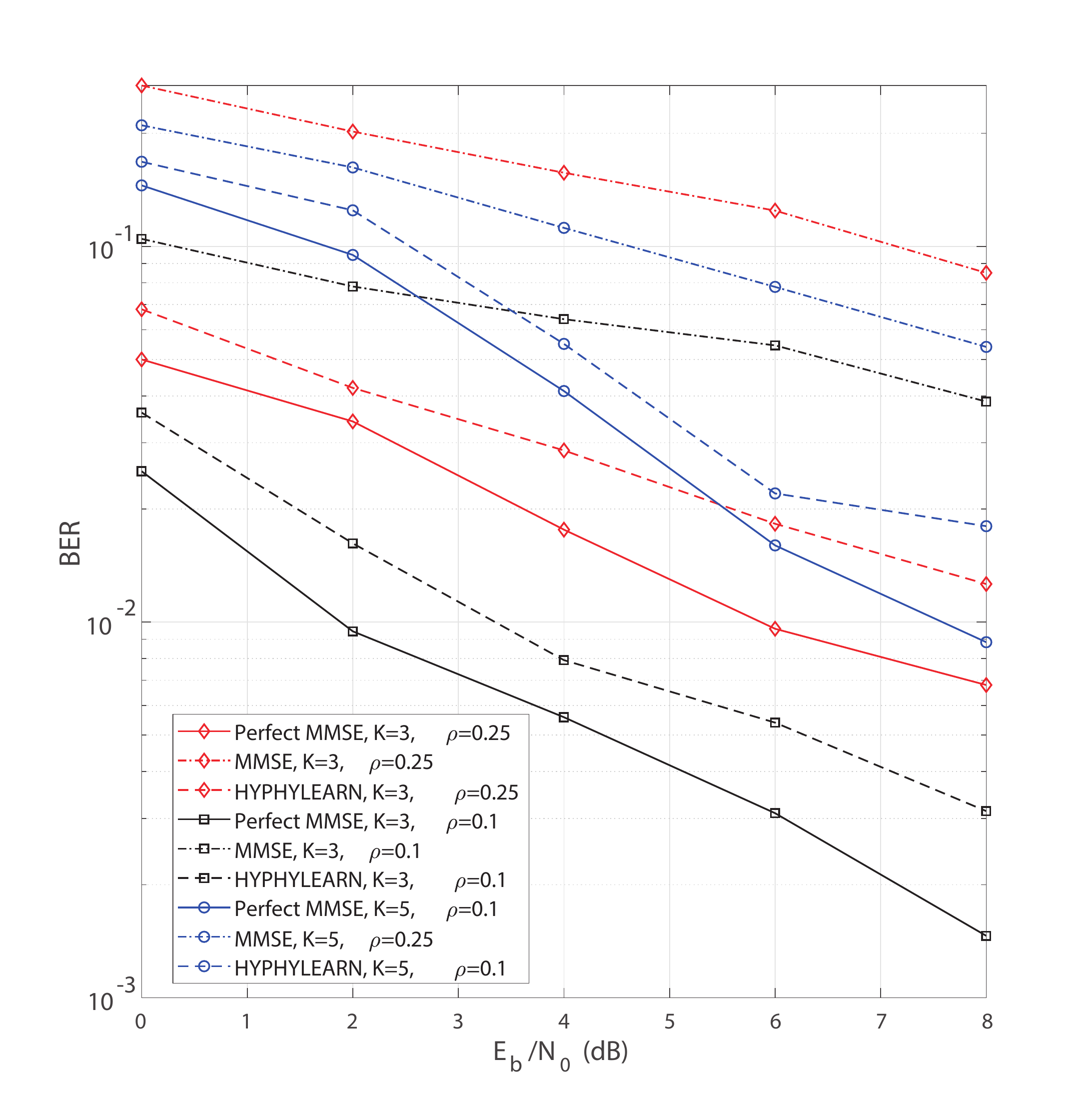} 
        \caption{BER performance of the MMSE multi-user detector and \textsf{\textsc{HyPhyLearn}} as a function of SNR. The results are provided for two different parameters, i.e., the number of users ($K$) and the mismatch parameter ($\rho$).}
        \label{Fig:CDMA1}
 \end{figure}
 \normalsize
 
In Fig. \ref{Fig:CDMA2}, the BER performance of the multi-user detectors is investigated as a function of number of available training {data}. For this example, SNR at the BS is assumed to be fixed at the BS according to $8$ dB. It is demonstrated that increasing the number of data {samples} does not lead to substantial performance improvements in the case of MMSE method. This is attributed to the aforementioned mismatch phenomenon in the pseudo-noise sequences which prevents the MMSE detector from benefiting {from} the larger amount of data considerably. Furthermore, it is further shown that the performance gap between \textsf{\textsc{HyPhyLearn}} and the perfect MMSE shrinks as {the} number of data increases. However, the degree to which this gap decreases is higher for the case of $\rho=0.1$ in comparison to that of $\rho=0.25$. Indeed, \textsf{\textsc{HyPhyLearn}} gets more benefit from the data at lower levels of mismatch where the parameter estimates enjoy higher levels of accuracy.
  \begin{figure}
        \centering
         \includegraphics[width=8.5cm]{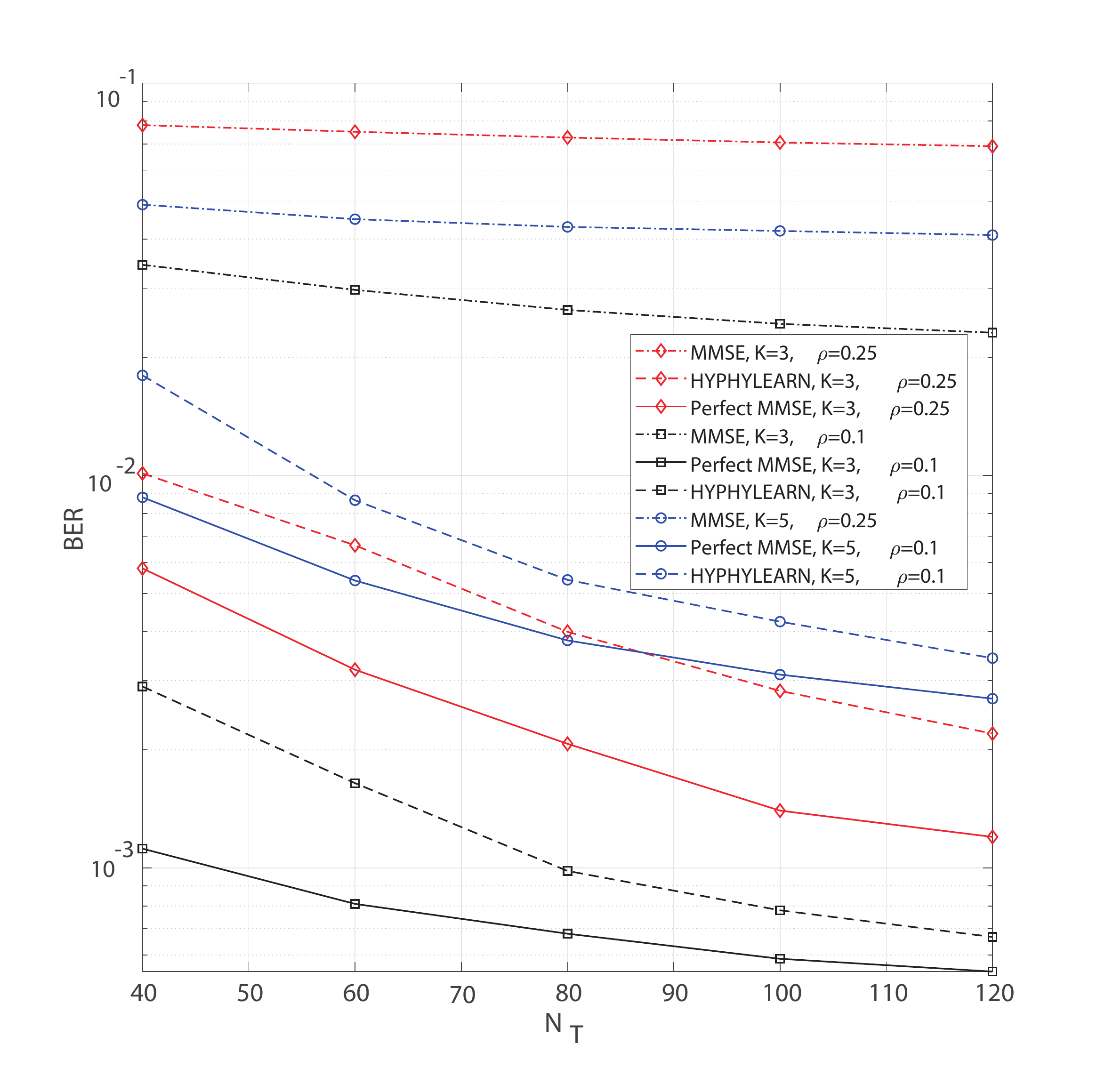} 
        \caption{BER performance of the MUD as a function of {the} number of {training} data available at each user.}
        \label{Fig:CDMA2}
 \end{figure}
\section{Conclusions}
\label{section:Conclusions}
We have considered the problem of hypothesis testing in the context of parametric classification where there is known model for each hypotheses but the corresponding parameters are unknown. Towards designing a classifier in this setting, we have taken into account several practical considerations including the assumptions that available training data are limited and there could be labeling errors associated with them. Furthermore, the model under each hypothesis is assumed to be complex such that the MLEs of its parameters is computationally intractable. In this vein, we have proposed to use sub-optimal parameter estimation algorithms for this purpose and generate synthetic data leveraging the model knowledge. Then, we have utilized the domain adversarial framework for learning a classifier using these synthetic data and the empirical training data. We have shown the applicability of our proposed approach in two tangible communication scenarios, i.e., spoofing detection and multiuser detection problems where the rich and complex models are available for the real data. We have finally shown through numerical results the superiority of our proposed approach in designing a classifier under the aforementioned practical limitations w.r.t to the existing statistical and machine learning methods.  
\appendices
  \section{Proof of Lemma \ref{lemma:Rademachar-A_distance}}
  \label{Appen:lemma:Rademachar-A_distance}
  We apply Lemma \ref{lemma:technical1} to the distributions $p_{{\bm{\psi},{\bm{\theta}^*}}}(\mathbf{z})$ and $p_{\bm{\psi},{\hat{\bm{\theta}}}}(\mathbf{z})$ for the functions of the form  $\mathds{1}_{\{h_{\bm{\phi}}(\mathbf{z})=1\}}$ where $h_{\bm{\phi}}\in \mathcal{H}_{\bm{\Phi}}$.
The resulting inequality for $p_{{\bm{\psi},{\bm{\theta}^*}}}(\mathbf{z})$, for instance, would be $2R_{\mathcal{Z}_{{\bm{\psi},{\bm{\theta}^*}}}}(\mathcal{H}_\Phi)+3\sqrt{(\log{2\delta})/2N_r}\geq
   \int_{A_{\bm{\phi}}}p_{{\bm{\psi},{\bm{\theta}^*}}}(\mathbf{z})d\mathbf{z}-\sum_{i=1}^{N_r}\mathds{1}_{\{h_\phi(\mathbf{z})=1\}}$ where ${A}_{\bm{\phi}}=\{\mathbf{z}|h_{\bm{\phi}}(\mathbf{z})=1,\mathbf{z}\in\mathcal{Z},h_{\bm{\phi}}\in\mathcal{H}_{\bm{\Phi}}\}$. By summing the corresponding sides of the resulting inequalities, we can write (\ref{eq:lemma_A_empirical_line0})-(\ref{eq:lemma_A_empirical_line4}) at the bottom of this page where (\ref{eq:lemma_A_empirical_line2}) and (\ref{eq:lemma_A_empirical_line3}) follows from the inequalities $|C|+|D|\geq|C-D|\geq|C|-|D|$.
       
    \section{Proof of Theorem \ref{theorem:mainbound}}
  \label{Appen:theorem:mainbound}
  Starting from adding and subtracting the terms, $ \mathbb{P}_{\bm{\psi},{\hat{\bm{\theta}}}}[{e}_{\bm{\phi}_1}]$ to one side of $\mathbb{P}_{{\bm{\psi},{\bm{\theta}^*}}}[{e}_{\bm{\phi}_1}]= \mathbb{P}_{{\bm{\psi},{\bm{\theta}^*}}}[{e}_{\bm{\phi}_1}]$, we get
\small
   \begin{subequations}
\begin{flalign}
&\mathbb{P}_{{\bm{\psi},{\bm{\theta}^*}}}[{e}_{\bm{\phi}_1}]=  \mathbb{P}_{{\bm{\psi},{\bm{\theta}^*}}}[{e}_{\bm{\phi}_1}]+  \mathbb{P}_{\bm{\psi},{\hat{\bm{\theta}}}}[{e}_{\bm{\phi}_1}]- \mathbb{P}_{\bm{\psi},{\hat{\bm{\theta}}}}[{e}_{\bm{\phi}_1}]\leq\\&
   \mathbb{P}_{\bm{\psi},{\hat{\bm{\theta}}}}[{e}_{\bm{\phi}_1}]+\big|\mathbb{P}_{\bm{\psi},{\hat{\bm{\theta}}}}[{e}_{\bm{\phi}_1}]-\mathbb{P}_{\bm{\psi},{\hat{\bm{\theta}}}}[{e}_{\bm{\phi}_1}]\big|\leq\\&
   \mathbb{P}_{\bm{\psi},{\hat{\bm{\theta}}}}[{e}_{\bm{\phi}_1}]+
    \frac{1}{2}d_{{\mathcal{B}}_\Phi}(p_{{\bm{\psi},{\bm{\theta}^*}}}(\mathbf{z}),p_{\bm{\psi},{\hat{\bm{\theta}}}}(\mathbf{z}))\label{eqn:line-2}\leq \\& \mathbb{P}_{\bm{\psi},{\hat{\bm{\theta}}}}[{e}_{\bm{\phi}_1}]+
    \frac{1}{2}\hat{d}_{\mathcal{A}_\Phi}(\mathcal{Z}_r,\mathcal{Z}_s)+R_{\mathcal{Z}_r}(\mathcal{H}_\Phi)+R_{\mathcal{Z}_s}(\mathcal{H}_\Phi)\\&+\frac{3}{2}\sqrt{(\log{2/\delta})/2N_r}+\frac{3}{2}\sqrt{(\log{2/\delta})/2N_s)}\label{eqn:line-3},
\end{flalign}
\end{subequations}
\normalsize
where (\ref{eqn:line-2}) stems from the definition of $d_{\mathcal{B}_\Phi}$. Also, (\ref{eqn:line-3}) is a result of Lemma \ref{lemma:Rademachar-A_distance} and noting that $d_{\mathcal{A}_{\Phi}}$ is an upper bound for $d_{\mathcal{B}_{\Phi}}$.
\section{Proof of Lemma \ref{lemma:Dist-of-CFR}}
\label{Appen:lemma:Dist-of-CFR}
Note that $\mathbf{d}_k$ in (\ref{eq:variablepart-channel}) is a linear combination of $L$ Gaussian random variables $A_{u,l}\sim\mathcal{CN}\big(\mathbf{0},\text{Var}(A_{u,l})\big)$ where $\mathbb{E}[A_{u,l_1}A_{u,l_2}]$ for $\forall l_1\neq l_2$ under WSSUS assumption. Therefore, $\mathbf{d}_k$ is also Gaussian with mean
\small
\begin{align}
    &\mathbb{E}\big[\mathbf{q}_u[m]\big]=\sum_{l=0}^{L-1} \mathbb{E}\big[A_{u,l}e^{-j2\pi(f_0-W/2+m\Delta f)l/W}\big]=\nonumber\\&\sum_{l=0}^{L-1} \mathbb{E}\big[A_{u,l}\big]e^{-j2\pi(f_0-W/2+m\Delta f)l/W}=0,
\end{align}
\normalsize
and variance
\begin{align}
    &\text{Var}\big[\mathbf{q}_u[m]\big]=\sum_{l=0}^{L-1}  \text{Var}\big[A_{u,l}e^{-j2\pi(f_0-W/2+m\Delta f)l/W}\big]=\nonumber\\&\sum_{l=0}^{L-1}  \text{Var}\big[A_{u,l}\big]=\alpha^2(1-e^{-2\pi\beta L}).
\end{align}
The diagonal elements of $\mathbf{R}$ equal to $\text{Var}\big[\mathbf{q}_u[m]\big]$. For the $(m,n)$th element ($m\neq n$), on the other hand, we can write
\begin{subequations}
\begin{align}
    &\text{Cov}[\mathbf{q}_u[m],\mathbf{q}_u[n]]=\mathbb{E}\big[\mathbf{q}_u[m]\mathbf{q}_u[n]^*\big]\label{eq:cov-for-the-CFR__1}\\
    &=\sum_{l=0}^{L-1}  \mathbb{E}\big[A_{u,l}A_{u,l}\big]e^{-j2\pi[(f_0-W/2+m\Delta f)l-(f_0-W/2+n\Delta f)l]/W}\label{eq:cov-for-the-CFR__2}\\
    &=\sum_{l=0}^{L-1}  \text{Var}\big[A_{u,l}A_{u,l}\big]e^{j2\pi(n-m)\Delta fl/W}\\&=\sum_{l=0}^{L-1}  \sigma^2_T(1-e^{-2\pi\beta})e^{-2\pi\beta L}e^{j2\pi(n-m)\Delta fl/W}\label{eq:cov-for-the-CFR__3}\\
    &=\frac{\alpha^2(1-e^{-2\pi\beta})(1-e^{-2\pi L(\beta-\frac{(n-m)j}{M})})}{(1-e^{-2\pi (\beta-\frac{(n-m)j}{M})})}.\label{eq:cov-for-the-CFR__4}
\end{align}
\label{eq:cov-for-the-CFR}
\end{subequations}
As $\text{Cov}[\mathbf{q}_u[m],\mathbf{q}_u[n]]$ only depends on the difference $n-m$, and it equals to complex conjugate of $\text{Cov}[\mathbf{q}_u[n],\mathbf{q}_u[m]]$ the proof is completed.
\section{Proof of Lemma \ref{lemma:Dist_Null_Hypothesis}}
\label{Appen:lemma:Dist_Null_Hypothesis}
Noting that $\mathbf{q}_{u+1}=\mathbf{h}^A_{u+1}$ under $\mathcal{H}_0$ along side with Lemma \ref{lemma:Dist-of-CFR}, we conclude that  $\mathbf{q}_{u+1}-\mathbf{q}^A_{u}$ is normally distributed with zero mean.
Regarding the covariance matrix derivation, we first note that 
\begin{subequations}
\begin{align}
    &\mathbb{E}\big[A^A_{u+1,l}A^A_{u,l}\big]=\mathbb{E}\big[a^AA^A_{u,l}A^A_{u,l}\big]+\\&\ \ \ \ \ \ \ \ \ \ \ \ \mathbb{E}\big[\sqrt{(1-(a^A)^2)\text{Var}(A_{u+1,l})}u_{k+1,l}A^A_{u,l}\big]\\
    &=a^A(\alpha^A)^2(1-e^{-2\pi\beta^A})e^{-2\pi\beta^A l}.
\end{align}
\label{eq:lemma-null-temporal}
\end{subequations}
As an immediate result of the above equation and (\ref{eq:cov-for-the-CFR__2}), we can write $
    \mathbb{E}\big[\mathbf{q}^A_{u+1}[m]\mathbf{q}^A_{u}[n]^*\big]=a^A\kappa\bigg(\theta_{\mathbf{q}}^A,\frac{n-m}{M}\bigg)$. Then, the diagonal elements of $\mathbf{R}_{\mathbf{q},\mathcal{H}_0}$, we have
    \small
\begin{subequations}
\begin{flalign}
    &\text{Var}\big[\mathbf{q}^A_{u+1}[m]-\mathbf{q}^A_{u}[m]\big]=2\text{Var}[\mathbf{q}^A_{u}[m]]-2\mathbb{E}\big[\mathbf{q}^E_{u+1}[m]\mathbf{q}^A_{u}[m]^*\big]\\
    &=2\kappa(\theta_{\mathbf{q}}^A,0)-2a^A\kappa(\theta_{\mathbf{q}}^A,0)=2(1-a^A)\kappa(\theta_{\mathbf{q}}^A,0).&
\end{flalign}
\end{subequations}
\normalsize
 For the off-diagonal elements ($m\neq n)$ we can write
 \small
\begin{subequations}
\begin{flalign}
    &\text{Cov}\big[\mathbf{q}^A_{u+1}[m]-\mathbf{q}^A_{u}[m],\mathbf{q}^A_{u+1}[n]-\mathbf{q}^A_{u}[n]\big]\\
    &=\mathbb{E}[\mathbf{q}^A_{u+1}[m]\mathbf{q}^A_{u+1}[n]^*]-\mathbb{E}[\mathbf{q}^A_{u+1}[m]\mathbf{q}^A_{u}[n]^*]\\&-\mathbb{E}[\mathbf{q}^A_{u+1}[n]\mathbf{q}^A_{u}[m]^*]+\mathbb{E}[\mathbf{q}^A_{u}[m]\mathbf{q}^A_{u}[n]^*]\\&=\kappa\bigg(  \theta_{\mathbf{q}}^A,\frac{n-m}{M}\bigg)-a^A\kappa\bigg(\theta_{\mathbf{q}}^A,\frac{n-m}{M}\bigg)-a^A\kappa\bigg(\theta_{\mathbf{q}}^A,\frac{n-m}{M}\bigg)\\&+\kappa\bigg(\theta_{\mathbf{q}}^A,\frac{n-m}{M}\bigg)=2(1-a^A)\kappa\bigg(\theta_{\mathbf{q}}^A,\frac{n-m}{M}\bigg).
\end{flalign}
\end{subequations}
\normalsize
Finally, as the values of the off-diagonal elements only depend on the difference between the indices, the Toeplitz structure of $\mathbf{R}_{q,\mathcal{H}_0}$ is deduced.
\section{Proof of Lemma \ref{lemma:Dist-Alternate-Hypothesis}}
\label{Appen:lemma:Dist-Alternate-Hypothesis}
Similar to the null hypothesis, normality of $\mathbf{q}_{u+1}-\mathbf{q}^A_u|\mathcal{H}_1$ with a zero mean is deduced from Lemma \ref{lemma:Dist-of-CFR}. Now, considering (\ref{eq:lemma-null-temporal}) with a similarity parameter $a^E$ along with (\ref{eq:cov-for-the-CFR__2}) we can write
\begin{align}
    \mathbb{E}\big[\mathbf{d}^E_{u+1}[m]\mathbf{d}^A_{u}[n]^*\big]=a^E\kappa\bigg(\theta_{\mathbf{q}}^A,\frac{n-m}{M}\bigg).
\end{align}
Then, the diagonal element of $\mathbf{R}_{\mathbf{q},\mathcal{H}_1}$ can be computed through
\begin{subequations}
\begin{align}
    &\text{Var}\big[\mathbf{q}^E_{u+1}[m]-\mathbf{q}^A_{u}[m]\big]
    =\text{Var}\big[\mathbf{q}^E_{u+1}[m]\big]+\text{Var}\big[\mathbf{q}^A_{u}[m]\big]\\&-2\mathbb{E}\big[\mathbf{q}^E_{u+1}[m]\mathbf{q}^A_{u}[m]^*\big]=\kappa(\theta_{\mathbf{q}}^E,0)+\kappa(\theta_{\mathbf{q}}^A,0)\\&-2a^E\kappa(\theta_{\mathbf{q}}^A,0)=\kappa'(a^E,\theta_{\mathbf{q}}^A,\theta_{\mathbf{q}}^E,0).
\end{align}
\end{subequations}
Similarly for the off-diagonal elements we can write
\small
\begin{subequations}
\begin{align}
    &\text{Cov}\big[\mathbf{q}^E_{u+1}[m]-\mathbf{q}^A_{u}[m],\mathbf{q}^E_{u+1}[n]-\mathbf{q}^A_{u}[n]\big]=\mathbb{E}[\mathbf{q}^E_{u+1}[m]\mathbf{q}^E_{u+1}[n]^*]\\
    &-\mathbb{E}[\mathbf{q}^E_{u+1}[m]\mathbf{q}^A_{u}[n]^*]-\mathbb{E}[\mathbf{q}^E_{u+1}[n]\mathbf{q}^A_{u}[m]^*]+\mathbb{E}[\mathbf{q}^A_{u}[m]\mathbf{q}^A_{u}[n]^*]\\&=\kappa\bigg(\theta_{\mathbf{q}}^E,\frac{n-m}{M}\bigg)-2a^E\kappa\bigg(\theta_{\mathbf{q}}^A,\frac{n-m}{M}\bigg)
    +\kappa\bigg(\theta_{\mathbf{q}}^A,\frac{n-m}{M}\bigg)\\&=
    \kappa'(a^E,\theta_{\mathbf{q}}^A,\theta_{\mathbf{q}}^E,\frac{n-m}{M}),
\end{align}
\end{subequations}
\normalsize
which imposes a Toeplitz structure for $\mathbf{R}_{\mathbf{q},\mathcal{H}_1}$.

\section{A heuristic approach for channel parameter estimation for CDMA system}
\label{Appen:A heuristic approach for channel parameter estimation for CDMA system}
In this appendix, we present an overview of the channel parameter estimation technique described in \cite{Buzzi} for completeness. {The} estimation process start with finding the parameters of a single path, i.e., it initially assumes $L=1$. Then, it forms an $MN+8M-1\times 1$ vector $\mathbf{m}_k=\widehat{\mathbf{g}}_k\circ \widehat{\mathbf{g}}^*_k$ for the $k$th user given $\widehat{\mathbf{g}}_k$. Next, it obtains the sliding window correlation between the entries
of $\mathbf{m}_k$ and the samples of the raised cosine waveform given by
\small
\begin{align}
    \mathbf{q}_k(l)=\sum_{i=1}^{8M-1}\mathbf{m}_k    (l+i-1)|h_{RC}(iT_c/M)|^2,\ l=1,\dots,MN+1.
\end{align}
\normalsize
It is argued in \cite{Buzzi} the index of the maximum element of $\mathbf{q}_k$ denoted by $i_k$ gives information on the $k$th user's delay. Subsequently, the entries of $\mathbf{m}_k$ contributing to this peak are denoted by $\mathbf{p}_k=[\mathbf{m}_k(i_k+1),\dots,\mathbf{m}_k(i_k+8M)]$. Next, an interval $[(i_k-2)T_c/M+Tc/(10M),(i_k+2)T_c/M-Tc/(10M)]$ with a predefined resolution of $T_c/10M$ is spanned. Then, an $n'$ is found as the index for which $\gamma^T_{n'}\mathbf{p}_k>\max_{n\neq n'}\gamma^T_n\mathbf{p}_k$, where 

\small
\begin{align}
&\gamma_n=\Big[\big|h_{RC}\big(\frac{T_c}{M}+\frac{nT_c}{10M}\big)\big|^2,\big|h_{RC}\big(\frac{2T_c}{M}+\frac{nT_c}{10M}\big)\big|^2,\dots,\nonumber\\&\ \ \ \ \ \ \ \ \ \ \ \ \ \ \ \ \ \ \ \ \ \ \ \ \ \ \ \ \ \ \ \ \ \ \ \ \ \ \ \  \big|h_{RC}\big(\frac{8T_c}{M}+\frac{nT_c}{10M}\big)\big|^2\Big],
\end{align}
\normalsize
and {$n\in[-19,\dots,19]$}. In this way, the delay can be estimated by $\widehat{\tau}_{k,0}=i_k\frac{T_c}{M}+n'\frac{T_c}{10M}$ with an approximation error of $T_c/(10M)$. Regarding estimation of the phase and the amplitude, first the following vectors of length $8M$ are obtained:

\small
\begin{flalign}
&\mathbf{\Psi}_{k,0}=\bigg[h_{RC}\Big(\frac{(i_k+1)T_c}{M}-\widehat{\tau}_{k,0}\Big),\dots,\nonumber\\&\ \ \ \ \ \ \ \ \ \ \ \ \ \ \ \ \ \ \ \ \ \ \ \ \ \ \ \ \ \ \ \ h_{RC}\Big(\frac{(i_k+8M)T_c}{M}-\widehat{\tau}_{k,0}\Big)\bigg],\\
&\mathbf{g}'_k=[\widehat{\mathbf{g}}_k(i_k+1),\widehat{\mathbf{g}}_k(i_k+2),\dots,\widehat{\mathbf{g}}_k(i_k+8M)].
\end{flalign}
\normalsize
Then, $\widehat{a}_{k,0}$ and $\widehat{\phi}_{k,0}$ are obtained as the magnitude and phase of the complex quantity $\frac{\mathbf{\Psi}_{k,0}^H\mathbf{g}'_k}{||\mathbf{\Psi}_{k,0}||^2}$. For estimating the parameters of a multipath channel ($L\geq2$), the estimation procedure in \cite{Buzzi} relies on a recursive adoption of the single path estimation algorithm. In short, first the above single path estimation algorithm is applied in order to estimate the parameters corresponding to the strongest path. Then, the contribution of this path is subtracted from the estimated CIR $\widehat{\mathbf{g}}$ and the result is denoted by $\widehat{\mathbf{g}}_1$. Next, the single-path estimation method is applied to $\widehat{\mathbf{g}}_1$ which leads to formation of $\widehat{\mathbf{g}}_2$. {Iterating this procedure $L$ times results in estimating all the channel parameters.}

\end{document}